\pdfoutput=1
\documentclass[twoside,11pt]{article}

\usepackage[preprint]{jmlr2e}

\usepackage{microtype}
\usepackage{bm}
\usepackage{url}
\usepackage{amsmath,amsfonts,paralist}
\usepackage{color}
\usepackage{algorithm}
\usepackage{algorithmic}
\usepackage{epstopdf}
\usepackage{subfigure}

\def \ze {\mathbf{0}}
\def \R {\mathbb{R}}
\def \K {\mathcal{K}}

\def \C {\mathcal{C}}
\def \B {\mathcal{B}}
\def \SS {\mathcal{S}}
\def \dd {\mathbf{d}}
\def \E {\mathbb{E}}
\def \x {\mathbf{x}}
\def \y {\mathbf{y}}
\def \g {\mathbf{g}}
\def \z {\mathbf{z}}
\def \ccc {\mathbf{c}}
\def \uu {\mathbf{u}}
\def \vv {\mathbf{v}}
\def \w {\mathbf{w}}

\DeclareMathOperator*{\prob}{Pr}
\DeclareMathOperator*{\ii}{in}
\DeclareMathOperator*{\oo}{out}

\DeclareMathOperator*{\argmin}{argmin}
\DeclareMathOperator*{\argmax}{argmax}

\newtheorem{thm}{Theorem}
\newtheorem{lem}{Lemma}
\newtheorem{myDef}{Definition}
\newtheorem{assum}{Assumption}

\newtheorem{cor}{Corollary}

\makeatletter
\newcommand\figcaption{\def\@captype{figure}\caption}
\newcommand\tabcaption{\def\@captype{table}\caption}
\makeatother

\jmlrheading{1}{2021}{1-52}{4/00}{10/00}{xxx}{Yuanyu Wan, Guanghui Wang, Wei-Wei Tu, and Lijun Zhang}

\ShortHeadings{Projection-free Distributed Online Learning}{Yuanyu Wan, Guanghui Wang, Wei-Wei Tu, and Lijun Zhang}
\firstpageno{1}

\begin{document}

\title{Projection-free Distributed Online Learning with Sublinear Communication Complexity}

\author{\name Yuanyu Wan \email wanyy@lamda.nju.edu.cn\\
       \name Guanghui Wang \email wanggh@lamda.nju.edu.cn\\
       \addr National Key Laboratory for Novel Software Technology, Nanjing University, Nanjing 210023, China\\
       \name Wei-Wei Tu \email tuweiwei@4paradigm.com\\
       \addr 4Paradigm Inc., Beijing 100000, China\\
       \name Lijun Zhang \email zhanglj@lamda.nju.edu.cn \\
       \addr National Key Laboratory for Novel Software Technology, Nanjing University, Nanjing 210023, China}

\editor{Csaba Szepesvari}

\maketitle

\begin{abstract}
To deal with complicated constraints via locally light computations in distributed online learning, a recent study has presented a projection-free algorithm called distributed online conditional gradient (D-OCG), and achieved an $O(T^{3/4})$ regret bound for convex losses, where $T$ is the number of total rounds. However, it requires $T$ communication rounds, and cannot utilize the strong convexity of losses. In this paper, we propose an improved variant of D-OCG, namely D-BOCG, which can attain the same $O(T^{3/4})$ regret bound with only $O(\sqrt{T})$ communication rounds for convex losses, and a better regret bound of $O(T^{2/3}(\log T)^{1/3})$ with fewer $O(T^{1/3}(\log T)^{2/3})$ communication rounds for strongly convex losses. The key idea is to adopt a delayed update mechanism that reduces the communication complexity, and redefine the surrogate loss function in D-OCG for exploiting the strong convexity. Furthermore, we provide lower bounds to demonstrate that the $O(\sqrt{T})$ communication rounds required by D-BOCG are optimal (in terms of $T$) for achieving the $O(T^{3/4})$ regret with convex losses, and the $O(T^{1/3}(\log T)^{2/3})$ communication rounds required by D-BOCG are near-optimal (in terms of $T$) for achieving the $O(T^{2/3}(\log T)^{1/3})$ regret with strongly convex losses up to polylogarithmic factors. Finally, to handle the more challenging bandit setting, in which only the loss value is available, we incorporate the classical one-point gradient estimator into D-BOCG, and obtain similar theoretical guarantees.

\end{abstract}

\section{Introduction}
Conditional gradient (CG) \citep{FW-56,Revist_FW} (also known as Frank-Wolfe) is a simple yet efficient offline algorithm for solving high-dimensional problems with complicated constraints. To find a feasible solution, instead of performing the time-consuming projection step, CG utilizes the linear optimization step, which can be carried out much more efficiently. For example, in the matrix completion problem \citep{Hazan2012}, where the feasible set consists of all matrices with bounded trace norm, the projection step needs to compute the singular value decomposition (SVD) of a matrix. In contrast, the linear optimization step in CG only requires computing the top singular vector pair of a matrix, which is at least an order of magnitude faster than the SVD. Due to the emergence of large-scale problems, online conditional gradient (OCG) \citep{Hazan2012,Hazan2016} (also known as online Frank-Wolfe) was proposed for online convex optimization (OCO)---a multi-round game between a learner and an adversary \citep{Zinkevich2003}, and achieved a regret bound of $O(T^{3/4})$ for convex losses, where $T$ is the number of total rounds. In each round, OCG updates the learner by utilizing one linear optimization step to minimize a surrogate loss function. Different from CG that requires all data related to the objective function are given beforehand, OCG only requires a single data point per round.

Recently, \citet{wenpeng17} further proposed D-OCG by extending OCG into a more practical scenario---distributed OCO over a network. It is well motivated by many distributed applications such as multi-agent coordination and distributed tracking in sensor networks \citep{Distrbuted02,Xiaolin07,Angelia09,DADO2011,Yang19survey}. Specifically, by defining the network as an undirected graph, each node of the graph represents a local learner, and can only communicate with its neighbors. The key idea of D-OCG is to maintain OCG for each local learner, and update it according to the local gradient as well as those received from its neighbors in each round. Compared with projection-based distributed algorithms \citep{D-OGD,D-ODA,DAOL_TKDE}, D-OCG significantly reduces the time cost for solving high-dimensional problems with complicated constraints, because it only utilizes one linear optimization step for each update of local learners. Moreover, D-OCG is more scalable than OCG, since it can utilize many locally light computation resources to handle large-scale problems.

However, there exist two interesting questions about D-OCG. First, the local learners of D-OCG communicate with their neighbors to share the local gradients in each round, so it requires $T$ communication rounds in total. Since the communication overhead is often the performance bottleneck in distributed systems, it is natural to ask whether the communication complexity of D-OCG can be reduced without increasing its regret. Second, similar to OCG in the standard OCO, \citet{wenpeng17} have proved that D-OCG in the distributed OCO achieves an $O(T^{3/4})$ regret bound for convex losses. Note that recent studies \citep{Garber_SOFW,SC_OFW} in the standard OCO have proposed variants of OCG to attain better regret for strongly convex losses. It is thus natural to ask whether the strong convexity can also be utilized to improve the regret of D-OCG in the distributed OCO.

In this paper, we provide affirmative answers for these two questions by developing an improved variant of D-OCG, namely distributed block online conditional gradient (D-BOCG), which can attain the same $O(T^{3/4})$ regret bound with only $O(\sqrt{T})$ communication rounds for convex losses, and a better regret bound of $O(T^{2/3}(\log T)^{1/3})$ with fewer $O(T^{1/3}(\log T)^{2/3})$ communication rounds for strongly convex losses. Compared with the original D-OCG, there exist three critical changes.
\begin{itemize}
\item To further utilize the strong convexity of losses, a more general surrogate loss function is introduced in our D-BOCG, which is inspired by the surrogate loss function used in strongly convex variants of OCG \citep{Garber_SOFW,SC_OFW} and is able to cover that used in D-OCG.
\item To reduce the communication complexity, our D-BOCG adopts a delayed update mechanism, which divides the total $T$ rounds into a smaller number of equally-sized blocks, and only updates the local learners for each block. In this way, the local learners only need to communicate with their neighbors once for each block, and the total number of communication rounds is reduced from $T$ to the number of blocks.
\item According to the delayed update mechanism, the number of updates in our D-BOCG could be much smaller than that in D-OCG, which brings a new challenge, i.e., only performing 1 linear optimization step as D-OCG for each update will increase the $O(T^{3/4})$ regret of D-OCG for convex losses. To address this problem, we perform iterative linear optimization steps for each update. Specifically, the number of linear optimization steps for each update is set to be the same as the block size, which ensures that the total number of linear optimization steps required by our D-BOCG is the same as that required by D-OCG.
\end{itemize}
Note that the delayed update mechanism and the iterative linear optimization steps in the last two changes are borrowed from \citet{Garber19}, which employed them to improve projection-free bandit convex optimization. In contrast, we apply them to the distributed setting considered here.

Furthermore, to complement theoretical guarantees of our D-BOCG, for any distributed online algorithm with $C$ communication rounds, we provide an $\Omega(T/\sqrt{C})$ lower regret bound with convex losses, and an $\Omega(T/C)$ lower regret bound with strongly convex losses, respectively. These lower bounds imply that the $O(\sqrt{T})$ communication rounds required by D-BOCG are optimal (in terms of $T$) for achieving the $O(T^{3/4})$ regret with convex losses, and the $O(T^{1/3}(\log T)^{2/3})$ communication rounds required by D-BOCG are near-optimal (in terms of $T$) for achieving the $O(T^{2/3}(\log T)^{1/3})$ regret with strongly convex losses up to polylogarithmic factors. To the best of our knowledge, we are the first to study lower regret bounds for distributed online algorithms with limited communication rounds. Finally, to handle the more challenging bandit setting, we propose  distributed block bandit conditional gradient (D-BBCG) by combining with the classical one-point gradient estimator \citep{OBO05}, which can approximate the gradient with a single loss value. Our theoretical analysis first reveals that in expectation, D-BBCG can also attain a regret bound of $O(T^{3/4})$ with $O(\sqrt{T})$ communication rounds for convex losses, and a regret bound of $O(T^{2/3}(\log T)^{1/3})$ with $O(T^{1/3}(\log T)^{2/3})$ communication rounds for strongly convex losses. Moreover, for convex losses, we show that D-BBCG enjoys a high-probability regret bound of $O(T^{3/4}(\log T)^{1/2})$ with $O(\sqrt{T})$ communication rounds.

A preliminary version of this paper was presented at the 37th International Conference on Machine Learning in 2020 \citep{Wan-ICML-2020}. In this paper, we have significantly enriched the preliminary version by adding the following extensions.
\begin{itemize}
\item Different from \citet{Wan-ICML-2020} that only considered convex losses, we generalize D-BOCG and D-BBCG to further exploit the strong convexity, and establish improved theoretical guarantees for strongly convex losses.
\item Different from \citet{Wan-ICML-2020} that only studied upper regret bounds, we provide lower bounds on the regret of distributed online algorithms with limited communication rounds for convex losses as well as strongly convex losses. %
\item We provide more experiments including new results for distributed online binary classification, different networks topologies, different network sizes, and three additional datasets.
\end{itemize}

\section{Related Work}
In this section, we briefly review existing projection-free algorithms for OCO and the distributed OCO.
\subsection{Projection-free Algorithms for OCO}
OCO is a general framework for online learning, which covers a variety of problems such as online portfolio selection \citep{Blum1999,OPS_Hazan}, online routing \citep{Awerbuch04,Awerbuch2008}, online metric learning \citep{OML08,Tsagkatakis2011}, and learning with expert advice \citep{LEA97,Yoav97}. It is generally viewed as a repeated game between a leaner and an adversary. In each round $t$, the learner first chooses a decision $\x(t)$ from a convex decision set $\K\subseteq\mathbb{R}^d$. Then, the adversary reveals a convex function $f_t(\x):\K\to\mathbb{R}$, which incurs a loss $f_t(\x(t))$ to the learner. The goal of the learner is to minimize the regret with respect to any fixed optimal decision, which is defined as
\[
R_{T}=\sum_{t=1}^Tf_{t}(\x(t))-\min\limits_{\x\in\K}\sum_{t=1}^Tf_{t}(\x).
\]
OCG \citep{Hazan2012,Hazan2016} is the first projection-free algorithm for OCO, which enjoys a regret bound of $O(T^{3/4})$ for convex losses and updates as the following linear optimization step
\begin{equation}
\label{old_OCG}
\begin{split}
&\vv=\argmin_{\x\in\K}\nabla F_{t}(\x(t))^\top\x\\
&\x({t+1})=\x(t)+s_t(\vv-\x(t))
\end{split}
\end{equation}
where
\begin{equation}
\label{ocg_surrogate}
F_{t}(\x)=\eta\sum_{k=1}^{t-1}\nabla f_k(\x(k))^\top\x+\|\x-\x(1)\|_2^2
\end{equation}
is a surrogate loss function, and $s_t,\eta$ are two parameters. 

Recent studies have proposed to improve the regret of OCG by exploiting the additional curvature of loss functions including smoothness and strong convexity. For convex and smooth losses, \citet{Hazan20} proposed the online smooth projection free algorithm, and obtained an expected regret bound of $O(T^{2/3})$ as well as a high-probability regret bound of $O(T^{2/3}\log T)$. If the losses are $\alpha$-strongly convex, \citet{SC_OFW} proposed a strongly convex variant of OCG by redefining the surrogate loss function as
\begin{equation}
\label{sc_surrogate}
F_{t}(\x)=\sum_{k=1}^{t-1}\left(\nabla f_k(\x(k))^\top\x+\frac{\alpha}{2}\|\x-\x(k)\|_2^2\right)
\end{equation}
and using a line search rule to select the original parameter $s_t$ in (\ref{old_OCG}). This algorithm can enjoy a regret bound of $O(T^{2/3})$ for strongly convex losses, and a very similar algorithm was concurrently proposed by \citet{Garber_SOFW}. Moreover, when the decision set is polyhedral \citep{Garber16} or smooth \citep{kevy_smooth}, projection-free algorithms have been proposed to enjoy an $O(\sqrt{T})$ regret bound for convex losses and an $O(\log T)$ regret bound for strongly convex losses, respectively. If the decision set is strongly convex, \citet{SC_OFW} have proved that OCG can achieve an $O(T^{2/3})$ regret bound for convex losses, and the strongly convex variant of OCG can achieve an $O(\sqrt{T})$ regret bound for strongly convex losses.

Furthermore, OCG has been extended to handle the more challenging bandit setting, where only the loss value is available to the learner. Due to the lack of the gradient, \citet{PBCO2019} proposed a bandit variant of OCG by combining with the one-point gradient estimator \citep{OBO05}, which can approximate the gradient with a single loss value. For convex losses, the bandit variant of OCG achieves an expected regret bound of $O(T^{4/5})$, which is worse than the $O(T^{3/4})$ regret bound of OCG. Later, by dividing the total rounds into several equally-sized blocks and performing iterative linear optimization steps for each block, \citet{Garber19} improved the bandit variant of OCG, and reduced the expected regret bound for convex losses from $O(T^{4/5})$ to $O(T^{3/4})$. For strongly convex losses, \citet{Garber_SOFW} further developed a projection-free bandit algorithm that attains an expected regret bound of $O(T^{2/3}\log T)$.

We also note that \citet{chen18c} developed a projection-free algorithm for another interesting setting where the learner is allowed to access to the stochastic gradients.

\subsection{Projection-free Algorithms for the Distributed OCO}
According to previous studies \citep{D-ODA,wenpeng17}, distributed OCO is a variant of OCO over a network defined by an undirected graph $\mathcal{G}=(V,E)$, where $V=[n]$ is the node set and $E\subseteq V\times V$ is the edge set. Different from OCO where only 1 learner exists, in the distributed OCO, each node $i\in V$ is a local learner, and can only communicate with its immediate neighbors
\[N_i=\left\{j\in V|(i,j)\in E\right\}.\]
In each round $t$, each local learner $i\in V$ chooses a decision $\x_i(t)\in\K$, and then it receives a convex loss function $f_{t,i}(\x):\K\to\mathbb{R}$ chosen by the adversary. Moreover, the global loss function $f_t(\x)$ is defined as the sum of local loss functions \[f_t(\x)=\sum_{j=1}^nf_{t,j}(\x).\] The goal of each local learner $i\in V$ is to minimize the regret measured by the global loss with respect to the optimal fixed decision, which is defined as
\begin{equation*}
R_{T,i}=\sum_{t=1}^Tf_{t}(\x_i(t))-\min\limits_{\x\in\K}\sum_{t=1}^Tf_{t}(\x).
\end{equation*}
Since the local loss function $f_{t,i}(\x)$ is only available to the local learner $i$, to achieve this global goal for all local learners, it is necessary to utilize both their local gradients and those received from their neighbors.

Therefore, to make OCG distributed, \citet{wenpeng17} first introduced a non-negative weight matrix $P\in \mathbb{R}^{n\times n}$, and redefined the surrogate loss function $F_t(\x)$ in OCG as
\begin{equation}
\label{D-OCG-F}
F_{t,i}(\x)=\eta\z_i(t)^\top\x+\|\x-\x_1(1)\|_2^2
\end{equation}
for each local learner $i$ by replacing $\sum_{k=1}^{t-1}\nabla f_k(\x(k))$ with $\z_i(t)$, where $\z_i(1)=\mathbf{0}$ and
\begin{equation}
\label{D-OCG-Z}
\z_i(t+1)=\sum_{j\in N_i}P_{ij}\z_j(t)+\nabla f_{t,i}(\x_i(t)).
\end{equation}
Note that the matrix $P$ is also referred to as a gossip, consensus, or averaging matrix \citep{Gossip-averging,DSCO-gossip,ICML-Gossip}. Moreover, $\z_i(t)$ is a weighted sum of historical local gradients and those received from the neighbors, which could be viewed as an approximation for the sum of global gradients and is critical for minimizing the global regret.

Then, with a time-varying parameter $s_t$, they proposed D-OCG updating as follows
\begin{equation}
\label{D-OCG-code}
\begin{split}
&\textbf{for}\text{ each local learner $i\in V$ }\textbf{do}\\
&\quad\vv_i=\argmin_{\x\in\K}\nabla F_{t,i}(\x_i(t))^\top\x\\
&\quad\x_i({t+1})=\x_i(t)+s_t(\vv_i-\x_i(t))\\
&\textbf{end for}
\end{split}
\end{equation}
and established a regret bound of $O(T^{3/4})$ for convex losses. However, in each round $t$, each local learner $i$ needs to compute $\z_i(t+1)$ by communicating with its neighbors, which requires $T$ communication rounds in total.

\section{Preliminaries}
In this section, we introduce necessary preliminaries including standard definitions, common assumptions, and basic algorithmic ingredients.
\subsection{Definitions}
We first recall the standard definitions for smooth and strongly convex functions \citep{Boyd04}.
\begin{myDef}
\label{def1}
Let $f(\x):\K\to\mathbb{R}$ be a function over $\K$. It is called $\beta$-smooth over $\K$ if for all $\x\in\K,\mathbf{y}\in \K$
\[f(\mathbf{y})\leq f(\x)+\nabla f(\x)^{\top}(\mathbf{y}-\x)+\frac{\beta}{2}\|\mathbf{y}-\x\|_2^2.\]
\end{myDef}
\begin{myDef}
\label{def2}
Let $f(\x):\K\to\mathbb{R}$ be a function over $\K$. It is called $\alpha$-strongly convex over $\K$ if for all $\x\in\K,\mathbf{y}\in \K$
\[f(\mathbf{y})\geq f(\x)+\nabla f(\x)^{\top}(\mathbf{y}-\x)+\frac{\alpha}{2}\|\mathbf{y}-\x\|_2^2.\]
\end{myDef}
Let $f(\x):\K\to\mathbb{R}$ be $\alpha$-strongly convex over $\K$ and $\x^\ast=\argmin_{\x\in\K} f(\x)$.
According to \citet{Hazan2012}, for any $\x\in\K$, it holds that
\begin{equation}
\label{cor_scvx}
\frac{\alpha}{2}\|\x-\x^\ast\|_2^2\leq f(\x)-f(\x^\ast).
\end{equation}
\subsection{Assumptions}
Then, similar to previous studies on OCO and the distributed OCO, we introduce the following assumptions.
\begin{assum}
\label{assum4}
At each round $t$, each local loss function $f_{t,i}(\x)$ is $G$-Lipschitz over $\K$, i.e., $|f_{t,i}(\x)-f_{t,i}(\mathbf{y})|\leq G\|\x-\mathbf{y}\|_2$  for any $\x\in\K,\mathbf{y}\in \K$.
\end{assum}
\begin{assum}
\label{assum1}
The convex decision set $\K$ is full dimensional and contains the origin. Moreover, there exist two constants $r,R>0$ such that $r\B^d\subseteq\K\subseteq R\B^d$ where $\B^d$ denotes the unit Euclidean ball centered at the origin in $\mathbb{R}^d$.
\end{assum}
\begin{assum}
\label{assum5}
The non-negative weight matrix $P\in \mathbb{R}^{n\times n}$ is supported on the graph $\mathcal{G}=(V,E)$, symmetric, and doubly stochastic, which satisfies
\begin{itemize}
\item $P_{ij}>0$ only if $(i,j)\in E$;
\item $\sum_{j=1}^nP_{ij}=\sum_{j\in N_i}P_{ij}=1,\forall i\in V$; $\sum_{i=1}^nP_{ij}=\sum_{i\in N_j}P_{ij}=1,\forall j\in V$.
\end{itemize}
Moreover, the second largest singular value of $P$ denoted by $\sigma_2(P)$ is strictly smaller than 1.
\end{assum}
The non-negative weight matrix $P$ in Assumption \ref{assum5} will be utilized to model the communication between the local learners in the distributed OCO.
\begin{assum}
\label{assum_sc}
At each round $t$, each local loss function $f_{t,i}(\x)$ is $\alpha$-strongly convex over $\K$.
\end{assum}
Note that if Assumption \ref{assum_sc} holds with $\alpha=0$, it reduces to the case with only convex losses.

Moreover, according to previous studies \citep{OBO05,Garber19}, the following assumption is required in the bandit setting.
\begin{assum}
\label{assum2}
At each round $t$, each local loss function $f_{t,i}(\x)$ is bounded over $\K$, i.e., $|f_{t,i}(\x)|\leq M$,
for any $\x\in\K$. Moreover, all local loss functions are chosen beforehand, i.e., the adversary is oblivious.
\end{assum}
\begin{algorithm}[t]
\caption{CG}
\label{ILO}
\begin{algorithmic}[1]
\STATE \textbf{Input:} feasible set $\mathcal{K}$, $L$, $F(\x)$, $\x_{\ii}$
\STATE $\mathbf{c}_{0}=\mathbf{x}_{\ii}$
\FOR{$\tau = 0,\dots,L-1$}
\STATE $\mathbf{v}_{\tau}\in\argmin\limits_{\mathbf{x}\in\mathcal{K}}\nabla F(\mathbf{c}_\tau)^\top\mathbf{x}$
\STATE $s_{\tau}=\argmin\limits_{s\in[0,1]}F(\mathbf{c}_{\tau}+s(\mathbf{v}_{\tau}-\mathbf{c}_{\tau}))$
\STATE $\mathbf{c}_{\tau+1}=\mathbf{c}_{\tau}+s_{\tau}(\mathbf{v}_{\tau}-\mathbf{c}_{\tau})$
\ENDFOR
\STATE \textbf{return} $\mathbf{x}_{\oo}=\mathbf{c}_{L}$
\end{algorithmic}
\end{algorithm}
\subsection{Algorithmic Ingredients}
Next, we present conditional gradient (CG) \citep{FW-56,Revist_FW}, which will be utilized to minimize surrogate loss functions in our algorithms. Given a function $F(\x):\K\to\mathbb{R}$ and an initial point $\ccc_0=\x_{\ii}\in\K$, it iteratively performs the linear optimization step as follows
\begin{equation*}
\begin{split}
&\mathbf{v}_{\tau}\in\argmin\limits_{\mathbf{x}\in\mathcal{K}} \nabla F(\mathbf{c}_\tau)^\top\mathbf{x}\\
&\mathbf{c}_{\tau+1}=\mathbf{c}_{\tau}+s_{\tau}(\mathbf{v}_{\tau}-\mathbf{c}_{\tau})
\end{split}
\end{equation*}
for $\tau = 0,\dots,L-1$, where $L$ is the number of iterations and $s_{\tau}$ is selected by line search
\[s_{\tau}=\argmin\limits_{s\in[0,1]}F(\mathbf{c}_{\tau}+s(\mathbf{v}_{\tau}-\mathbf{c}_{\tau})).\]
The detailed procedures of CG are summarized in Algorithm \ref{ILO}, and its convergence rate is presented in the following lemma.
\begin{lem}
\label{lem_ILO}(Derived from Theorem 1 of \citet{Revist_FW}) If $F(\x):\K\to\mathbb{R}$ is a convex and $\beta$-smooth function and $\|\x\|_2\leq R$ for any $\x\in\K$, Algorithm \ref{ILO} with $L\geq 1$ ensures
\begin{align*}
F(\x_{\oo})-F(\x^\ast)\leq\frac{8\beta R^2}{L+2}
\end{align*}
where $\x^\ast\in\argmin_{\x\in\K}F(\x)$.
\end{lem}
According to Lemma \ref{lem_ILO}, when $L$ is large enough, CG can output a point $\x_{\oo}$ such that the approximation error $F(\x_{\oo})-F(\x^\ast)$ is very small. 
As a result, with an appropriate $L>1$, we can minimize our surrogate loss functions more accurately than only performing 1 linear optimization step, which is critical for achieving our desired regret bounds with only sublinear communication complexity.
Moreover, CG has been employed to develop projection-free algorithms with improved regret bounds for bandit convex optimization \citep{Garber19,Garber_SOFW}. In this paper, we introduce it into the distributed OCO to propose projection-free algorithms with sublinear communication complexity for the full information and bandit settings, respectively.

Finally, we introduce a standard technique called one-point gradient estimator \citep{OBO05}, which can approximate
the gradient with a single loss value and will be utilized in the bandit setting. For a function $f(\x)$, its $\delta$-smoothed version is defined as \[\widehat{f}_\delta(\x)=\mathbb{E}_{\uu\sim\B^d}[f(\x+\delta\uu)]\]
and satisfies the following lemma.
\begin{lem}
\label{smoothed_lem2}
(Lemma 1 in \citet{OBO05})
Let $\delta>0$, we have
\[\nabla\widehat{f}_\delta(\x)=\mathbb{E}_{\uu\sim\SS^d}\left[\frac{d}{\delta}f(\x+\delta\uu)\uu\right]\]
where $\SS^d$ denotes the unit sphere in $\mathbb{R}^d$.
\end{lem}
Lemma \ref{smoothed_lem2} provides an unbiased estimator of the gradient $\nabla\widehat{f}_\delta(\x)$ by only observing the single value $f(\x+\delta\uu)$. Note that there also exist two-point and $(d+1)$-point gradient estimators \citep{Agarwal2010_COLT,Duchi-TIT-15}, which can approximate the gradient more accurately than the one-point gradient estimator. However, by querying two or $(d+1)$ points per round, in expectation, \citet{Agarwal2010_COLT} have recovered the best regret bounds established in the full information setting for both convex and strong convex losses, which actually implies that the bandit setting with two or $(d+1)$ points is almost as simple as the full information setting. So, in this paper, we only consider the most challenging bandit setting, where only one point is available per round. Moreover, we will show that the one-point gradient estimator is sufficient for projection-free algorithms to obtain theoretical guarantees similar to those obtained in the full information setting.

\section{Distributed Block Online Conditional Gradient (D-BOCG)}
\label{Sec4}
In this section, we present our D-BOCG with the corresponding theoretical guarantees for convex losses and strongly convex losses.

\subsection{The Algorithm}
First, to reduce the communication complexity of D-OCG, we divide the total $T$ rounds into $B$ blocks of size $K$, where $K$ is a parameter and we assume that $B=T/K$ is an integer without loss of generality. In this way, each block $m\in [B]$ consists of a set of rounds \[\mathcal{T}_m=\{(m-1)K+1,\dots,mK\}.\] For each local learner $i\in V$, its decision in each block $m$ stays the same and is denoted by $\x_i(m)$. The local gradient of local learner $i$ in each round $t\in\mathcal{T}_m$ is denoted by $\g_{i}(t)=\nabla f_{t,i}(\x_{i}(m))$ and the cumulative gradient of local learner $i$ in each block $m$ is denoted by \[\widehat{\g}_i(m)=\sum_{t\in\mathcal{T}_m}\g_{i}(t).\] Then, we describe how to compute $\x_i(m)$ for each local learner in each block $m$. Initially, we set $\x_i(1)=\x_{\ii}$ for each local learner $i$, where $\x_{\ii}$ is arbitrarily chosen from $\K$. For any $m>1$, following \citet{wenpeng17}, we can update the decision $\x_i(m)$ by approximately minimizing a surrogate loss function $F_{m-1,i}(\x)$. One may adopt a surrogate loss function similar to (\ref{D-OCG-F}) used by D-OCG. However, it was only designed for convex losses, which cannot utilize the strong convexity. To address this limitation, we define a more general surrogate loss function for $\alpha$-strongly convex losses, which can utilize the strong convexity for $\alpha>0$ and also cover that used in D-OCG for $\alpha=0$. To help understanding, we start by introducing our surrogate loss function for the simple case with $n=1$, and then extend it to the general case for any $n\geq1$.
\paragraph{The simple case}
In the simple case with $n=1$, the distributed OCO reduces to the standard OCO, and we only need to define a surrogate loss function $F_{m,1}(\x)$ for each block $m\in[B]$. Note that when we assume that all local losses are $\alpha$-strongly convex, the cumulative loss function $\sum_{t\in\mathcal{T}_{m}}f_{t,1}(\x)$ in each block $m$ is $\alpha K$-strongly convex, because of $|\mathcal{T}_{m}|=K$. 
Therefore, to utilize the strong convexity, inspired by (\ref{sc_surrogate}), we can define $F_{m,1}(\x)$ as
\[F_{m,1}(\x)=\sum_{k=1}^{m-1}\left(\widehat{\g}_1(k)^\top\x+\frac{\alpha K}{2}\|\x-\x_{1}(k)\|_2^2\right)+h\|\x-\x_{\ii}\|_2^2
\]
where $h$ is a parameter that allows us to recover the surrogate loss function (\ref{ocg_surrogate}) for convex losses when $\alpha=0$, though it does not exist in (\ref{sc_surrogate}). Since $\|\x_1(k)\|_2^2$ does not affect the minimizer of the function $F_{m,1}(\x)$, we further simplify $F_{m,1}(\x)$ to
\begin{align*}
F_{m,1}(\x)=\sum_{k=1}^{m-1}\left(\widehat{\g}_1(k)-\alpha K\x_1(k)\right)^\top\x+\frac{(m-1)\alpha K}{2}\|\x\|_2^2+h\|\x-\x_{\ii}\|_2^2.
\end{align*}
By initializing $\z_1(1)=\ze$ and computing $\z_1(m+1)$ as
\begin{equation}
\label{simple_z}
\z_{1}(m+1)=\z_{1}(m)+\widehat{\g}_{1}(m)-\alpha K\x_1(m)
\end{equation}
we can rewrite the above $F_{m,1}(\x)$ as
\begin{align*}
F_{m,1}(\x)
=\z_1(m)^\top\x+\frac{(m-1)\alpha K}{2}\|\x\|_2^2+h\|\x-\x_{\ii}\|_2^2.
\end{align*}
\paragraph{The general case}
Note that $F_{m,1}(\x)$ and $\z_1(m)$ in the simple case only contain the information of the local learner $1$, which cannot be used to achieve a regret bound measured by the global loss in the general case. Therefore, inspired by (\ref{D-OCG-Z}) used in D-OCG, in the general case, we update $\z_i(m+1)$ as
\begin{equation}
\label{new_Z}
\z_{i}(m+1)=\sum_{j\in N_i}P_{ij}\z_{j}(m)+\widehat{\g}_{i}(m)-\alpha K\x_i(m)
\end{equation}
for each local learner $i$, where $P$ is a non-negative weight matrix satisfying Assumption \ref{assum5}.
Different from (\ref{simple_z}), this update further incorporates the information from the neighbors of local learner $i$. Moreover, in each block $m$, the surrogate loss function for each local learner $i$ is defined as
\begin{align}
\label{new_F}
F_{m,i}(\x)
=\z_i(m)^\top\x+\frac{(m-1)\alpha K}{2}\|\x\|_2^2 +h\|\x-\x_{\ii}\|_2^2.
\end{align}
Obviously, for convex losses with $\alpha=0$, this $F_{m,i}(\x)$ is equivalent to (\ref{D-OCG-F}) in D-OCG by setting $K=1$ and $h=1/\eta$.

Finally, we need to specify how to compute $\x_i(m+1)$ by approximately minimizing $F_{m,i}(\x)$ defined in (\ref{new_F}). Similar to the update rules of D-OCG (\ref{D-OCG-code}), one may simply perform 1 linear optimization step with the above $F_{m,i}(\x)$ to compute $\x_i(m+1)$ for any block $m$ and local learner $i$. However, this naive update will increase the $O(T^{3/4})$ regret for convex losses established by D-OCG, since the number of updates is decreased.
To address this problem, we invoke CG for each update as
\begin{equation}
\label{D-BOCG-upX}
\mathbf{x}_{i}(m+1)=\text{CG}(\mathcal{K}, L, F_{m,i}(\x), \x_i(m))
\end{equation}
where $L$ is an appropriate parameter. The detailed procedures of our algorithm are presented in Algorithm \ref{DBOCG-SC}, and it is called distributed block online conditional gradient (D-BOCG).
\begin{algorithm}[t]
\caption{D-BOCG}
\label{DBOCG-SC}
\begin{algorithmic}[1]
\STATE \textbf{Input:} feasible set $\mathcal{K}$, $\mathbf{x}_{\ii}\in\mathcal{K}$, $\alpha, h, L$, and $K$
\STATE \textbf{Initialization:} choose $\mathbf{x}_1(1)=\dots=\mathbf{x}_n(1)=\mathbf{x}_{\ii}$ and set $\z_{1}(1)=\dots=\z_{n}(1)=\ze$
\FOR{$m=1,\dots,T/K$}
\FOR{each local learner $i\in V$}
\STATE define $F_{m,i}(\mathbf{x})=\z_{i}(m)^{\top}\mathbf{x}+\frac{(m-1)\alpha K}{2}\|\mathbf{x}\|_2^2+h\|\x-\mathbf{x}_{\ii}\|_2^2$
\STATE $\widehat{\g}_{i}(m)=\ze$
\FOR{$t=(m-1)K+1,\dots,mK$}
\STATE play $\x_{i}(m)$ and observe $\g_{i}(t)=\nabla f_{t,i}(\x_{i}(m))$
\STATE $\widehat{\g}_{i}(m)=\widehat{\g}_{i}(m)+\g_{i}(t)$
\ENDFOR
\STATE $\mathbf{x}_{i}(m+1)=\text{CG}(\mathcal{K}, L, F_{m,i}(\x), \x_i(m))$ //This step can be executed \emph{in parallel} to the above \emph{for} loop.
\STATE $\z_{i}(m+1)=\sum_{j\in N_i}P_{ij}\z_{j}(m)+\widehat{\g}_{i}(m){-\alpha K\x_i(m)}$
\ENDFOR
\ENDFOR
\end{algorithmic}
\end{algorithm}
\begin{remark}
\label{rem1}
\emph{
We first note that Algorithm \ref{DBOCG-SC} requires $(T/K)L$ linear optimization steps in total. We can limit the total number of linear optimization steps to $T$ by simply setting $L=K$, which is the same as that required by D-OCG \citep{wenpeng17}. Moreover, it is also important to note that at the step 11 in Algorithm \ref{DBOCG-SC}, the computation about $\x_i(m+1)$ only depends on $\x_i(m)$ and $\z_i(m)$, which is available at the beginning of the block $m$. Therefore, the step 11 in Algorithm \ref{DBOCG-SC} can be executed in parallel to the \emph{for} loop from steps 7 to 10 in Algorithm \ref{DBOCG-SC}, which implies that the $L$ linear optimization steps utilized to compute $\x_i(m+1)$ can be uniformly allocated to all $K$ rounds in the block $m$, instead of only the last round in the block $m$. Specifically, Algorithm \ref{DBOCG-SC} with $L=K$ only needs to perform 1 linear optimization step in each round, which is computationally as efficient as D-OCG. This parallel strategy is significant, because without it, Algorithm \ref{DBOCG-SC} needs to stop at the end of each block $m$ and wait until $L$ linear optimization steps are completed. It has also been utilized by \citet{Garber19,Garber_SOFW} when developing improved projection-free algorithms for bandit convex optimization.
}
\end{remark}

\subsection{Theoretical Guarantees}
In the following, we present theoretical guarantees of our D-BOCG. To help understand the effect of the CG method, we start with the following regret bound, the first term in which clearly depends on the approximation error of the CG method.
\begin{thm}
\label{thm1-sc}
Under Assumptions \ref{assum4}, \ref{assum1}, \ref{assum5}, and \ref{assum_sc}, for any $i\in V$, Algorithm \ref{DBOCG-SC} ensures
\begin{equation*}
\begin{split}
R_{T,i}\leq&3nGK\sum_{m=2}^B\sqrt{\frac{2\epsilon_m}{(m-2)\alpha K+2h}}+3nGK\sum_{m=2}^B\frac{K(G+\alpha R)\sqrt{n}}{((m-2)\alpha K+2h) (1-\sigma_2(P))}\\
&+n\sum_{m=1}^B\frac{4K^2(G+2\alpha R)^2}{m\alpha K+2h}+4nhR^2
\end{split}
\end{equation*}
where $\epsilon_m=\max_{i\in [n]}\left(F_{m-1,i}(\x_i(m))-\min_{\x\in\K}F_{m-1,i}(\x)\right)$.
\end{thm}
\begin{remark}
\emph{
Note that D-BOCG with $K=L=1$, $h=1/\eta$, and $\alpha=0$ reduces to D-OCG \citep{wenpeng17}. In that case, according to their analysis (see the proof of Lemmas 2 and 4 in \citet{wenpeng17} for further details), we can prove that $\epsilon_m/h= O(\sqrt{1/{m}})$ by setting $h=\Omega(GT^{3/4})$, where the constant $G$ in $h$ is introduced due to the upper bound of $F_{m,i}(\x)-F_{m-1,i}(\x)$. If we further consider $L=1$, $\alpha=0$, and $K=\sqrt{T}$, we can similarly prove that $\epsilon_m/h= O(\sqrt{1/{m}})$ by setting $h=\Omega(GKB^{3/4})$, since now the upper bound of $F_{m,i}(\x)-F_{m-1,i}(\x)$ in on the order of $O(GK)$ and the maximum $m$ is changed from $T$ to $B$. However, this will make the first term in the above regret bound worse than $O(T^{3/4})$, as explained below
\[3nGK\sum_{m=2}^B\sqrt{\frac{2\epsilon_m}{2h}}= O\left(K\sum_{m=2}^B\frac{1}{m^{1/4}}\right)=O\left(KB^{3/4}\right)=O\left(T^{7/8}\right).\]
Therefore, to keep the $O(T^{3/4})$ regret for convex losses with $K=\sqrt{T}$, we need to use more linear optimization steps. According to Lemma \ref{lem_ILO}, if $L$ is large enough, the approximation error $\epsilon_m$ could be very small. More specifically, by combining Theorem \ref{thm1-sc} with Lemma \ref{lem_ILO}, we have the following theorem.
}
\end{remark}
\begin{thm}
\label{thm1-sc-2}
Under Assumptions \ref{assum4}, \ref{assum1}, \ref{assum5}, and \ref{assum_sc}, for any $i\in V$, Algorithm \ref{DBOCG-SC} ensures
\begin{equation*}
\begin{split}
R_{T,i}\leq&\frac{12nGRT}{\sqrt{L+2}}+\sum_{m=2}^B\frac{3nGK^2(G+\alpha R)\sqrt{n}}{((m-2)\alpha K+2h) (1-\sigma_2(P))}+\sum_{m=1}^B\frac{4nK^2(G+2\alpha R)^2}{m\alpha K+2h}+4nhR^2.
\end{split}
\end{equation*}
\end{thm}
\begin{remark}
\emph{
Note that Theorem \ref{thm1-sc-2} with $K=L=1$, $h=1/\eta$, and $\alpha=0$ cannot recover the $O(T^{3/4})$ regret bound of D-OCG, because $\frac{12nGRT}{\sqrt{L+2}}$ would be on the order of $O(T)$. The main reason is that for the CG method, the approximation error bound in Lemma \ref{lem_ILO} is too loose when $L=1$. According to \citet{wenpeng17}, instead of using Lemma \ref{lem_ILO}, a more complicated analysis is required for bounding the approximation error when only utilizing 1 linear optimization step. To recover the $O(T^{3/4})$ regret bound with $K=L=1$, $h=1/\eta$, and $\alpha=0$, one potential way is to extend the analysis of \citet{wenpeng17} from the case with $L=1$ to the case with any $L$. However, we find that this extension is highly non-trivial, and notice that Theorem \ref{thm1-sc-2} is sufficient to achieve our desired regret bounds and communication complexity.
}
\end{remark}
For convex losses, we can simplify Theorem \ref{thm1-sc-2} to the following corollary.
\begin{cor}
\label{cor2}
Under Assumptions \ref{assum4}, \ref{assum1}, \ref{assum5}, and \ref{assum_sc} with $\alpha=0$, for any $i\in V$, Algorithm \ref{DBOCG-SC} with $\alpha=0$, $K=L=\sqrt{T}$, and $h=\frac{n^{1/4}T^{3/4}G}{\sqrt{1-\sigma_2(P)}R}$ has
\[
R_{T,i}\leq\left(12n+2\sqrt{1-\sigma_2(P)}n^{3/4}+\frac{11}{2}n^{5/4}(1-\sigma_2(P))^{-1/2}\right)GRT^{3/4}.
\]
\end{cor}
\begin{remark}\emph{
The above corollary shows that our D-BOCG can enjoy the $O(T^{3/4})$ regret bound with only $O(\sqrt{T})$ communication rounds for convex losses. By contrast, D-OCG \citep{wenpeng17} obtains the $O(T^{3/4})$ regret bound with a larger number of $T$ communication rounds for convex losses.
}
\end{remark}
Moreover, for strongly convex losses, we can simplify Theorem \ref{thm1-sc-2} to the following corollary.
\begin{cor}
\label{cor-sc}
Under Assumptions \ref{assum4}, \ref{assum1}, \ref{assum5}, and \ref{assum_sc} with $\alpha>0$, for any $i\in V$, Algorithm \ref{DBOCG-SC} with $\alpha>0$, $K=L=T^{2/3}(\ln T)^{-2/3}$, and $h=\alpha K$ ensures
\begin{align*}
R_{T,i}\leq&\left(\frac{3n^{3/2}G(G+\alpha R)}{\alpha (1-\sigma_2(P))}+\frac{4n(G+2\alpha R)^2}{\alpha}\right)T^{2/3}((\ln T)^{-2/3}+(\ln T)^{1/3})\\
&+12nGRT^{2/3}(\ln T)^{1/3}+4nR^2\alpha T^{2/3}(\ln T)^{-2/3}.
\end{align*}
\end{cor}
\begin{remark}
\emph{
The above corollary shows that our D-BOCG can enjoy a regret bound of $O(T^{2/3}(\log T)^{1/3})$ with $O(T^{1/3}(\log T)^{2/3})$ communication rounds for strongly convex losses. Compared with Corollary \ref{cor2}, both the regret and communication complexity of our D-BOCG are improved by utilizing the strong convexity.
}
\end{remark}
\begin{remark}
\emph{
Besides the dependence on $T$, Corollaries \ref{cor2} and \ref{cor-sc} also explicitly show how
the regret of our D-BOCG depends on the network size $n$ and the spectral gap $1-\sigma_2(P)$. First, the dependence on $n$ shows that the regret of our D-BOCG will be larger on larger networks for convex losses and strongly convex losses. Second, the spectral gap actually reflects the connectivity of the network: a larger spectral gap value implies better connectivity \citep{DADO2011,wenpeng17}. Therefore, the dependence on the spectral gap implies that the regret of our D-BOCG will be smaller on ``well connected" networks than on ``poorly connected" networks for convex losses and strongly convex losses. More specifically, with a particular choice of the matrix $P$, \citet{DADO2011} have bounded the spectral gap for several classes of interesting networks, such as $1-\sigma_2(P)=\Omega(1)$ for expanders and the complete graph, and $1-\sigma_2(P)=\Omega(1/n^2)$ for the cycle graph (see Section 3.2 in \citet{DADO2011} for details). By replacing the dependence on $1-\sigma_2(P)$ with that on $n$, our Corollaries \ref{cor2} and \ref{cor-sc} further imply that
\begin{itemize}
\item in the case with convex losses, the regret of D-BOCG can be bounded by $O(n^{5/4}T^{3/4})$ for ``well connected" networks and $O(n^{9/4}T^{3/4})$ for ``poorly connected" networks;
\item in the case with strongly convex losses, the regret of D-BOCG can be bounded by $O(n^{3/2}T^{2/3}(\log T)^{1/3})$ for ``well connected" networks and $O(n^{7/2}T^{2/3}(\log T)^{1/3})$ for ``poorly connected" networks.
\end{itemize}
}
\end{remark}

\subsection{Analysis}
In the following, we only provide the proofs of Theorems \ref{thm1-sc} and \ref{thm1-sc-2}. The proofs of Corollary \ref{cor2} and Corollary \ref{cor-sc} can be found in the Appendix.
\subsubsection{Proof of Theorem \ref{thm1-sc}}
\label{A1.1}
We start this proof by defining several auxiliary variables. Let $\bar{\z}(m)=\frac{1}{n}\sum_{i=1}^n\z_i(m)$ for $m\in[B+1]$, and let $\dd_i(m)=\widehat{\g}_i(m)-\alpha K\x_i(m)$ and $\bar{\dd}(m)=\frac{1}{n}\sum_{i=1}^n\dd_i(m)$ for $m\in[B]$. According to Assumption \ref{assum5}, we have
\begin{equation}
\label{thm_sc-eq1}
\begin{split}
\bar{\z}(m+1)=&\frac{1}{n}\sum_{i=1}^n\z_i(m+1)=\frac{1}{n}\sum_{i=1}^n\left(\sum_{j\in N_i}P_{ij}\z_{j}(m)+\widehat{\g}_{i}(m)-\alpha K\x_i(m)\right)\\
=&\frac{1}{n}\sum_{i=1}^n\sum_{j=1}^nP_{ij}\z_{j}(m)+\bar{\dd}(m)=\frac{1}{n}\sum_{j=1}^n\sum_{i=1}^nP_{ij}\z_{j}(m)+\bar{\dd}(m)=\bar{\z}(m)+\bar{\dd}(m).
\end{split}
\end{equation}
Then, we define $\bar{\x}(1)=\mathbf{x}_{\ii}$ and $\bar{\x}(m+1)=\argmin_{\x\in\K}\bar{F}_{m}(\x)$ for any $m\in[B+1]$, where
\[\bar{F}_{m}(\x)=\bar{\z}(m)^{\top}\mathbf{x}+\frac{(m-1)\alpha K}{2}\|\mathbf{x}\|_2^2+h\|\mathbf{x}-\mathbf{x}_{\ii}\|_2^2.\]
Similarly, we define $\widehat{\x}_i(m+1)=\argmin_{\x\in\K}F_{m,i}(\x)$ for any $m\in[B+1]$, where \[F_{m,i}(\mathbf{x})=\z_{i}(m)^{\top}\mathbf{x}+\frac{(m-1)\alpha K}{2}\|\mathbf{x}\|_2^2+h\|\mathbf{x}-\mathbf{x}_{\ii}\|_2^2.\] is defined in Algorithm \ref{DBOCG-SC}.

Then, we derive an upper bound of $\|\x_i(m)-\bar{\x}(m)\|_2$ for any $m\in[B]$. If $m=1$, according to the definition and Algorithm \ref{DBOCG-SC}, it is easy to verify that
\begin{equation}
\label{thm_sc-eq2}
\|\x_i(m)-\bar{\x}(m)\|_2=0.
\end{equation}
For any $B\geq m\geq2$, due to $\epsilon_m=\max_{i\in [n]}\left(F_{m-1,i}(\x_i(m))-\min_{\x\in\K}F_{m-1,i}(\x)\right)$, we have
\begin{equation}
\label{thm_sc-eq3}
\begin{split}
\|\x_i(m)-\bar{\x}(m)\|_2\leq&\|\x_i(m)-\widehat{\x}_i(m)\|_2+\|\widehat{\x}_i(m)-\bar{\x}(m)\|_2\\
\leq&\sqrt{\frac{2F_{m-1,i}(\x_i(m))-2F_{m-1,i}(\widehat{\x}_i(m))}{(m-2)\alpha K+2h}}+\|\widehat{\x}_i(m)-\bar{\x}(m)\|_2\\
\leq&\sqrt{\frac{2\epsilon_m}{(m-2)\alpha K+2h}}+\|\widehat{\x}_i(m)-\bar{\x}(m)\|_2
\end{split}
\end{equation}
where the second inequality is due to the fact that $F_{m-1,i}(\x)$ is $((m-2)\alpha K+2h)$-strongly convex and (\ref{cor_scvx}).

To further bound $\|\widehat{\x}_i(m)-\bar{\x}(m)\|_2$ in (\ref{thm_sc-eq3}), we introduce the following two lemmas.
\begin{lem}
\label{graph_lem_zhang}
(Derived from the Proof of Lemma 6 in \citet{wenpeng17})
For any $i\in[n]$, let $\dd_i(1),\dots,\dd_i(m)\in\mathbb{R}^d$ be a sequence of vectors. Let $\z_i(1)=\ze$, $\z_{i}(m+1)=\sum_{j\in N_i}P_{ij}\z_{j}(m)+\dd_{i}(m)$, and $\bar{\z}(m)=\frac{1}{n}\sum_{i=1}^n\z_i(m)$ for $m\in[B]$, where $P$ satisfies Assumption \ref{assum5}. For any $i\in [n]$ and $m\in[B]$, assuming $\|\dd_i(m)\|_2\leq \widehat{G}$ where $\widehat{G}>0$ is a constant, we have \[\|\z_i(m)-\bar{\z}(m)\|_2\leq\frac{\widehat{G}\sqrt{n}}{1-\sigma_2(P)}.\]
\end{lem}
\begin{lem}
\label{dual_lem1}
(Lemma 5 in \citet{DADO2011}) Let $\Pi_{\K}(\uu,\eta)=\argmin_{\x\in\K}\eta\uu^\top\x+\|\x\|_2^2$. We have
\[\|\Pi_{\K}(\uu,\eta)-\Pi_{\K}(\vv,\eta)\|_2\leq\frac{\eta}{2}\|\uu-\vv\|_2.\]
\end{lem}
According to Assumptions \ref{assum4} and \ref{assum1}, for any $m\in[B]$, we have
\begin{equation}
\label{bound_d1}
\begin{split}
\|\dd_i(m)\|_2=&\|\widehat{\g}_i(m)-\alpha K\x_i(m)\|_2=\left\|\sum_{t\in\mathcal{T}_m}\g_{i}(t)-\alpha K\x_i(m)\right\|_2\\
\leq&\sum_{t\in\mathcal{T}_m}\|\g_{i}(t)\|_2+\alpha K\|\x_i(m)\|_2\leq K(G+\alpha R)
\end{split}
\end{equation}
where $\mathcal{T}_m=\{(m-1)K+1,\dots,mK\}$.

By applying Lemma \ref{graph_lem_zhang} with $\|\dd_i(m)\|_2\leq K(G+\alpha R)$, for any $B\in[m]$, we have
\begin{equation}
\label{thm_sc-eq4}
\|\z_i(m)-\bar{\z}(m)\|_2\leq\frac{K(G+\alpha R)\sqrt{n}}{1-\sigma_2(P)}.
\end{equation}
Moreover, for any $B\geq m\geq2$, we notice that
\begin{equation}
\label{reuse_eq1}
\begin{split}
\widehat{\x}_i(m)&=\argmin\limits_{\x\in\K}\z_{i}(m-1)^{\top}\mathbf{x}+\frac{(m-2)\alpha K}{2}\|\mathbf{x}\|_2^2+h\|\mathbf{x}-\mathbf{x}_{\ii}\|_2^2\\
&=\argmin\limits_{\x\in\K}(\z_{i}(m-1)-2h\mathbf{x}_{\ii})^{\top}\mathbf{x}+\frac{(m-2)\alpha K+2h}{2}\|\mathbf{x}\|_2^2\\
&=\argmin\limits_{\x\in\K}\frac{2}{(m-2)\alpha K+2h}(\z_{i}(m-1)-2h\mathbf{x}_{\ii})^{\top}\mathbf{x}+\|\mathbf{x}\|_2^2.
\end{split}
\end{equation}
Similarly, for any $B\geq m\geq2$, we have
\begin{equation}
\label{reuse_eq2}
\begin{split}
\bar{\x}(m)&=\argmin\limits_{\x\in\K}\bar{\z}(m-1)^{\top}\mathbf{x}+\frac{(m-2)\alpha K}{2}\|\mathbf{x}\|_2^2+h\|\mathbf{x}-\mathbf{x}_{\ii}\|_2^2\\
&=\argmin\limits_{\x\in\K}(\bar{\z}(m-1)-2h\mathbf{x}_{\ii})^{\top}\mathbf{x}+\frac{(m-2)\alpha K+2h}{2}\|\mathbf{x}\|_2^2\\
&=\argmin\limits_{\x\in\K}\frac{2}{(m-2)\alpha K+2h}(\bar{\z}(m-1)-2h\mathbf{x}_{\ii})^{\top}\mathbf{x}+\|\mathbf{x}\|_2^2.
\end{split}
\end{equation}
Therefore, by combining Lemma \ref{dual_lem1} with (\ref{thm_sc-eq4}), for any $B\geq m\geq2$, we have
\begin{align*}\|\widehat{\x}_i(m)-\bar{\x}(m)\|_2&\leq\frac{1}{(m-2)\alpha K+2h}\|\z_i(m-1)-2h\mathbf{x}_{\ii}-\bar{\z}(m-1)+2h\mathbf{x}_{\ii}\|_2\\
&=\frac{1}{(m-2)\alpha K+2h}\|\z_i(m-1)-\bar{\z}(m-1)\|_2\\
&\leq\frac{K(G+\alpha R)\sqrt{n}}{((m-2)\alpha K+2h) (1-\sigma_2(P))}
\end{align*}
By substituting the above inequality into (\ref{thm_sc-eq3}), for any $B\geq m\geq2$, we have
\begin{equation}
\label{thm_sc-eq5}
\begin{split}
\|\x_i(m)-\bar{\x}(m)\|_2\leq\sqrt{\frac{2\epsilon_m}{(m-2)\alpha K+2h}}+\frac{K(G+\alpha R)\sqrt{n}}{((m-2)\alpha K+2h) (1-\sigma_2(P))}.
\end{split}
\end{equation}
Then, let $u_1=0$ and $u_m=\sqrt{\frac{2\epsilon_m}{(m-2)\alpha K+2h}}+\frac{K(G+\alpha R)\sqrt{n}}{((m-2)\alpha K+2h) (1-\sigma_2(P))}$ for any $B\geq m\geq2$. From (\ref{thm_sc-eq2}) and (\ref{thm_sc-eq5}), for any $m\in[B]$, it holds that $\|\x_i(m)-\bar{\x}(m)\|_2\leq u_m.$

Next, let $\x^\ast\in\argmin_{\x\in\K}\sum_{t=1}^Tf_t(\x)$. For any $i,j\in V$, $m\in[B]$, and $t\in\mathcal{T}_m$, according to Assumptions \ref{assum4} and \ref{assum_sc}, we have
\begin{equation*}
\begin{split}
&f_{t,j}(\x_{i}(m))-f_{t,j}(\x^\ast)\\
\leq&f_{t,j}(\bar{\x}(m))+G\|\bar{\x}(m)-\x_i(m)\|_2-f_{t,j}(\x^\ast)\\
\leq&f_{t,j}(\x_{j}(m))+G\|\bar{\x}(m)-\x_j(m)\|_2-f_{t,j}(\x^\ast)+Gu_{m}\\
\leq&\nabla f_{t,j}(\x_{j}(m))^\top(\x_{j}(m)-\x^\ast)-\frac{\alpha}{2}\|\x_{j}(m)-\x^\ast\|_2^2+2Gu_{m}\\
=&\nabla f_{t,j}(\x_{j}(m))^\top(\x_{j}(m)-\bar{\x}(m))+\nabla f_{t,j}(\x_{j}(m))^\top(\bar{\x}(m)-\x^\ast)-\frac{\alpha}{2}\|\x_{j}(m)-\x^\ast\|_2^2+2Gu_{m}\\
\leq&\nabla f_{t,j}(\x_{j}(m))^\top(\bar{\x}(m)-\x^\ast)-\frac{\alpha}{2}\|\x_{j}(m)-\x^\ast\|_2^2+3Gu_{m}
\end{split}
\end{equation*}
where the third inequality is due to the strong convexity of $f_{t,j}(\x)$ and the last inequality is due to \[\nabla f_{t,j}(\x_{j}(m))^\top(\x_{j}(m)-\bar{\x}(m))\leq\|\nabla f_{t,j}(\x_{j}(m))\|_2\|\x_{j}(m)-\bar{\x}(m)\|_2\leq G u_m.\]
Moreover, we note that
\begin{align*}
\|\x_{j}(m)-\x^\ast\|_2^2=&\|\x_{j}(m)-\bar{\x}(m)\|_2^2+2\x_{j}(m)^\top(\bar{\x}(m)-\x^\ast)+\|\x^\ast\|_2^2-\|\bar{\x}(m)\|_2^2\\
\geq&2\x_{j}(m)^\top(\bar{\x}(m)-\x^\ast)+\|\x^\ast\|_2^2-\|\bar{\x}(m)\|_2^2.
\end{align*}
Therefore, for any $i,j\in V$, $m\in[B]$, and $t\in\mathcal{T}_m$, we have
\begin{equation*}
\begin{split}
&f_{t,j}(\x_{i}(m))-f_{t,j}(\x^\ast)-3Gu_{m}\\
\leq&(\nabla f_{t,j}(\x_{j}(m))-\alpha\x_j(m))^\top(\bar{\x}(m)-\x^\ast)-\frac{\alpha}{2}(\|\x^\ast\|_2^2-\|\bar{\x}(m)\|_2^2)
\end{split}
\end{equation*}
By summing over $t\in\mathcal{T}_m$ and $m\in[B]$, for any $i,j\in V$, we have
\begin{equation*}
\begin{split}
&\sum_{m=1}^B\sum_{t\in\mathcal{T}_m}(f_{t,j}(\x_{i}(m))-f_{t,j}(\x^\ast))-3G\sum_{m=1}^B\sum_{t\in\mathcal{T}_m}u_{m}\\
\leq&\sum_{m=1}^B\sum_{t\in\mathcal{T}_m}(\nabla f_{t,j}(\x_{j}(m))-\alpha\x_j(m))^\top(\bar{\x}(m)-\x^\ast)-\sum_{m=1}^B\sum_{t\in\mathcal{T}_m}\frac{\alpha}{2}(\|\x^\ast\|_2^2-\|\bar{\x}(m)\|_2^2)\\
=&\sum_{m=1}^B(\widehat{\g}_j(m)-\alpha K\x_{j}(m))^\top(\bar{\x}(m)-\x^\ast)-\sum_{m=1}^B\frac{\alpha K}{2}(\|\x^\ast\|_2^2-\|\bar{\x}(m)\|_2^2).
\end{split}
\end{equation*}
Furthermore, by summing over $j=1,\dots,n$, for any $i\in V$, we have
\begin{equation*}
\begin{split}
&\sum_{m=1}^B\sum_{t\in\mathcal{T}_m}\sum_{j=1}^n(f_{t,j}(\x_{i}(m))-f_{t,j}(\x^\ast))-3G\sum_{m=1}^B\sum_{t\in\mathcal{T}_m}\sum_{j=1}^nu_{m}\\
\leq&\sum_{m=1}^B\sum_{j=1}^n(\widehat{\g}_j(m)-\alpha K\x_{j}(m))^\top(\bar{\x}(m)-\x^\ast)-\frac{\alpha nK}{2}\sum_{m=1}^B(\|\x^\ast\|_2^2-\|\bar{\x}(m)\|_2^2)\\
=&n\sum_{m=1}^B\left(\bar{\dd}(m)^\top(\bar{\x}(m)-\x^\ast)-\frac{\alpha K}{2}(\|\x^\ast\|_2^2-\|\bar{\x}(m)\|_2^2)\right).
\end{split}
\end{equation*}
Then, it is easy to verify that
\begin{equation}
\label{thm_sc-eq6}
\begin{split}
R_{T,i}=&\sum_{m=1}^B\sum_{t\in\mathcal{T}_m}\sum_{j=1}^n(f_{t,j}(\x_{i}(m))-f_{t,j}(\x^\ast))\\
\leq&n\sum_{m=1}^B\left(\bar{\dd}(m)^\top(\bar{\x}(m)-\x^\ast)-\frac{\alpha K}{2}(\|\x^\ast\|_2^2-\|\bar{\x}(m)\|_2^2)\right)+3G\sum_{m=1}^B\sum_{t\in\mathcal{T}_m}\sum_{j=1}^nu_{m}\\
=&n\sum_{m=1}^B\left(\bar{\dd}(m)^\top(\bar{\x}(m)-\x^\ast)-\frac{\alpha K}{2}(\|\x^\ast\|_2^2-\|\bar{\x}(m)\|_2^2)\right)+3nGK\sum_{m=1}^Bu_{m}.
\end{split}
\end{equation}
To bound $\sum_{m=1}^B\left(\bar{\dd}(m)^\top(\bar{\x}(m)-\x^\ast)-\frac{\alpha K}{2}(\|\x^\ast\|_2^2-\|\bar{\x}(m)\|_2^2)\right)$, we introduce the following lemma.
\begin{lem}
\label{ftrl1}
(Lemma 2.3 in \citet{Online:suvery}) Let $\widehat{\x}_t^\ast=\argmin_{\x\in\K}\sum_{i=1}^{t-1}f_i(\x)+\mathcal{R}(\x),\forall t\in[T]$, where $\mathcal{R}(\x)$ is a strongly convex function. Then, $\forall \x\in\K$, it holds that
\begin{align*}
\sum_{t=1}^T\left(f_t(\widehat{\x}_t^\ast)-f_t(\x)\right)\leq\mathcal{R}(\x)-\mathcal{R}(\widehat{\x}_1^\ast)+\sum_{t=1}^T\left(f_t(\widehat{\x}_t^\ast)-f_t(\widehat{\x}_{t+1}^\ast)\right).
\end{align*}
\end{lem}
Before applying Lemma \ref{ftrl1}, we define $\widetilde{f}_m(\x)=\bar{\dd}(m)^\top\x+\frac{\alpha K}{2}\|\x\|_2^2.$ For any $\x\in\K$, it is easy to verify that
\begin{equation}
\label{thm_sc-eq7}
\begin{split}
\|\nabla\widetilde{f}_m(\x)\|_2=&\|\bar{\dd}(m)+\alpha K\x\|_2\leq\left\|\frac{1}{n}\sum_{j=1}^n\dd_j(m)\right\|_2+\|\alpha K\x\|_2\leq K(G+2\alpha R)
\end{split}
\end{equation}
where the last inequality is due to Assumption \ref{assum1} and (\ref{bound_d1}).

Moreover, according to the definition and (\ref{thm_sc-eq1}), for any $m\in[B]$, we have
\[
\bar{\x}(m+1)=\argmin\limits_{\x\in\K}\bar{\z}(m)^{\top}\mathbf{x}+\frac{(m-1)\alpha K}{2}\|\mathbf{x}\|_2^2+h\|\mathbf{x}-\mathbf{x}_{\ii}\|_2^2=\argmin\limits_{\x\in\K}\sum_{\tau=1}^{m-1}\widetilde{f}_\tau(\x)+h\|\mathbf{x}-\mathbf{x}_{\ii}\|_2^2.
\]
By applying Lemma \ref{ftrl1} with the loss functions $\{\widetilde{f}_m(\x)\}_{m=1}^B$, the decision set $\K$, and the regularizer $\mathcal{R}(\x)=h\|\mathbf{x}-\mathbf{x}_{\ii}\|_2^2$, we have
\begin{equation}
\label{thm_sc-eq8}
\begin{split}
&\sum_{m=1}^B\left(\widetilde{f}_m(\bar{\x}(m+1))-\widetilde{f}_m(\x^\ast)\right)\\
&\leq h\|\x^\ast-\mathbf{x}_{\ii}\|_2^2-h\|\bar{\x}(2)-\mathbf{x}_{\ii}\|_2^2+\sum_{m=1}^B\left(\widetilde{f}_m(\bar{\x}(m+1))-\widetilde{f}_m(\bar{\x}(m+2))\right)\\
&\leq4hR^2+\sum_{m=1}^B\nabla\widetilde{f}_m(\bar{\x}(m+1))^\top(\bar{\x}(m+1)-\bar{\x}(m+2))\\
&\leq4hR^2+\sum_{m=1}^BK(G+2\alpha R)\|\bar{\x}(m+1)-\bar{\x}(m+2)\|_2\\
\end{split}
\end{equation}
where the last inequality is due to the Cauchy-Schwarz inequality and (\ref{thm_sc-eq7}).

Note that $\bar{F}_{m+1}(\x)$ is $(m\alpha K+2h)$-strongly convex and $\bar{\x}(m+2)=\argmin_{\x\in\K}\bar{F}_{m+1}(\x)$. For any $m\in[B]$, we have
\begin{equation*}
\begin{split}
&\frac{m\alpha K+2h}{2}\|\bar{\x}(m+1)-\bar{\x}(m+2)\|_2^2\\
\leq& \bar{F}_{m+1}(\bar{\x}(m+1))-\bar{F}_{m+1}(\bar{\x}(m+2))\\
=&\bar{F}_{m}(\bar{\x}(m+1))+\widetilde{f}_m(\bar{\x}(m+1))-\bar{F}_{m}(\bar{\x}(m+2))-\widetilde{f}_m(\bar{\x}(m+2))\\
\leq&\widetilde{f}_m(\bar{\x}(m+1))-\widetilde{f}_m(\bar{\x}(m+2))\\
\leq&\nabla\widetilde{f}_m(\bar{\x}(m+1))^\top(\bar{\x}(m+1)-\bar{\x}(m+2))\\
\leq&K(G+2\alpha R)\|\bar{\x}(m+1)-\bar{\x}(m+2)\|_2
\end{split}
\end{equation*}
where the first inequality is due to (\ref{cor_scvx}), the second inequality is due to $\bar{\x}(m+1)=\argmin_{\x\in\K}\bar{F}_m(\x)$, and the last inequality is due to the Cauchy-Schwarz inequality and (\ref{thm_sc-eq7}).

For any $m\in[B]$, the above equality can be simplified as
\begin{equation*}
\begin{split}
\|\bar{\x}(m+1)-\bar{\x}(m+2)\|_2\leq&\frac{2K(G+2\alpha R)}{m\alpha K+2h}.
\end{split}
\end{equation*}
Then, by combining the above inequality with (\ref{thm_sc-eq8}), we have
\begin{equation}
\label{thm_sc-eq11}
\begin{split}
&\sum_{m=1}^B\left(\bar{\dd}(m)^\top(\bar{\x}(m)-\x^\ast)-\frac{\alpha K}{2}(\|\x^\ast\|_2^2-\|\bar{\x}(m)\|_2^2)\right)\\
=&\sum_{m=1}^B\left(\widetilde{f}_m(\bar{\x}(m))-\widetilde{f}_m(\x^\ast)\right)\\
=&\sum_{m=1}^B\left(\widetilde{f}_m(\bar{\x}(m))-\widetilde{f}_m(\bar{\x}(m+1))\right)+\sum_{m=1}^B\left(\widetilde{f}_m(\bar{\x}(m+1))-\widetilde{f}_m(\x^\ast)\right)\\
\leq&K(G+2\alpha R)\sum_{m=1}^B\left(\|\bar{\x}(m)-\bar{\x}(m+1)\|_2+\|\bar{\x}(m+1)-\bar{\x}(m+2)\|_2\right)+4hR^2\\
\leq&2K(G+2\alpha R)\sum_{m=1}^B\|\bar{\x}(m+1)-\bar{\x}(m+2)\|_2+4hR^2\\
\leq&\sum_{m=1}^B\frac{4K^2(G+2\alpha R)^2}{m\alpha K+2h}+4hR^2
\end{split}
\end{equation}
where the second inequality is due to $\bar{\x}(2)=\argmin_{\x\in\K}\bar{F}_1(\x)=\mathbf{x}_{\ii}=\bar{\x}(1)$ and $\|\bar{\x}(1)-\bar{\x}(2)\|_2=0\leq\|\bar{\x}(B+1)-\bar{\x}(B+2)\|_2$.

Finally, we complete the proof by substituting the definition of $u_m$ and (\ref{thm_sc-eq11}) into (\ref{thm_sc-eq6}).
\subsubsection{Proof of Theorem \ref{thm1-sc-2}}
\label{A1.2}
We note that $F_{m-1,i}(\x)$ is $((m-2)\alpha K+2h)$-smooth, and according to our Algorithm \ref{DBOCG-SC}, we have
\[
\mathbf{x}_{i}(m)=\text{CG}(\mathcal{K}, L, F_{m-1,i}(\x), \x_i(m-1)).
\]
Therefore, for any $B\geq m\geq2$, by applying Lemma \ref{lem_ILO}, it is easy to verify that
\begin{equation*}
\epsilon_m=\max_{i\in[n]}\left(F_{m-1,i}(\x_i(m))-\min_{\x\in\K}F_{m-1,i}(\x)\right)\leq\frac{8((m-2)\alpha K+2h) R^2}{L+2}.
\end{equation*}
By substituting the above inequality and $K(B-1)\leq KB=T$ into Theorem \ref{thm1-sc}, we have
\begin{equation*}
\begin{split}
R_{T,i}\leq\frac{12nGRT}{\sqrt{L+2}}+\sum_{m=2}^B\frac{3nGK^2(G+\alpha R)\sqrt{n}}{((m-2)\alpha K+2h) (1-\sigma_2(P))}+\sum_{m=1}^B\frac{4nK^2(G+2\alpha R)^2}{m\alpha K+2h}+4nhR^2.
\end{split}
\end{equation*}
\section{Lower Bounds}
In this section, we present lower bounds regarding the communication complexity for convex losses and strongly convex losses, respectively.
\subsection{Convex Losses}
Following previous studies \citep{D-ODA,wenpeng17}, when developing distributed online algorithms, we need to upper bound the regret of all local learners simultaneously. Correspondingly, to establish a lower regret bound for these distributed online algorithms, we actually only need to prove that the regret of one local learner has a lower bound. For simplicity, in the following, we will consider to derive a lower regret bound for the local learner 1.

For convex losses, we present a lower bound in the following theorem.
\begin{thm}
\label{thm_lowerB}
Suppose $\K=\left[-R/\sqrt{d},R/\sqrt{d}\right]^d$, which satisfies Assumption \ref{assum1} with $R=R$ and $r=R/\sqrt{d}$. For distributed OCO with $n>1$ local learners over $\K$ and any distributed online algorithm communicating $C$ rounds before the round $T$, there exists a sequence of local loss functions satisfying Assumption \ref{assum4} such that
\[R_{T,1}\geq\frac{nRGT}{2\sqrt{2(C+1)}}.\]
\end{thm}
\begin{proof}
In each round $t$, we simply set $f_{t,1}(\x)=0$ for the local learner $1$, and select $f_{t,2}(\x),\dots,f_{t,n}(\x)$ for other local learners with a more careful strategy. In this way, the global loss function is $f_t(\x)=f_{t,1}(\x)+\sum_{i=2}^nf_{t,i}(\x)=\sum_{i=2}^nf_{t,i}(\x)$. Since the local loss function $f_{t,i}(\x)$ is only revealed to the local learner $i\in[n]$, the local learner $1$ cannot access to the global loss unless it communicates with other local learners. In this way, we can maximize the impact of communication on the regret of the local learner $1$.

Without loss of generality, we denote the set of communication rounds by $\C=\{c_1,\dots,c_C\}$, where $1\leq c_1<\dots<c_C<T$. Let $c_0=0,c_{C+1}=T$. Then, we can divide the total $T$ rounds into the following $C+1$ intervals
\[[c_0+1,c_1],[c_1+1,c_2],\dots,[c_C+1,c_{C+1}].\]
For any $i\in\{0,\dots,C\}$ and $t\in[c_i+1,c_{i+1}]$, we will set $f_{t,2}(\x)=\dots=f_{t,n}(\x)=h_i(\x)$, which is revealed to the local learner $1$ after the decision $\x_1(c_{i+1})$ is made. In this way, the global loss can be written as $f_t(\x)=(n-1)h_i(\x)$ for any $i\in\{0,\dots,C\}$ and $t\in[c_i+1,c_{i+1}]$.

For any distributed online algorithm with communication rounds $\C=\{c_1,\dots,c_C\}$, we denote the sequence of decisions made by the local learner $1$ as $\x_1(1),\dots,\x_1(T)$. For any $i\in\{0,\dots,C\}$, we note that the decisions $\x_1(c_i+1),\dots,\x_1(c_{i+1})$ are made before the loss function $h_i(\x)$ is revealed.

Inspired by the lower bound for the general OCO \citep{Abernethy08}, we first utilize a randomized strategy to select $h_i(\x)$ for any $i\in\{0,\dots,C\}$, and derive an expected lower bound for $R_{T,1}$. Specifically, we independently select $h_i(\x)=\w_i^\top\x$ for any $i\in\{0,\dots,C\}$, where the coordinates of $\w_i$ are $\pm G/\sqrt{d}$ with probability $1/2$ and $h_i(\x)$ satisfies Assumption \ref{assum4}. Then, it is not hard to verify that
\begin{align*}
\E_{\w_0,\dots,\w_C}[R_{T,1}]=&\E_{\w_0,\dots,\w_C}\left[\sum_{t=1}^Tf_{t}(\x_1(t))-\min\limits_{\x\in\K}\sum_{t=1}^Tf_{t}(\x)\right]\\
=&\E_{\w_0,\dots,\w_C}\left[\sum_{i=0}^C\sum_{t=c_i+1}^{c_{i+1}}(n-1)h_i(\x_1(t))-\min\limits_{\x\in\K}\sum_{i=0}^C\sum_{t=c_i+1}^{c_{i+1}}(n-1)h_i(\x)\right]\\
=&(n-1)\E_{\w_0,\dots,\w_C}\left[\sum_{i=0}^C\sum_{t=c_i+1}^{c_{i+1}}\w_i^\top\x_1(t)-\min\limits_{\x\in\K}\sum_{i=0}^C(c_{i+1}-c_i)\w_i^\top\x\right]\\
=&(n-1)\E_{\w_0,\dots,\w_C}\left[-\min\limits_{\x\in\K}\sum_{i=0}^C(c_{i+1}-c_i)\w_i^\top\x\right]
\end{align*}
where the last equality is due to $\E_{\w_0,\dots,\w_C}[\w_i^\top\x_1(t)]=0$ for any $t\in[c_i+1,c_{i+1}]$.

Then, we have
\begin{align*}
\E_{\w_0,\dots,\w_C}[R_{T,1}]=&-(n-1)\E_{\w_0,\dots,\w_C}\left[\min\limits_{\x\in\K}\x^\top\sum_{i=0}^C(c_{i+1}-c_i)\w_i\right]\\
=&-(n-1)\E_{\w_0,\dots,\w_C}\left[\min\limits_{\x\in\left\{-R/\sqrt{d},R/\sqrt{d}\right\}^d}\x^\top\sum_{i=0}^C(c_{i+1}-c_i)\w_i\right]
\end{align*}
where the last equality is due to the fact that a linear function is minimized at the vertices of the cube.

Let $\epsilon_{01},\dots,\epsilon_{0d},\dots,\epsilon_{C1},\dots,\epsilon_{Cd}$ be independent and identically distributed variables with $\Pr(\epsilon_{ij}=\pm 1)=1/2$ for $i\in\{0,\dots,C\}$ and $j\in\{1,\dots,d\}$. Then, we have
\begin{equation}
\label{eq_lowerB}
\begin{split}
\E_{\w_0,\dots,\w_C}[R_{T,1}]=&-(n-1)\E_{\epsilon_{01},\dots,\epsilon_{Cd}}\left[\sum_{j=1}^d-\frac{R}{\sqrt{d}}\left|\sum_{i=0}^C(c_{i+1}-c_i)\frac{\epsilon_{ij}G}{\sqrt{d}}\right|\right]\\
=&(n-1)RG\E_{\epsilon_{01},\dots,\epsilon_{C1}}\left[\left|\sum_{i=0}^C(c_{i+1}-c_i)\epsilon_{i1} \right|\right]\\
\geq&\frac{(n-1)RG}{\sqrt{2}}\sqrt{\sum_{i=0}^C(c_{i+1}-c_i)^2}\geq\frac{(n-1)RG}{\sqrt{2}}\sqrt{\frac{(c_{C+1}-c_0)^2}{C+1}}\\
=&\frac{(n-1)RGT}{\sqrt{2(C+1)}}
\end{split}
\end{equation}
where the first inequality is due to the Khintchine inequality and the second inequality is due to the Cauchy-Schwarz inequality. The expected lower bound in (\ref{eq_lowerB}) implies that for any distributed online algorithm with communication rounds $\C=\{c_1,\dots,c_C\}$, there exists a particular choice of $\w_0,\dots,\w_C$ such that \[R_{T,1}\geq\frac{(n-1)RGT}{\sqrt{2(C+1)}}\geq\frac{nRGT}{2\sqrt{2(C+1)}}\]
where the last inequality is due to $n-1\geq n/2$ for any integer $n\geq2$.
\end{proof}
\begin{remark}
\label{rem4}
\emph{
Theorem \ref{thm_lowerB} essentially establishes an $\Omega(\sqrt{T})$ lower bound on the communication rounds required by any distributed online algorithm whose all local learners achieve the $O(T^{3/4})$ regret bound for convex losses, which matches (in terms of $T$) the $O(\sqrt{T})$ communication rounds required by our D-BOCG up to constant factors.
}
\end{remark}
\subsection{Strongly Convex Losses}
For strongly convex losses, we provide a lower bound in the following theorem.
\begin{thm}
\label{thm_lowerB-sc}
Suppose $\K=\left[-R/\sqrt{d},R/\sqrt{d}\right]^d$, which satisfies Assumption \ref{assum1} with $R=R$ and $r=R/\sqrt{d}$. For distributed OCO with $n>1$ local learners over $\K$ and any distributed online algorithm communicating at the end of $C$ rounds before the round $T$, there exists a sequence of local loss functions satisfying Assumption \ref{assum_sc} with $\alpha>0$ and Assumption \ref{assum4} with $G=2\alpha R$ respectively such that
\[R_{T,1}\geq\frac{\alpha nR^2T}{8(C+1)}.\]
\end{thm}
\begin{proof}
This proof is similar to that of Theorem \ref{thm_lowerB}. The main difference is to add a term $\frac{\alpha}{2}\|\x\|_2^2$ to previous local loss functions, which makes them $\alpha$-strongly convex.

For any distributed online algorithm with $C$ communication rounds, we still denote the set of communication rounds by $\C=\{c_1,\dots,c_C\}$ where $1\leq c_1<\dots<c_C<T$, and the sequence of decisions made by the local learner $1$ by $\x_1(1),\dots,\x_1(T)$. Let $c_0=0$ and $c_{C+1}=T$. Then, we can divide the total $T$ rounds into $C+1$ intervals \[[c_0+1,c_1],[c_1+1,c_2],\dots,[c_C+1,c_{C+1}].\]
In each round $t$, for the local learner $1$, we simply set $f_{t,1}(\x)=\frac{\alpha}{2}\|\x\|_2^2$ that satisfies Assumption \ref{assum4} with $G=2\alpha R$ and Assumption \ref{assum_sc}.
Moreover, for any $i\in\{0,\dots,C\}$ and $t\in[c_i+1,c_{i+1}]$, we set \[f_{t,2}(\x)=\dots=f_{t,n}(\x)=h_i(\x).\]
Specifically, we independently select \[h_i(\x)=\w_i^\top\x+\frac{\alpha}{2}\|\x\|^2_2\] for any $i\in\{0,\dots,C\}$, where the coordinates of $\w_i$ are $\pm \alpha R/\sqrt{d}$ with probability $1/2$. It is easy to verify that $h_i(\x)$ satisfies Assumption \ref{assum4} with $G=2\alpha R$ and Assumption \ref{assum_sc}, respectively.

Note that the local learner 1 does not communicate with other local learners between rounds $c_{i}+1$ and $c_{i+1}$. Therefore, the decisions $\x_1(c_i+1),\dots,\x_1(c_{i+1})$ are independent of $\w_i$.

Let $\bar{\w}=\frac{1}{\alpha T}\sum_{i=0}^C(c_{i+1}-c_i)\w_i$. In this way, the total loss for any $\x\in\K$ is equal to
\begin{equation}
\label{toal_loss-sc}
\begin{split}
\sum_{t=1}^Tf_t(\x)=&\sum_{t=1}^T\left(\sum_{j=2}^nf_{t,j}(\x)+\frac{\alpha n}{2}\|\x\|_2^2\right)\\
=&\sum_{i=0}^C(c_{i+1}-c_i)\left((n-1)\w_i^\top\x+\frac{\alpha n}{2}\|\x\|_2^2\right)\\
=&\alpha(n-1)T\bar{\w}^\top\x+\frac{\alpha nT}{2}\|\x\|_2^2\\
=&\frac{\alpha T}{2}\left(\left\|\sqrt{n}\x+\frac{(n-1)}{\sqrt{n}}\bar{\w}\right\|_2^2-\left\|\frac{(n-1)}{\sqrt{n}}\bar{\w}\right\|_2^2\right).
\end{split}
\end{equation}
Since the absolute value of each element in $\w_i$ is equal to $\alpha R/\sqrt{d}$, we note that the absolute value of each element in $-\frac{n-1}{n}\bar{\w}$ is bounded by
\[\frac{n-1}{n\alpha T}\sum_{i=0}^C\frac{(c_{i+1}-c_{i})\alpha R}{\sqrt{d}}=\frac{(n-1)R}{n\sqrt{d}}\leq\frac{R}{\sqrt{d}}\]
which implies that $-\frac{n-1}{n}\bar{\w}$ belongs to $\K=\left[-R/\sqrt{d},R/\sqrt{d}\right]^d$.

By combining with (\ref{toal_loss-sc}), we have
\begin{equation*}
\argmin_{\x\in\K}\sum_{t=1}^Tf_t(\x)=-\frac{n-1}{n}\bar{\w} \text{ and }\min_{\x\in\K}\sum_{t=1}^Tf_t(\x)=-\frac{\alpha T}{2}\left\|\frac{(n-1)}{\sqrt{n}}\bar{\w}\right\|_2^2.
\end{equation*}
Then, we have
\begin{equation}
\label{lower_eq_final-sc}
\begin{split}
&\E_{\w_0,\dots,\w_C}\left[\sum_{t=1}^Tf_{t}(\x_1(t))-\min\limits_{\x\in\K}\sum_{t=1}^Tf_{t}(\x)\right]\\
=&\E_{\w_0,\dots,\w_C}\left[\sum_{i=0}^C\sum_{t=c_i+1}^{c_{i+1}}\left((n-1)\w_i^\top\x_1(t)+\frac{\alpha n}{2}\|\x_1(t)\|_2^2\right)+\frac{\alpha T}{2}\left\|\frac{(n-1)}{\sqrt{n}}\bar{\w}\right\|_2^2\right]\\
\geq&\E_{\w_0,\dots,\w_C}\left[\sum_{i=0}^C\sum_{t=c_i+1}^{c_{i+1}}(n-1)\w_i^\top\x_1(t)+\frac{\alpha(n-1)^2T}{2n}\|\bar{\w}\|_2^2\right]\\
=&\E_{\w_0,\dots,\w_C}\left[\frac{\alpha(n-1)^2T}{2n}\|\bar{\w}\|_2^2\right]
\end{split}
\end{equation}
where the inequality is due to $\alpha\|\x\|_2^2\geq0$ for any $\x$ and the last equality is due to $\E_{\w_0,\dots,\w_C}[\w_i^\top\x_1(t)]=0$ for any $t\in[c_i+1,c_{i+1}]$.

Let $\epsilon_{01},\dots,\epsilon_{0d},\dots,\epsilon_{C1},\dots,\epsilon_{Cd}$ be independent and identically distributed variables with $\Pr(\epsilon_{ij}=\pm 1)=1/2$ for $i\in\{0,\dots,C\}$ and $j\in\{1,\dots,d\}$. Then, we have
\begin{equation}
\label{eq_lowerB-sc}
\begin{split}
&\E_{\w_0,\dots,\w_C}\left[\frac{\alpha(n-1)^2T}{2n}\left\|\bar{\w}\right\|_2^2\right]\\
=&\frac{(n-1)^2}{2\alpha nT}\E_{\w_0,\dots,\w_C}\left[\left\|\sum_{i=0}^C(c_{i+1}-c_i)\w_i\right\|_2^2\right]\\
=&\frac{(n-1)^2}{2\alpha nT}\E_{\epsilon_{01},\dots,\epsilon_{Cd}}\left[\sum_{j=1}^d\left|\sum_{i=0}^C(c_{i+1}-c_i)\frac{\epsilon_{ij}\alpha R}{\sqrt{d}}\right|^2\right]\\
=&\frac{\alpha(n-1)^2R^2}{2nT}\E_{\epsilon_{01},\dots,\epsilon_{C1}}\left[\left|\sum_{i=0}^C(c_{i+1}-c_i)\epsilon_{i1} \right|^2\right]\\
\geq&\frac{\alpha(n-1)^2R^2}{2nT}\sum_{i=0}^C(c_{i+1}-c_i)^2\\
\geq&\frac{\alpha(n-1)^2R^2}{2nT}\cdot\frac{(c_{C+1}-c_0)^2}{C+1}=\frac{\alpha(n-1)^2R^2T}{2n(C+1)}
\end{split}
\end{equation}
where the first inequality is due to the Khintchine inequality, and the second inequality is due to the Cauchy-Schwarz inequality.

By combining (\ref{lower_eq_final-sc}) with (\ref{eq_lowerB-sc}), we derive an expected lower bound as
\[\E_{\w_0,\dots,\w_C}[R_{T,1}]=\E_{\w_0,\dots,\w_C}\left[\sum_{t=1}^Tf_{t}(\x_1(t))-\min\limits_{\x\in\K}\sum_{t=1}^Tf_{t}(\x)\right]\geq\frac{\alpha(n-1)^2R^2T}{2n(C+1)}\]
which implies that for any distributed online algorithm with communication rounds $\C=\{c_1,\dots,c_C\}$, there exists a particular choice of $\w_0,\dots,\w_C$ such that \[R_{T,1}\geq\frac{\alpha(n-1)^2R^2T}{2n(C+1)}\geq\frac{\alpha nR^2T}{8(C+1)}\]
where the last inequality is due to $n-1\geq n/2$ for any integer $n\geq2$.
\end{proof}
\begin{remark}
\label{rem5}
\emph{
Theorem \ref{thm_lowerB-sc} essentially establishes an $\Omega\left(T^{1/3}(\log T)^{-1/3}\right)$ lower bound on the communication
rounds required by any distributed online algorithm whose all local learners achieve the $O(T^{2/3}(\log T)^{1/3})$ regret bound for strongly convex losses, which almost matches (in terms of $T$) the $O(T^{1/3}(\log T)^{2/3})$ communication rounds required by our D-BOCG up to polylogarithmic factors.
}
\end{remark}

\subsection{Discussions}
\label{sec5.3}
Besides the dependence on $T$, the lower bounds in our Theorems \ref{thm_lowerB} and \ref{thm_lowerB-sc} also depend on the network size $n$. When the number of communication rounds is limited to $O(\sqrt{T})$ and the losses are convex, Theorem \ref{thm_lowerB} provides a lower regret bound of $\Omega(nT^{3/4})$, but our D-BOCG only achieves a regret bound of $O(n^{5/4}(1-\sigma_2(P))^{-1/2}T^{3/4})$ as shown in Corollary \ref{cor2}. Similarly, when the number of communication rounds is limited to $O(T^{1/3}(\log T)^{2/3})$ and the losses are strongly convex, Theorem \ref{thm_lowerB-sc} provides a lower regret bound of $\Omega(nT^{2/3}(\log T)^{-2/3})$, but our D-BOCG only achieves a regret bound of $O(n^{3/2}(1-\sigma_2(P))^{-1}T^{2/3}(\log T)^{1/3})$ as shown in Corollary \ref{cor-sc}. So, in terms of the dependence on $n$ and $1-\sigma_2(P)$, there still exist some gaps between our upper bounds and lower bounds. To eliminate these gaps, one potential way is to reduce the dependence of the upper bounds on $n$, and establish lower bounds depending on the spectral gap $1-\sigma_2(P)$ by carefully considering the topology of the network, which is non-trivial and will be investigated in the future.

Moreover, we note that in the proof of Theorems \ref{thm_lowerB} and \ref{thm_lowerB-sc}, only the regret of the local learner 1 is analyzed. It is also interesting to ask whether the regret of other local learner $i\neq1$ simultaneously has a lower bound similar to that of the local learner 1. Unfortunately, the answer is negative when we utilize the sequence of local losses selected in the proof of Theorems \ref{thm_lowerB} and \ref{thm_lowerB-sc}. Let us consider a distributed online algorithm, which directly computes $\x_i(t+1)\in\argmin_{\x\in\K}f_{t,i}(\x)$. Following notations used in the proof of Theorem \ref{thm_lowerB}, for $\x^\ast\in\argmin_{\x\in\K}\sum_{t=1}^Tf_{t}(\x)$, the regret of its local learner $i\neq1$ can be upper bounded as
\begin{align*}
\sum_{t=1}^Tf_{t}(\x_i)-\sum_{t=1}^Tf_{t}(\x^\ast)=&\sum_{j=0}^C\sum_{t=c_j+1}^{c_{j+1}}(n-1)\w_j^\top(\x_i(t)-\x^\ast)\\
\leq&\sum_{j=0}^C(n-1)\w_j^\top(\x_i(c_j+1)-\x^\ast)\\
\leq&2(n-1)(C+1)RG
\end{align*}
where the first inequality is due to $\x_i(t)\in\argmin_{\x\in\K}\w_j^\top\x$ for $c_{j+1}\geq t>c_j+1$ and the last inequality is due to the fact that $\x_i(c_j+1)\in\K$ and the coordinates of $\w_j$ belong to $\pm G/\sqrt{d}$. This regret bound is smaller than the lower bound $nRGT/(2\sqrt{2(C+1)})$, when $C$ is small. A similar result can be derived when we use the sequence of local losses selected in the proof of Theorem \ref{thm_lowerB-sc}. However, as discussed before, deriving a lower bound for one local learner is sufficient in this paper. So, we leave the problem of simultaneously lower bounding the regret of all local learners as a future work.

\section{An Extension of D-BOCG to the Bandit Setting}
In this section, we extend our D-BOCG to the bandit setting, where only the loss value is available to each local learner. The main idea is to combine D-BOCG with the one-point gradient estimator \citep{OBO05}.
\subsection{The Algorithm}
By combining our D-BOCG with the one-point gradient estimator, our algorithm for the bandit setting is outlined in Algorithm \ref{DBBCG-SC}, and named as distributed block bandit conditional gradient (D-BBCG), where $0<\delta\leq r$ and
\[\K_\delta=(1-\delta/r)\K=\{(1-\delta/r)\x|\x\in\K\}.\]
Comparing D-BBCG with D-BOCG, there exist three differences as follows. First, in line 8 of D-BBCG, the actual decision $\y_{i}(t)$ is $\x_{i}(m)$ plus a random decision $\delta\uu_i(t)$, where $\uu_{i}(t)\sim \SS^d$. Second, in line 9 of D-BBCG, we can only observe the loss value $f_{t,i}(\y_{i}(t))$ instead of the gradient $\nabla f_{t,i}(\x_{i}(m))$, and adopt the one-point gradient estimator to approximate the gradient as \[\g_{i}(t)=\frac{d}{\delta}f_{t,i}(\y_{i}(t))\uu_{i}(t).\]
Third, to ensure $\y_{i}(t)\in\K$, in line 12 of D-BBCG, we perform
\[\mathbf{x}_{i}(m+1)=\text{CG}(\mathcal{K}_\delta, L, F_{m,i}(\x), \x_i(m))\]
by replacing $\K$ in line 11 of D-BOCG with a smaller set $\mathcal{K_\delta}\subseteq \K$, which limits $\x_{i}(m)$ in the set $\K_\delta$. Because of Assumption \ref{assum1} and $0<\delta\leq r$, it is easy to verify that $\x+\delta\uu\in\K$ for any $\x\in\K_\delta$ and $\uu\sim\mathcal{S}^d$ by utilizing the fact that $r\mathcal{B}^d\subseteq \K$.
\begin{algorithm}[t]
\caption{D-BBCG}
\label{DBBCG-SC}
\begin{algorithmic}[1]
\STATE \textbf{Input:} feasible set $\mathcal{K}$, $\delta$, $\mathbf{x}_{\ii}\in\mathcal{K}_\delta$, $\alpha$, $h$, $L$, and $K$
\STATE \textbf{Initialization:} choose $\mathbf{x}_1(1)=\dots=\mathbf{x}_n(1)=\mathbf{x}_{\ii}$ and set $\z_{1}(1)=\dots=\z_{n}(1)=\ze$
\FOR{$m=1,\dots,T/K$}
\FOR{each local learner $i\in V$}
\STATE define $F_{m,i}(\mathbf{x})=\z_{i}(m)^{\top}\mathbf{x}+\frac{(m-1)\alpha K}{2}\|\mathbf{x}\|_2^2+h\|\x-\mathbf{x}_{\ii}\|_2^2$
\STATE $\widehat{\g}_{i}(m)=\ze$
\FOR{$t=(m-1)K+1,\dots,mK$}
\STATE play $\y_i(t)=\x_{i}(m)+\delta\uu_i(t)$ where $\uu_i(t)\sim\SS^d$
\STATE observe $f_{t,i}(\y_{i}(t))$ and compute $\g_{i}(t)=\frac{d}{\delta}f_{t,i}(\y_{i}(t))\uu_{i}(t)$
\STATE $\widehat{\g}_{i}(m)=\widehat{\g}_{i}(m)+\g_{i}(t)$
\ENDFOR
\STATE $\mathbf{x}_{i}(m+1)=\text{CG}(\mathcal{K}_\delta, L, F_{m,i}(\x), \x_i(m))$ //This step can be executed \emph{in parallel} to the above \emph{for} loop.
\STATE $\z_{i}(m+1)=\sum_{j\in N_i}P_{ij}\z_{j}(m)+\widehat{\g}_{i}(m){-\alpha K\x_i(m)}$
\ENDFOR
\ENDFOR
\end{algorithmic}
\end{algorithm}
\subsection{Theoretical Guarantees}
In the following, we present theoretical guarantees of our D-BBCG. We first provide expected regret bounds of D-BBCG for convex losses and strongly convex losses, respectively.
\begin{thm}
\label{thm2-sc}
Let $\alpha=0$, $K=L=\sqrt{T}$, $h=\frac{n^{1/4}dMT^{3/4}}{\sqrt{1-\sigma_2(P)}R}$, $\delta=cT^{-1/4}$, where $c>0$ is a constant such that $\delta\leq r$. Under Assumptions \ref{assum4}, \ref{assum1}, \ref{assum5}, and \ref{assum2}, for any $i\in V$, Algorithm \ref{DBBCG-SC} ensures
\[\E\left[R_{T,i}\right]= O\left(n^{5/4}(1-\sigma_2(P))^{-1/2}T^{3/4}\right).\]
\end{thm}
\begin{thm}
\label{thm2-sc-c}
Let $\alpha>0$, $K=T^{2/3}(\ln T)^{-2/3}$, $\delta=cT^{-1/3}(\ln T)^{1/3}$, and $h=\alpha K$, where $c>0$ is a constant such that $\delta\leq r$. Under Assumptions \ref{assum4}, \ref{assum1}, \ref{assum5}, \ref{assum_sc}, and \ref{assum2}, for any $i\in V$, Algorithm \ref{DBBCG-SC} ensures
\[\E\left[R_{T,i}\right]=O\left(n^{3/2}(1-\sigma_2(P))^{-1}T^{2/3}(\log T)^{1/3}\right).\]
\end{thm}
\begin{remark}
\emph{
Theorems \ref{thm2-sc} and \ref{thm2-sc-c} show that D-BBCG can attain an expected regret bound of $O(T^{3/4})$ with $O(\sqrt{T})$ communication rounds for convex losses, and attain an expected regret bound of $O(T^{2/3}(\log T)^{1/3})$ with $O(T^{1/3}(\log T)^{2/3})$ communication rounds, which is similar to D-BOCG in the full information setting.
}
\end{remark}
Moreover, we show that D-BBCG enjoys a high-probability regret bound of $O(T^{3/4}(\log T)^{1/2})$ with $O(\sqrt{T})$ communication rounds for convex losses.
\begin{thm}
\label{pro_thm2}
Let $\alpha=0$, $h=\frac{n^{1/4}\xi_TdMT^{3/4}}{\sqrt{1-\sigma_2(P)}R}$, $\delta=cT^{-1/4}$, $K=L=\sqrt{T}$, where $c>0$ is a constant such that $\delta\leq r$, and $\xi_T=1+\sqrt{8\ln\frac{n\sqrt{T}}{\gamma}}$, where $0.5>\gamma>0$ is a constant. Under Assumptions \ref{assum4}, \ref{assum1}, \ref{assum5}, and \ref{assum2}, for any $i\in V$, with probability at least $1-2\gamma$, Algorithm \ref{DBBCG-SC} has
\[
R_{T,i}= O\left(n^{5/4}(1-\sigma_2(P))^{-1/2}T^{3/4}\xi_T\right).
\]
\end{thm}
\begin{remark}
\label{rem8}
\emph{
While the above theorem presents a high-probability regret bound for convex losses, it is hard to extend it for strongly convex losses. We note that according to the proof of Theorem \ref{pro_thm2}, the high-probability regret bound of D-BBCG has a term $O(K\sqrt{B\log(1/\gamma)})$, where $K$ is incurred by the delayed update mechanism and $\sqrt{B\log(1/\gamma)}$ is incurred by the application of the classical Azuma's concentration inequality \citep{Azuma67}. If we consider strongly convex losses, we would like to set $K=T^{2/3}(\ln T)^{-2/3}$ to control the communication complexity, but in this case the term $O(K\sqrt{B\log(1/\gamma)})$ is worse than the expected regret bound in Theorem \ref{thm2-sc-c}. Therefore, to extend the above theorem for for strongly convex losses, we may need some novel techniques to improve the term $O(K\sqrt{B\log(1/\gamma)})$, which will be investigated in the future.
}
\end{remark}
\subsection{Analysis}
In the following, we only provide the proof of Theorems \ref{thm2-sc} and \ref{thm2-sc-c}. The proof of Theorem \ref{pro_thm2} can be found in the Appendix.
\subsubsection{Proof of Theorems \ref{thm2-sc} and \ref{thm2-sc-c}}
Similar to the proof of Theorem \ref{thm1-sc}, we first define several auxiliary variables. Let $\bar{\z}(m)=\frac{1}{n}\sum_{i=1}^n\z_i(m)$ for $m\in[B+1]$, and let $\dd_i(m)=\widehat{\g}_i(m)-\alpha K\x_i(m)$ and $\bar{\dd}(m)=\frac{1}{n}\sum_{i=1}^n\dd_i(m)$ for $m\in[B]$. Then, we define $\bar{\x}(1)=\x_{\ii}$ and $\bar{\x}(m+1)=\argmin_{\x\in\K_\delta}\bar{F}_{m}(\x)$ for any $m\in[B+1]$, where \[\bar{F}_{m}(\mathbf{x})=\bar{\z}(m)^{\top}\mathbf{x}+\frac{(m-1)\alpha K}{2}\|\mathbf{x}\|_2^2+h\|\x-\mathbf{x}_{\ii}\|_2^2.\]Similarly, we define $\widehat{\x}_i(m+1)=\argmin_{\x\in\K_\delta}F_{m,i}(\x)$ for any $m\in[B+1]$, where \[F_{m,i}(\x)=\z_{i}(m)^{\top}\mathbf{x}+\frac{(m-1)\alpha K}{2}\|\mathbf{x}\|_2^2+h\|\x-\mathbf{x}_{\ii}\|_2^2.\] is defined in Algorithm \ref{DBBCG-SC}.

Moreover, we need to introduce the following lemmas.
\begin{lem}
\label{bound_d}
Let $\dd_i(m)=\widehat{\g}_i(m)-\alpha K\x_i(m)$ for $m\in[B]$. Under Assumptions \ref{assum4}, \ref{assum1}, and \ref{assum2}, for any $i\in V$ and $m\in[B]$, Algorithm \ref{DBBCG-SC} ensures that
\[\E[\|\dd_i(m)\|_2]^2\leq\E[\|\dd_i(m)\|_2^2]\leq2K\left(\frac{dM}{\delta}\right)^2+2K^2G^2+2(\alpha KR)^2.\]
\end{lem}
\begin{lem}
\label{graph_lem2_exp}
(Derived from the Proof of Lemma 6 in \citet{wenpeng17}) For any $i\in[n]$, let $\dd_i(1),\dots,\dd_i(m)\in\mathbb{R}^d$ be a sequence of vectors. Let $\z_i(1)=\ze$, $\z_{i}(m+1)=\sum_{j\in N_i}P_{ij}\z_{j}(m)+\dd_{i}(m)$, and $\bar{\z}(m)=\frac{1}{n}\sum_{i=1}^n\z_i(m)$ for $m\in[B]$, where $P$ satisfies Assumption \ref{assum5}. For any $i\in V$ and $m\in[B]$, assuming $\E[\|\dd_i(m)\|_2]\leq \widehat{G}$ where $\widehat{G}>0$ is a constant, we have \[\E[\|\z_i(m)-\bar{\z}(m)\|_2]\leq\frac{\widehat{G}\sqrt{n}}{1-\sigma_2(P)}.\]
\end{lem}
\begin{lem}
\label{smoothed_lem1}
(Lemma 2.6 in \citet{Hazan2016} and Lemma 6 in \citet{Wan-DOGD21})
Let $f(\x):\mathbb{R}^d\to\mathbb{R}$ be $\alpha$-strongly convex and $G$-Lipschitz over a convex and compact set $\K\subset\mathbb{R}^d$. Then, its $\delta$-smoothed version $\widehat{f}_\delta(\x)=\mathbb{E}_{\uu\sim\B^d}[f(\x+\delta\uu)]$ has the following properties.
\begin{itemize}
\item $\widehat{f}_\delta(\x)$ is $\alpha$-strongly convex over $\K_\delta$;
\item $|\widehat{f}_\delta(\x)-f(\x)|\leq\delta G$ for any $\x\in\K_\delta$;
\item $\widehat{f}_\delta(\x)$ is $G$-Lipschitz over $\K_\delta$.
\end{itemize}
\end{lem}
Now, we derive an upper bound of $\E[\|\x_i(m)-\bar{\x}(m)\|_2]$ for any $m\in[B]$. If $m=1$, according to the definition and Algorithm \ref{DBBCG-SC}, it is easy to verify that
\begin{equation}
\label{thm_bsc-eq2}
\E[\|\x_i(m)-\bar{\x}(m)\|_2]=\E[0]=0.
\end{equation}
For any $B\geq m\geq2$, we note that $F_{m-1,i}(\x)$ is $((m-2)\alpha K+2h)$-smooth, and Algorithm \ref{DBBCG-SC} ensures \[\mathbf{x}_{i}(m)=\text{CG}(\mathcal{K}_\delta, L, F_{m-1,i}(\x), \x_i(m-1)).\] According to Lemma \ref{lem_ILO} and Assumption \ref{assum1}, for $B\geq m\geq2$, it is easy to verify that
\[F_{m-1,i}(\x_i(m))-F_{m-1,i}(\widehat{\x}_i(m))\leq\frac{8((m-2)\alpha K+2h)R^2}{L+2}.\]
Then, for any $B\geq m\geq2$, it is easy to verify that \begin{equation}
\label{thm_bsc-eq3}
\begin{split}
\|\x_i(m)-\bar{\x}(m)\|_2\leq&\|\x_i(m)-\widehat{\x}_i(m)\|_2+\|\widehat{\x}_i(m)-\bar{\x}(m)\|_2\\
\leq&\sqrt{\frac{2F_{m-1,i}(\x_i(m))-2F_{m-1,i}(\widehat{\x}_i(m))}{(m-2)\alpha K+2h}}+\|\widehat{\x}_i(m)-\bar{\x}(m)\|_2\\
\leq&\frac{4R}{\sqrt{L+2}}+\|\widehat{\x}_i(m)-\bar{\x}(m)\|_2
\end{split}
\end{equation}
where the second inequality is due to the fact that $F_{m-1,i}(\x)$ is also $((m-2)\alpha K+2h)$-strongly convex and (\ref{cor_scvx}).

Moreover, for any $B\geq m\geq2$, similar to (\ref{reuse_eq1}) and (\ref{reuse_eq2}), we have
\begin{equation*}
\begin{split}
\widehat{\x}_i(m)=&\argmin_{\x\in\K_\delta}\z_{i}(m-1)^{\top}\mathbf{x}+\frac{(m-2)\alpha K}{2}\|\mathbf{x}\|_2^2+h\|\x-\mathbf{x}_{\ii}\|_2^2\\
=&\argmin\limits_{\x\in\K_\delta}\frac{2}{(m-2)\alpha K+2h}(\z_{i}(m-1)-2h\mathbf{x}_{\ii})^{\top}\mathbf{x}+\|\mathbf{x}\|_2^2
\end{split}
\end{equation*}
and
\begin{equation*}
\begin{split}
\bar{\x}(m)
=&\argmin_{\x\in\K_\delta}\bar{\z}(m-1)^{\top}\mathbf{x}+\frac{(m-2)\alpha K}{2}\|\mathbf{x}\|_2^2+h\|\x-\mathbf{x}_{\ii}\|_2^2\\
=&\argmin\limits_{\x\in\K_\delta}\frac{2}{(m-2)\alpha K+2h}(\bar{\z}(m-1)-2h\mathbf{x}_{\ii})^{\top}\mathbf{x}+\|\mathbf{x}\|_2^2.
\end{split}
\end{equation*}
Therefore, for any $B\geq m\geq2$, by applying Lemma \ref{dual_lem1}, we have
\begin{align*}\|\widehat{\x}_i(m)-\bar{\x}(m)\|_2&\leq\frac{\|\z_i(m-1)-2h\mathbf{x}_{\ii}-\bar{\z}(m-1)+2h\mathbf{x}_{\ii}\|_2}{(m-2)\alpha K+2h}\\
&=\frac{\|\z_i(m-1)-\bar{\z}(m-1)\|_2}{(m-2)\alpha K+2h}
\end{align*}
By further combining with (\ref{thm_bsc-eq3}), for any $B\geq m\geq2$, we have
\begin{equation}
\label{thm_bsc-eq4}
\begin{split}
&\E[\|\x_i(m)-\bar{\x}(m)\|_2]\\
\leq&\frac{4R}{\sqrt{L+2}}+\E[\|\widehat{\x}_i(m)-\bar{\x}(m)\|_2]\\
\leq&\frac{4R}{\sqrt{L+2}}+\frac{\E[\|\z_i(m-1)-\bar{\z}(m-1)\|_2]}{(m-2)\alpha K+2h}\\
\leq&\frac{4R}{\sqrt{L+2}}+\sqrt{2K\left(\frac{dM}{\delta}\right)^2+2K^2G^2+2(\alpha KR)^2}\frac{\sqrt{n}}{((m-2)\alpha K+2h)(1-\sigma_2(P))}
\end{split}
\end{equation}
where the last inequality is due to
\[
\E[\|\z_i(m-1)-\bar{\z}(m-1)\|_2]\leq\sqrt{2K\left(\frac{dM}{\delta}\right)^2+2K^2G^2+2(\alpha KR)^2}\frac{\sqrt{n}}{1-\sigma_2(P)}.
\]
which is derived by combining Lemma \ref{bound_d} with Lemma \ref{graph_lem2_exp}.

Let $u_1=0$ and \[u_m=\frac{4R}{\sqrt{L+2}}+\sqrt{2K\left(\frac{dM}{\delta}\right)^2+2K^2G^2+2(\alpha KR)^2}\frac{\sqrt{n}}{((m-2)\alpha K+2h)(1-\sigma_2(P))}\] for any $B\geq m\geq2$. From (\ref{thm_bsc-eq2}) and (\ref{thm_bsc-eq4}), for any $m\in[B]$, it holds that
\begin{equation*}
\E[\|\x_i(m)-\bar{\x}(m)\|_2]\leq u_m.
\end{equation*}
Next, let $\x^\ast\in\argmin_{\x\in\K}\sum_{t=1}^Tf_{t}(\x)$, $\widetilde{\x}^\ast=(1-\delta/r)\x^\ast$, and $\widehat{f}_{t,j,\delta}(\x)$ denote the $\delta$-smoothed version of $f_{t,j}(\x)$.
For any $i,j\in V$, $m\in[B]$, and $t\in\mathcal{T}_m$, by applying Lemma \ref{smoothed_lem1}, we have
\begin{equation}
\label{thm_bsc-eq1}
\begin{split}
&\E[\widehat{f}_{t,j,\delta}(\x_{i}(m))-\widehat{f}_{t,j,\delta}(\widetilde{\x}^\ast)]\\
\leq&\E[\widehat{f}_{t,j,\delta}(\bar{\x}(m))-\widehat{f}_{t,j,\delta}(\widetilde{\x}^\ast)+G\|\bar{\x}(m)-\x_i(m)\|_2]\\
\leq&\E[\widehat{f}_{t,j,\delta}(\x_{j}(m))-\widehat{f}_{t,j,\delta}(\widetilde{\x}^\ast)+G\|\bar{\x}(m)-\x_j(m)\|_2]+Gu_m\\
\leq&\E\left[\nabla\widehat{f}_{t,j,\delta}(\x_{j}(m))^\top(\x_{j}(m)-\widetilde{\x}^\ast)-\frac{\alpha}{2}\|\x_{j}(m)-\widetilde{\x}^\ast\|_2^2\right]+2Gu_m\\
\leq&\E\left[\nabla\widehat{f}_{t,j,\delta}(\x_{j}(m))^\top(\bar{\x}(m+1)-\widetilde{\x}^\ast)-\frac{\alpha}{2}\|\x_{j}(m)-\widetilde{\x}^\ast\|_2^2\right]+2Gu_m\\
&+\E[\nabla\widehat{f}_{t,j,\delta}(\x_{j}(m))^\top(\x_{j}(m)-\bar{\x}(m+1))]\\
\leq&\E\left[\nabla\widehat{f}_{t,j,\delta}(\x_{j}(m))^\top(\bar{\x}(m+1)-\widetilde{\x}^\ast)-\frac{\alpha}{2}\|\x_{j}(m)-\widetilde{\x}^\ast\|_2^2\right]+2Gu_m\\
&+\E[\|\nabla\widehat{f}_{t,j,\delta}(\x_{j}(m))\|_2\|\x_{j}(m)-\bar{\x}(m+1)\|_2]\\
\leq&\E\left[\nabla\widehat{f}_{t,j,\delta}(\x_{j}(m))^\top(\bar{\x}(m+1)-\widetilde{\x}^\ast)-\frac{\alpha}{2}\|\x_{j}(m)-\widetilde{\x}^\ast\|_2^2\right]+2Gu_m\\
&+\E[G(\|\x_{j}(m)-\bar{\x}(m)\|_2+\|\bar{\x}(m)-\bar{\x}(m+1)\|_2)]\\
\leq&\E\left[(\nabla\widehat{f}_{t,j,\delta}(\x_{j}(m))-\alpha\x_{j}(m))^\top(\bar{\x}(m+1)-\widetilde{\x}^\ast)-\frac{\alpha}{2}(\|\widetilde{\x}^\ast\|_2^2-\|\bar{\x}(m+1)\|_2^2)\right]\\
&+\E\left[G\|\bar{\x}(m)-\bar{\x}(m+1)\|_2\right]+3Gu_{m}
\end{split}
\end{equation}
where the first two inequalities are due to the fact that $\widehat{f}_{t,j,\delta}(\x)$ is $G$-Lipschitz over $\K_\delta$, the third inequality is due to the strong convexity of $\widehat{f}_{t,j,\delta}(\x)$, and the last inequality is due to
\begin{align*}
\|\x_{j}(m)-\widetilde{\x}^\ast\|_2^2
=&\|\x_{j}(m)-\bar{\x}(m+1)\|_2^2+2\x_{j}(m)^\top(\bar{\x}(m+1)-\widetilde{\x}^\ast)+\|\widetilde{\x}^\ast\|_2^2-\|\bar{\x}(m+1)\|_2^2\\
\geq&2\x_{j}(m)^\top(\bar{\x}(m+1)-\widetilde{\x}^\ast)+\|\widetilde{\x}^\ast\|_2^2-\|\bar{\x}(m+1)\|_2^2.
\end{align*}
Moreover, it is not hard to verify that
\begin{equation}
\label{thm_bsc_new_eq1}
\begin{split}
R_{T,i}=&\sum_{m=1}^B\sum_{t\in\mathcal{T}_m}\sum_{j=1}^n(f_{t,j}(\x_{i}(m)+\delta \uu_{i}(t))-f_{t,j}(\x^\ast))\\
\leq&\sum_{m=1}^B\sum_{t\in\mathcal{T}_m}\sum_{j=1}^n((f_{t,j}(\x_{i}(m))+G\|\delta \uu_{i}(t)\|_2)-(f_{t,j}(\widetilde{\x}^\ast)-G\|\widetilde{\x}^\ast-\x^\ast\|_2))\\
\leq&\sum_{m=1}^B\sum_{t\in\mathcal{T}_m}\sum_{j=1}^n\left(f_{t,j}(\x_{i}(m))-f_{t,j}(\widetilde{\x}^\ast)+G\|\delta \uu_{i}(t)\|_2+\frac{\delta GR}{r}\right)\\
\leq&\sum_{m=1}^B\sum_{t\in\mathcal{T}_m}\sum_{j=1}^n((\widehat{f}_{t,j,\delta}(\x_{i}(m))+\delta G)-(\widehat{f}_{t,j,\delta}(\widetilde{\x}^\ast)-\delta G))+\delta nGT+\frac{\delta nGRT}{r}\\
=&\sum_{m=1}^B\sum_{t\in\mathcal{T}_m}\sum_{j=1}^n(\widehat{f}_{t,j,\delta}(\x_{i}(m))-\widehat{f}_{t,j,\delta}(\widetilde{\x}^\ast))+3\delta nGT+\frac{\delta nGRT}{r}
\end{split}
\end{equation}
where the first inequality is due to Assumption \ref{assum4}, the second inequality is due to $\x^\ast\in\K$ and Assumption \ref{assum1}, and the third inequality is due to Lemma \ref{smoothed_lem1}.

By combining (\ref{thm_bsc-eq1}) with (\ref{thm_bsc_new_eq1}), we have
\begin{equation}
\label{thm_bsc-pre0-eq6}
\begin{split}
&\E\left[R_{T,i}\right]\\
\leq&\sum_{m=1}^B\sum_{t\in\mathcal{T}_m}\sum_{j=1}^n\E\left[(\nabla\widehat{f}_{t,j,\delta}(\x_{j}(m))-\alpha\x_{j}(m))^\top(\bar{\x}(m+1)-\widetilde{\x}^\ast)-\frac{\alpha}{2}(\|\widetilde{\x}^\ast\|_2^2-\|\bar{\x}(m+1)\|_2^2)\right]\\
&+nKG\sum_{m=1}^B\E\left[\|\bar{\x}(m)-\bar{\x}(m+1)\|_2\right]+3nKG\sum_{m=1}^B{u_m}+3\delta nGT+\frac{\delta nGRT}{r}.
\end{split}
\end{equation}
Let $\widetilde{f}_m(\x)=\bar{\dd}(m)^\top\x+\frac{\alpha K}{2}\|\x\|_2^2.$ Due to Lemma \ref{smoothed_lem2}, we have
\begin{equation}
\label{thm_bsc-pre-eq6}
\begin{split}
&\sum_{m=1}^B\sum_{t\in\mathcal{T}_m}\sum_{j=1}^n\E\left[(\nabla\widehat{f}_{t,j,\delta}(\x_{j}(m))-\alpha\x_{j}(m))^\top(\bar{\x}(m+1)-\widetilde{\x}^\ast)-\frac{\alpha}{2}(\|\widetilde{\x}^\ast\|_2^2-\|\bar{\x}(m+1)\|_2^2)\right]\\
=&\sum_{m=1}^B\sum_{j=1}^n\E\left[(\widehat{\g}_j(m)-\alpha K\x_{j}(m))^\top(\bar{\x}(m+1)-\widetilde{\x}^\ast)-\frac{\alpha K}{2}(\|\widetilde{\x}^\ast\|_2^2-\|\bar{\x}(m+1)\|_2^2)\right]\\
=&n\sum_{m=1}^B\E\left[\bar{\dd}(m)^\top(\bar{\x}(m+1)-\widetilde{\x}^\ast)-\frac{\alpha K}{2}(\|\widetilde{\x}^\ast\|_2^2-\|\bar{\x}(m+1)\|_2^2)\right]\\
=&n\sum_{m=1}^B\E\left[\widetilde{f}_m(\bar{\x}(m+1))-\widetilde{f}_m(\widetilde{\x}^\ast)\right]
\end{split}
\end{equation}
According to the definition and (\ref{thm_sc-eq1}), we have \[\bar{\x}(m+1)=\argmin\limits_{\x\in\K_\delta}\bar{\z}(m)^{\top}\mathbf{x}+\frac{(m-1)\alpha K}{2}\|\mathbf{x}\|_2^2+h\|\x-\x_{\ii}\|_2^2=\argmin\limits_{\x\in\K_\delta}\sum_{\tau=1}^{m-1}\widetilde{f}_\tau(\x)+h\|\x-\x_{\ii}\|_2^2.\]
By applying Lemma \ref{ftrl1} with the loss functions $\{\widetilde{f}_m(\x)\}_{m=1}^B$, the decision set $\K_\delta$, and the regularizer $\mathcal{R}(\x)=h\|\mathbf{x}-\mathbf{x}_{\ii}\|_2^2$, we have
\begin{equation}
\label{thm_bsc-eq6}
\begin{split}
&\sum_{m=1}^B\left(\widetilde{f}_m(\bar{\x}(m+1))-\widetilde{f}_m(\widetilde{\x}^\ast)\right)\\
\leq& h\|\widetilde{\x}^\ast-\mathbf{x}_{\ii}\|_2^2+h\|\bar{\x}(2)-\mathbf{x}_{\ii}\|_2^2+\sum_{m=1}^B\left(\widetilde{f}_m(\bar{\x}(m+1))-\widetilde{f}_m(\bar{\x}(m+2))\right)\\
\leq&4hR^2+\sum_{m=1}^B\nabla\widetilde{f}_m(\bar{\x}(m+1))^\top(\bar{\x}(m+1)-\bar{\x}(m+2))\\
\leq&4hR^2+\sum_{m=1}^B\|\nabla\widetilde{f}_m(\bar{\x}(m+1))\|_2\|\bar{\x}(m+1)-\bar{\x}(m+2)\|_2\\
\leq&4hR^2+\sum_{m=1}^B\|\bar{\dd}(m)+\alpha K\bar{\x}(m+1)\|_2\|\bar{\x}(m+1)-\bar{\x}(m+2)\|_2.
\end{split}
\end{equation}
Note that $\bar{F}_{m+1}(\x)$ is $(m\alpha K+2h)$-strongly convex and $\bar{\x}(m+2)=\argmin_{\x\in\K_\delta}\bar{F}_{m+1}(\x)$. For any $m\in[B]$, we have
\begin{equation*}
\begin{split}
&\frac{m\alpha K+2h}{2}\|\bar{\x}(m+1)-\bar{\x}(m+2)\|_2^2\\
\leq& \bar{F}_{m+1}(\bar{\x}(m+1))-\bar{F}_{m+1}(\bar{\x}(m+2))\\
=&\bar{F}_{m}(\bar{\x}(m+1))+\widetilde{f}_m(\bar{\x}(m+1))-\bar{F}_{m}(\bar{\x}(m+2))-\widetilde{f}_m(\bar{\x}(m+2))\\
\leq&\nabla\widetilde{f}_m(\bar{\x}(m+1))^\top(\bar{\x}(m+1)-\bar{\x}(m+2))\\
\leq&\|\bar{\dd}(m)+\alpha K\bar{\x}(m+1)\|_2\|\bar{\x}(m+1)-\bar{\x}(m+2)\|_2
\end{split}
\end{equation*}
where the first inequality is due to (\ref{cor_scvx}) and the second inequality is due to $\bar{\x}(m+1)=\argmin_{\x\in\K_\delta}\bar{F}_{m}(\x)$ and the convexity of $\widetilde{f}_m(\x)$.

Moreover, for any $m\in[B]$, the above inequality can be simplified as
\begin{equation}
\label{thm_bsc-eq7}
\|\bar{\x}(m+1)-\bar{\x}(m+2)\|_2\leq\frac{2\|\bar{\dd}(m)+\alpha K\bar{\x}(m+1)\|_2}{m\alpha K+2h}.
\end{equation}
By combining (\ref{thm_bsc-pre0-eq6}), (\ref{thm_bsc-pre-eq6}), (\ref{thm_bsc-eq6}), and (\ref{thm_bsc-eq7}), we have
\begin{equation}
\label{thm_bsc-eq8}
\begin{split}
&\E\left[R_{T,i}\right]\\
\leq&4nhR^2+n\sum_{m=1}^B\E\left[\frac{2\|\bar{\dd}(m)+\alpha K\bar{\x}(m+1)\|_2^2}{m\alpha K+2h}\right]+nKG\sum_{m=1}^B\E\left[\|\bar{\x}(m)-\bar{\x}(m+1)\|_2\right]\\
&+3nKG\sum_{m=1}^B{u_m}+3\delta nGT+\frac{\delta nGRT}{r}\\
\leq&n\sum_{m=1}^B\E\left[\frac{2\|\bar{\dd}(m)+\alpha K\bar{\x}(m+1)\|_2^2}{m\alpha K+2h}\right]+nKG\sum_{m=2}^B\E\left[\frac{2\|\bar{\dd}(m-1)+\alpha K\bar{\x}(m)\|_2}{(m-1)\alpha K+2h}\right]\\
&+3nKG\sum_{m=1}^B{u_m}+3\delta nGT+\frac{\delta nGRT}{r}+4nhR^2\\
\leq&n\sum_{m=1}^B\E\left[\frac{2\|\bar{\dd}(m)+\alpha K\bar{\x}(m+1)\|_2^2}{m\alpha K+2h}\right]+nKG\sum_{m=1}^B\E\left[\frac{2\|\bar{\dd}(m)+\alpha K\bar{\x}(m+1)\|_2}{m\alpha K+2h}\right]\\
&+3nKG\sum_{m=1}^B{u_m}+3\delta nGT+\frac{\delta nGRT}{r}+4nhR^2
\end{split}
\end{equation}
where the second inequality is derived by  bounding $\|\bar{\x}(m)-\bar{\x}(m+1)\|_2$ using (\ref{thm_bsc-eq7}) for $m>1$ and $\bar{\x}(2)=\argmin_{\x\in\K_\delta}\bar{F}_1(\x)=\mathbf{x}_{\ii}=\bar{\x}(1)$ for $m=1$.

With the above inequality, we can establish the specific regret bound for convex losses and strongly convex losses, respectively.
\paragraph{Convex Losses} We first consider the case with convex losses, in which the parameters of our Algorithm \ref{DBBCG-SC} are set to $\alpha=0$, $K=L=\sqrt{T}$, $h=\frac{n^{1/4}dMT^{3/4}}{\sqrt{1-\sigma_2(P)}R}$, $\delta=cT^{-1/4}$.

Because of $\alpha=0$, $K=\sqrt{T}$, and $\delta=cT^{-1/4}$, we have
\begin{equation*}
\begin{split}
\E[\|\bar{\dd}(m)+\alpha K\bar{\x}(m+1)\|_2^2]=&\E[\|\bar{\dd}(m)\|_2^2]\leq2K\left(\frac{dM}{\delta}\right)^2+2K^2G^2+2(\alpha KR)^2\\
=&\left(\frac{2d^2M^2}{c^2}+2G^2\right)T
\end{split}
\end{equation*}
where the first inequality is due to Lemma \ref{bound_d}.

Therefore, with $\alpha=0$, $K=\sqrt{T}$, $h=\frac{n^{1/4}dMT^{3/4}}{\sqrt{1-\sigma_2(P)}R}$, and $\delta=cT^{-1/4}$, we have
\begin{equation}
\label{thm_bsc-eq9}
\begin{split}
n\sum_{m=1}^B\E\left[\frac{2\|\bar{\dd}(m)+\alpha K\bar{\x}(m+1)\|_2^2}{m\alpha K+2h}\right]
\leq&\left(\frac{d^2M^2}{c^2}+G^2\right)\frac{2n^{3/4}\sqrt{1-\sigma_2(P)}RT^{3/4}}{dM}\\
=&O(n^{3/4}T^{3/4}).
\end{split}
\end{equation}
Similarly, with $\alpha=0$, $K=\sqrt{T}$, $h=\frac{n^{1/4}dMT^{3/4}}{\sqrt{1-\sigma_2(P)}R}$, and $\delta=cT^{-1/4}$, we have
\begin{equation}
\label{thm_bsc-eq10}
\begin{split}
nKG\sum_{m=1}^B\E\left[\frac{2\|\bar{\dd}(m)+\alpha K\bar{\x}(m+1)\|_2}{m\alpha K+2h}\right]\leq&\sqrt{\frac{2d^2M^2}{c^2}+2G^2}\frac{n^{3/4}\sqrt{1-\sigma_2(P)}GRT^{3/4}}{dM}\\
=&O(n^{3/4}T^{3/4}).
\end{split}
\end{equation}
Note that $u_1=0$ and $u_m=\frac{4R}{\sqrt{L+2}}+\sqrt{2K\left(\frac{dM}{\delta}\right)^2+2K^2G^2+2(\alpha KR)^2}\frac{\sqrt{n}}{((m-2)\alpha K+2h)(1-\sigma_2(P))}$ for any $B\geq m\geq2$. With $\alpha=0$, $K=L=\sqrt{T}$, $h=\frac{n^{1/4}dMT^{3/4}}{\sqrt{1-\sigma_2(P)}R}$, and $\delta=cT^{-1/4}$, we have
\begin{equation}
\label{thm_bsc-eq11}
\begin{split}
3nKG\sum_{m=1}^B{u_m}=&3nKG\sum_{m=2}^B\left(\frac{4R}{\sqrt{L+2}}+\sqrt{2K\left(\frac{dM}{\delta}\right)^2+2K^2G^2}\frac{\sqrt{n}}{2h(1-\sigma_2(P))}\right)\\
\leq&\frac{12nK(B-1)GR}{\sqrt{L+2}}+\sqrt{2K\left(\frac{dM}{\delta}\right)^2+2K^2G^2}\frac{3n^{3/2}K(B-1)G}{2h(1-\sigma_2(P))}\\
\leq&12nGRT^{3/4}+\sqrt{\frac{2d^2M^2}{c^2}+2G^2}\frac{3n^{5/4}GRT^{3/4}}{2dM\sqrt{1-\sigma_2(P)}}\\
=&O\left(n^{5/4}(1-\sigma_2(P))^{-1/2}T^{3/4}\right)
\end{split}
\end{equation}
Moreover, with $K=\sqrt{T}$, $h=\frac{n^{1/4}dMT^{3/4}}{\sqrt{1-\sigma_2(P)}R}$, and $\delta=cT^{-1/4}$, we have
\begin{equation}
\label{thm_bsc-eq12}
\begin{split}
3\delta nGT+\frac{\delta nGRT}{r}+4nhR^2=&3cnGT^{3/4}+\frac{cnGRT^{3/4}}{r}+\frac{4n^{5/4}dMRT^{3/4}}{\sqrt{1-\sigma_2(P)}}\\
=&O\left(n^{5/4}(1-\sigma_2(P))^{-1/2}T^{3/4}\right).
\end{split}
\end{equation}
By combining (\ref{thm_bsc-eq8}), (\ref{thm_bsc-eq9}), (\ref{thm_bsc-eq10}), (\ref{thm_bsc-eq11}), and (\ref{thm_bsc-eq12}), our Algorithm \ref{DBBCG-SC} with $\alpha=0$, $K=L=\sqrt{T}$, $h=\frac{n^{1/4}dMT^{3/4}}{\sqrt{1-\sigma_2(P)}R}$, $\delta=cT^{-1/4}$ ensures
\[\E\left[R_{T,i}\right]= O\left(n^{5/4}(1-\sigma_2(P))^{-1/2}T^{3/4}\right)\]
for convex losses, which completes the proof of Theorem \ref{thm2-sc}.

\paragraph{Strongly Convex Losses}
We continue to consider the case with the strongly convex losses, in which the parameters of our Algorithm \ref{DBBCG-SC} are set to $\alpha>0$, $K=L=T^{2/3}(\ln T)^{-2/3}$, $\delta=cT^{-1/3}(\ln T)^{1/3}$, and $h=\alpha K$.

With $K=T^{2/3}(\ln T)^{-2/3}$ and $\delta=cT^{-1/3}(\ln T)^{1/3}$, we have
\begin{equation*}
\begin{split}
\E[\|\bar{\dd}(m)+\alpha K\bar{\x}(m)\|_2]^2\leq&\E[\|\bar{\dd}(m)+\alpha K\bar{\x}(m)\|_2^2]\leq\E[2\|\bar{\dd}(m)\|_2^2]+\E[2\|\alpha K\bar{\x}(m)\|_2^2]\\
\leq&4K\left(\frac{dM}{\delta}\right)^2+4K^2G^2+6(\alpha KR)^2\\
=&\left(\frac{4d^2M^2}{c^2}+4G^2+6\alpha^2R^2\right)\left(\frac{T}{\ln T}\right)^{4/3}
\end{split}
\end{equation*}
where the third inequality is due to Lemma \ref{bound_d} and Assumption \ref{assum1}.

For brevity, let $C=\frac{4d^2M^2}{c^2}+4G^2+6\alpha^2R^2$. With $\alpha>0$, $K=T^{2/3}(\ln T)^{-2/3}$, $\delta=cT^{-1/3}(\ln T)^{1/3}$, and $h=\alpha K$, we have
\begin{equation}
\label{thm_bsc-eq9-c}
\begin{split}
&n\sum_{m=1}^B\E\left[\frac{2\|\bar{\dd}(m)+\alpha K\bar{\x}(m+1)\|_2^2}{m\alpha K+2h}\right]\\
\leq&\frac{2nC}{\alpha K}\left(\frac{T}{\ln T}\right)^{4/3}\sum_{m=1}^B\frac{1}{m+2}\leq\frac{2nC}{\alpha K}\left(\frac{T}{\ln T}\right)^{4/3}\sum_{m=1}^B\frac{1}{m}\leq\frac{2nCT^{2/3}}{\alpha(\ln T)^{2/3}}(1+\ln B)\\
\leq&\frac{2nCT^{2/3}}{\alpha(\ln T)^{2/3}}+\frac{2nCT^{2/3}(\ln T)^{1/3}}{\alpha}=O(nT^{2/3}(\log T)^{1/3})
\end{split}
\end{equation}
where the last inequality is due to $B\leq T$.

Similarly, with $\alpha>0$, $K=T^{2/3}(\ln T)^{-2/3}$, $\delta=cT^{-1/3}(\ln T)^{1/3}$, and $h=\alpha K$, we have
\begin{equation}
\label{thm_bsc-eq10-c}
\begin{split}
&nKG\sum_{m=1}^B\E\left[\frac{2\|\bar{\dd}(m)+\alpha K\bar{\x}(m+1)\|_2}{m\alpha K+2h}\right]\\
\leq&\frac{2nG\sqrt{C}}{\alpha}\left(\frac{T}{\ln T}\right)^{2/3}\sum_{m=1}^B\frac{1}{m+2}\leq\frac{2nG\sqrt{C}}{\alpha}\left(\frac{T}{\ln T}\right)^{2/3}\sum_{m=1}^B\frac{1}{m}\\
\leq&\frac{2nG\sqrt{C}}{\alpha}\left(\frac{T}{\ln T}\right)^{2/3}(1+\ln B)=O(nT^{2/3}(\log T)^{1/3}).
\end{split}
\end{equation}
Moreover, with $\alpha>0$, $K=L=T^{2/3}(\ln T)^{-2/3}$, $\delta=cT^{-1/3}(\ln T)^{1/3}$, and $h=\alpha K$, we have
\begin{align*}
u_m=&\frac{4R}{\sqrt{L+2}}+\sqrt{2K\left(\frac{dM}{\delta}\right)^2+2K^2G^2+2(\alpha KR)^2}\frac{\sqrt{n}}{m\alpha K(1-\sigma_2(P))}\\
\leq&\frac{4R(\ln T)^{1/3}}{T^{1/3}}+\frac{\sqrt{Cn}}{\sqrt{2}m\alpha(1-\sigma_2(P))}
\end{align*}
for any $B\geq m\geq2$.

Then, with $u_1=0$, $\alpha>0$, $K=T^{2/3}(\ln T)^{-2/3}$, $\delta=cT^{-1/3}(\ln T)^{1/3}$, and $h=\alpha K$, we have
\begin{equation}
\label{thm_bsc-eq11-c}
\begin{split}
3nKG\sum_{m=1}^B{u_m}\leq&\frac{12nK(B-1)GR(\ln T)^{1/3}}{T^{1/3}}+\frac{3n\sqrt{Cn}G}{\sqrt{2}\alpha(1-\sigma_2(P))}\left(\frac{T}{\ln T}\right)^{2/3}\sum_{m=1}^{B}\frac{1}{m}\\
\leq&12nGRT^{2/3}(\ln T)^{1/3}+\frac{3n\sqrt{Cn}G}{\sqrt{2}\alpha(1-\sigma_2(P))}\left(\frac{T}{\ln T}\right)^{2/3}(1+\ln T)\\
=&O\left(n^{3/2}(1-\sigma_2(P))^{-1}T^{2/3}(\log T)^{1/3}\right).
\end{split}
\end{equation}
Moreover, with $\alpha>0$, $K=T^{2/3}(\ln T)^{-2/3}$, $\delta=cT^{-1/3}(\ln T)^{1/3}$, and $h=\alpha K$, we have
\begin{equation}
\label{thm_bsc-eq12-c}
\begin{split}
&3\delta nGT+\frac{\delta nGRT}{r}+4nhR^2\\
=&\left(3cnG+\frac{cnGR}{r}\right)T^{2/3}(\ln T)^{1/3}+4n\alpha R^2T^{2/3}(\ln T)^{-2/3}=O(nT^{2/3}(\ln T)^{1/3}).
\end{split}
\end{equation}
Finally, by combining (\ref{thm_bsc-eq8}), (\ref{thm_bsc-eq9-c}), (\ref{thm_bsc-eq10-c}), (\ref{thm_bsc-eq11-c}), and (\ref{thm_bsc-eq12-c}), our Algorithm \ref{DBBCG-SC} with $\alpha>0$, $K=T^{2/3}(\ln T)^{-2/3}$, $\delta=cT^{-1/3}(\ln T)^{1/3}$, and $h=\alpha K$ ensures
\[\E\left[R_{T,i}\right]=O\left(n^{3/2}(1-\sigma_2(P))^{-1}T^{2/3}(\log T)^{1/3}\right)\]
for $\alpha$-strongly convex losses, which completes the proof of Theorem \ref{thm2-sc-c}.
\subsubsection{Proof of Lemma \ref{bound_d}}
We first notice that
\[\|\dd_i(m)\|_2^2=\|\widehat{\g}_i(m)-\alpha K\x_i(m)\|_2^2\leq2\|\widehat{\g}_i(m)\|_2^2+2\|\alpha K\x_i(m)\|_2^2\leq2\|\widehat{\g}_i(m)\|_2^2+2(\alpha KR)^2\]
where the last inequality is due to Assumption \ref{assum1}.

Moreover, it is easy to provide an upper bound of $\E[\|\widehat{\g}_i(m)\|_2^2]$ by following the proof of Lemma 5 in \citet{Garber19}. We include the detailed proof for completeness.

Let $t_j=(m-1)K+j$ for $j=1,\dots,K$. We have
\begin{equation}
\label{eq_EN}
\begin{split}
&\E\left[\|\widehat{\g}_i(m)\|_2^2|\x_i(m)\right]\\
=&\E\left.\left[\sum_{j=1}^K\g_i(t_j)^\top \g_i(t_j)\right|\x_i(m)\right]+\E\left.\left[\sum_{j=1}^K\sum_{k\in[K]\cap k\neq j}\g_i(t_j)^\top \g_i(t_k)\right|\x_i(m)\right]\\
=&\E\left.\left[\sum_{j=1}^K\|\g_i(t_j)\|_2^2\right|\x_i(m)\right]+\sum_{j=1}^K\sum_{k\in[K]\cap k\neq j}\E\left.\left[\g_i(t_j)\right|\x_i(m)\right]^\top \E\left.\left[\g_i(t_k)\right|\x_i(m)\right]\\
\leq&K\left(\frac{dM}{\delta}\right)^2+\sum_{j=1}^K\sum_{k\in[K]\cap k\neq j}\|\E\left.\left[\g_i(t_j)\right|\x_i(m)\right]\|_2\|\E\left.\left[\g_i(t_k)\right|\x_i(m)\right]\|_2\\
\leq& K\left(\frac{dM}{\delta}\right)^2+(K^2-K)G^2\\
\leq& K\left(\frac{dM}{\delta}\right)^2+K^2G^2
\end{split}
\end{equation}
where the second inequality is due to Lemmas \ref{smoothed_lem2} and \ref{smoothed_lem1}.

Therefore, we have
\begin{align*}\E[\|\dd_i(m)\|_2^2]\leq& 2\E[\|\widehat{\g}_i(m)\|_2^2]+2(\alpha KR)^2=2\E\left[\E\left[\|\widehat{\g}_i(m)\|_2^2|\x_i(m)\right]\right]+2(\alpha KR)^2\\
\leq&2K\left(\frac{dM}{\delta}\right)^2+2K^2G^2+2(\alpha KR)^2.
\end{align*}
Moreover, according to Jensen's inequality, we have
\[\E[\|\dd_i(m)\|_2]^2\leq\E[\|\dd_i(m)\|_2^2].\]
\section{Experiments}
In this section, we perform simulation experiments on the multiclass classification problem and the binary classification problem to verify the performance of our proposed algorithms.

\subsection{Datasets and Topologies of the Networks}
We conduct experiments on four publicly available datasets---aloi, news20, a9a, and ijcnn1 from the LIBSVM repository \citep{LIBSVM}, and the details of these datasets are summarized in Table \ref{datasets}. Specifically, aloi and news20 are used in the multiclass classification problem, and the other two datasets are used in the binary classification problem. For any dataset, let $T_e$ denote the number of examples. We first divide it into $n$ equally-sized parts where each part contains $\lfloor T_e/n\rfloor$ examples, and then distribute them onto the $n$ computing nodes in the network,\footnote{The remaining $T_e-n\lfloor T_e/n\rfloor$ examples are not used.} where $n=9$ for the multiclass classification problem and $n=100$ for the binary classification problem. Moreover, each part of the dataset will be reused $n$ times, which implies that the number of rounds $T$ is equal to $n\lfloor T_e/n\rfloor$.
\begin{table}[t]
\centering
\tabcaption{Summary of datasets}
\label{datasets}
\setlength{\tabcolsep}{3mm}
\begin{tabular}{|c|c|c|c|}
\hline
Dataset& $\#$~Features & $\#$~Classes & $\#$~Examples \\ \hline
a9a &123  & 2 &32561\\ \hline
ijcnn1 &22  & 2 &49990\\ \hline
aloi & 128 & 1000 &108000\\ \hline
news20 & 62061      & 20 &  15935   \\ \hline
\end{tabular}
\end{table}

To model the distributed network, we will use three types of graphs including a complete graph, a two-dimensional grid graph, and a cycle graph. The complete graph is a "well connected" network, where each node is connected to all other nodes. In contrast, the cycle graph is a "poorly connected" network, each node of which is only connected to two other nodes. Moreover, in the two-dimensional grid graph, each node not in the boundary is connected to its four nearest neighbors in axis-aligned directions. Its connectivity is between that of the complete graph and the cycle graph.

For the weight matrix $P$, we first compute $P_{ij}$ for $i\neq j$ as
\[
P_{ij}=\left\{
\begin{array}{rcl}
0,& &\text{if }j\notin N_i,\\
1/\max(|N_i|,|N_j|),& &\text{if }j\in N_i.\\
\end{array}
	\right.
\]
Then, we compute $P_{ij}=1-\sum_{q\in N_i,q\neq i}P_{iq}$ for $i=j$. In this way, we can ensure that $P$ satisfies Assumption \ref{assum5} for all three types of graphs.

\subsection{Multiclass Classification}
Following \citet{wenpeng17}, we first compare our D-BOCG against their D-OCG by conducting experiments on distributed online multiclass classification. Let $k$ be the number of features, and let $v$ be the number of classes. In the $t$-th round, after receiving a single example $\mathbf{e}_i(t)\in\R^k$, each local learner $i$ chooses a decision matrix $X_i(t)=[\x_1^{\top};\x_2^{\top};\dots;\x_v^{\top}]\in\R^{v\times k}$ from the convex set \[\K=\{X\in\R^{v\times k}|\|X\|_{\ast}\leq\tau\}\] where $\|X\|_{\ast}$ denotes the trace norm of $X$ and $\tau$ is set to be $50$. Note that $X_i(t)$ can be utilized to predict the class label of $\mathbf{e}_i(t)$ by computing $\argmax_{\ell\in[v]}\x_\ell^{\top}\mathbf{e}_i(t)$. Then, the true class label $y_{i}(t)\in\{1,\dots,v\}$ is revealed, which incurs the multivariate logistic loss
\[f_{t,i}(X_i(t))=\ln\left(1+\sum_{\ell\neq y_{i}(t)}e^{\x_\ell^{\top}\mathbf{e}_i(t)-\x_{y_{i}(t)}^{\top}\mathbf{e}_i(t)}\right).\]
The average loss of node $i$ at the $t$-th round is defined as
\begin{equation}
\label{eq_al1}
AL(t,i)=\frac{1}{tn}\sum_{q=1}^{t}\sum_{j=1}^nf_{q,j}(X_i(q)).
\end{equation}
For both methods, we simply initialize $X_i(1)=\mathbf{0}_{v\times k},\forall i\in[n]$. According to \citet{wenpeng17}, we set $s_t=1/\sqrt{t}$ and $\eta=cT^{-3/4}$ for D-OCG by tuning the constant $c$. Because the multivariate logistic loss is not strongly convex, the parameters of our D-BOCG are selected according to Corollary \ref{cor2}. Specifically, we set $\alpha=0$, $K=L=\lfloor\sqrt{T}\rfloor$, and $h=T^{3/4}/c$ by tuning the constant $c$. For both methods, the constant c is selected from $\{0.01,\dots,1e5\}$.
\begin{figure}[t]
\centering
\subfigure[aloi]{\includegraphics[width=0.46\textwidth]{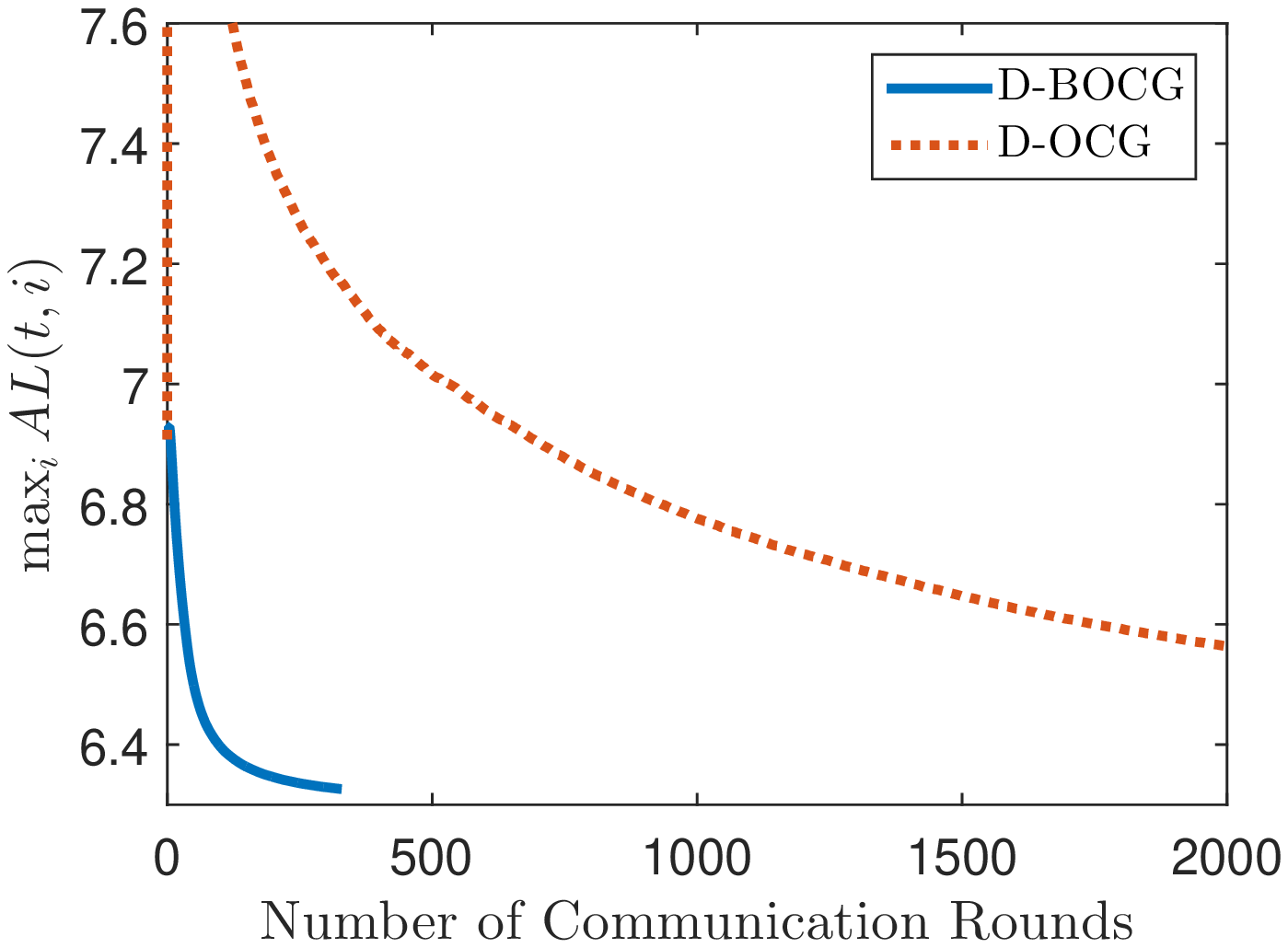}}
\centering
\subfigure[news20]{\includegraphics[width=0.46\textwidth]{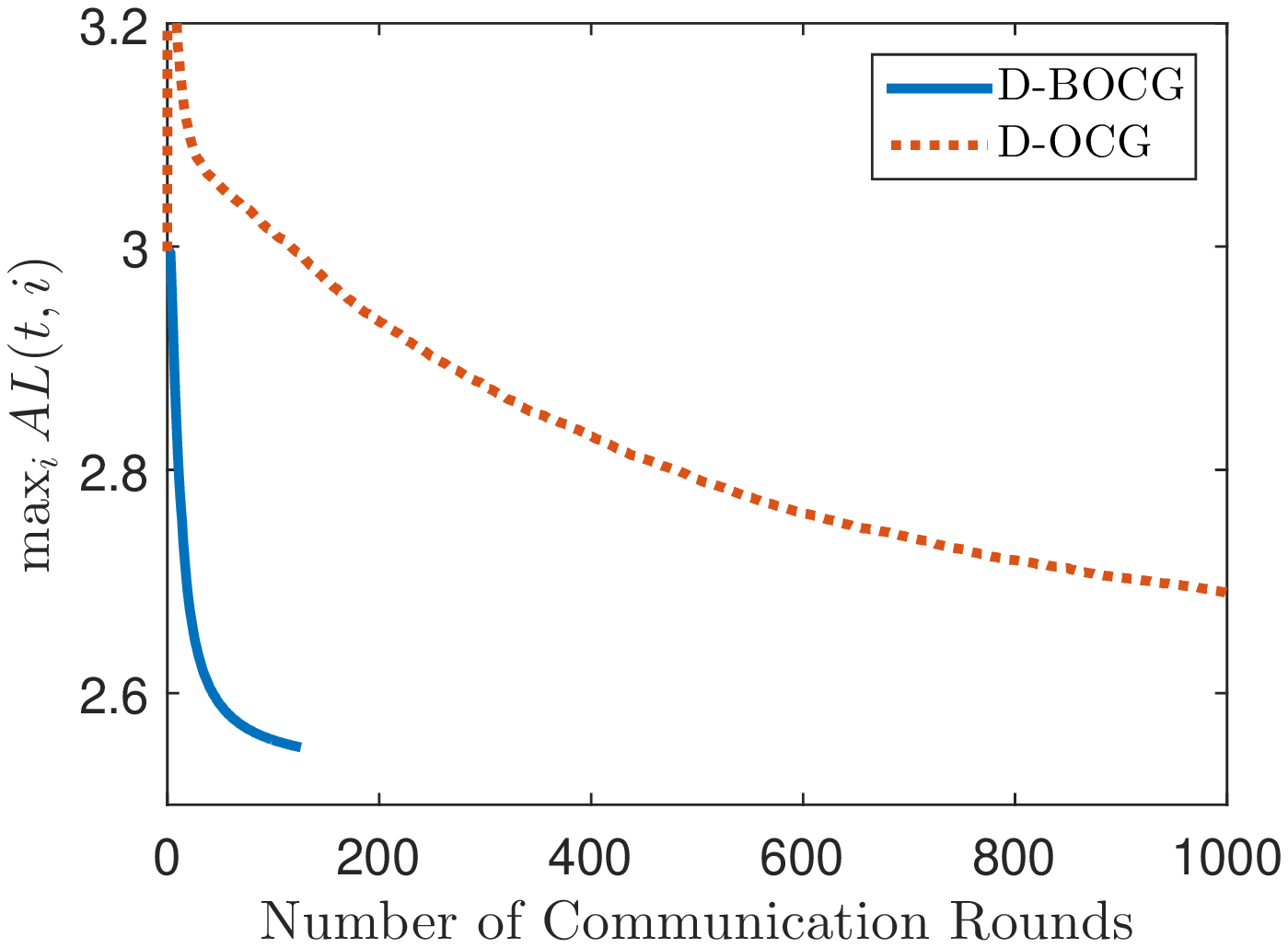}}
\caption{Comparisons of D-BOCG and D-OCG on distributed online multiclass classification over the complete graph.}
\label{OMC_fig1}
\centering
\subfigure[aloi]{\includegraphics[width=0.46\textwidth]{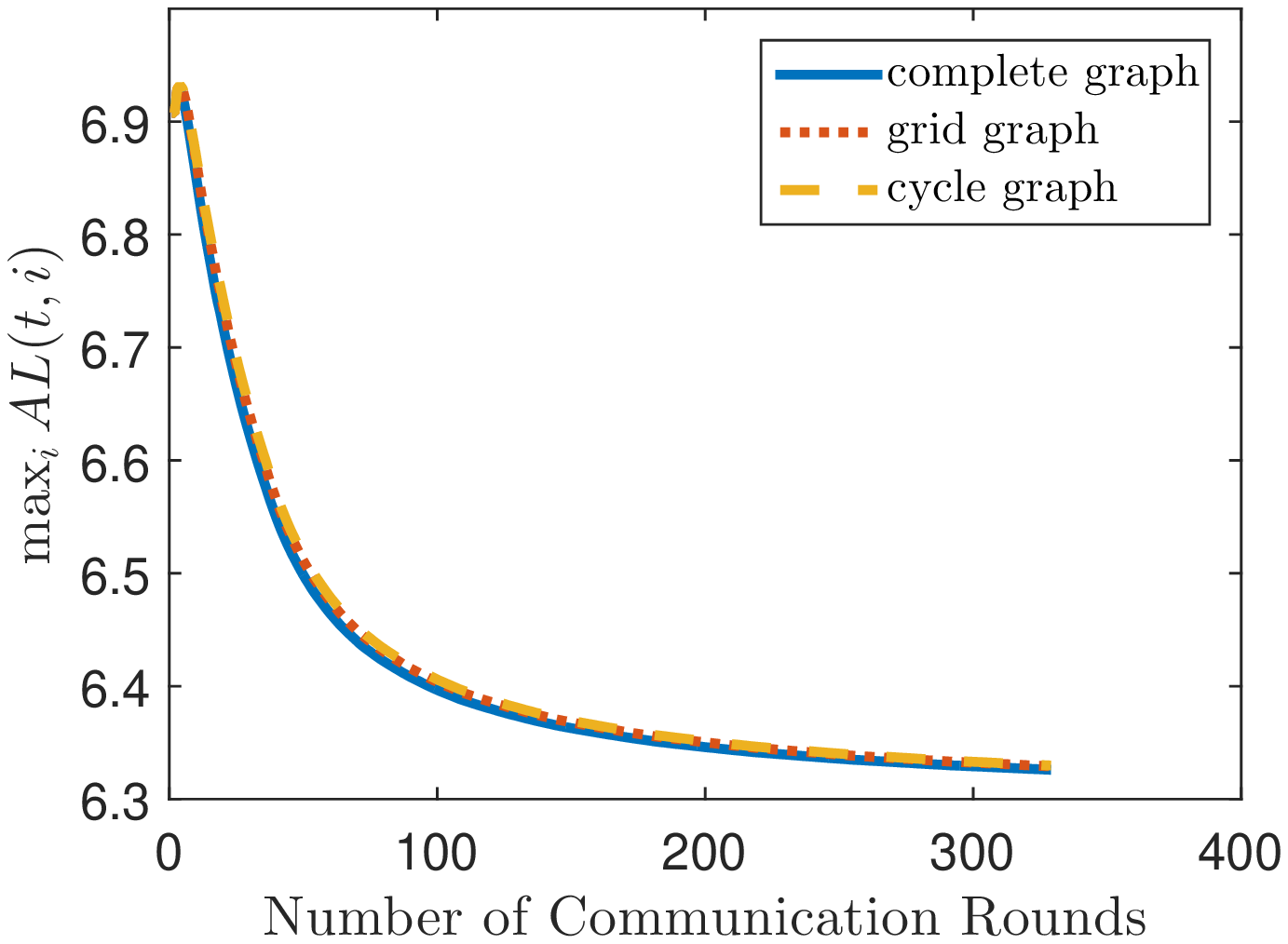}}
\centering
\subfigure[news20]{\includegraphics[width=0.46\textwidth]{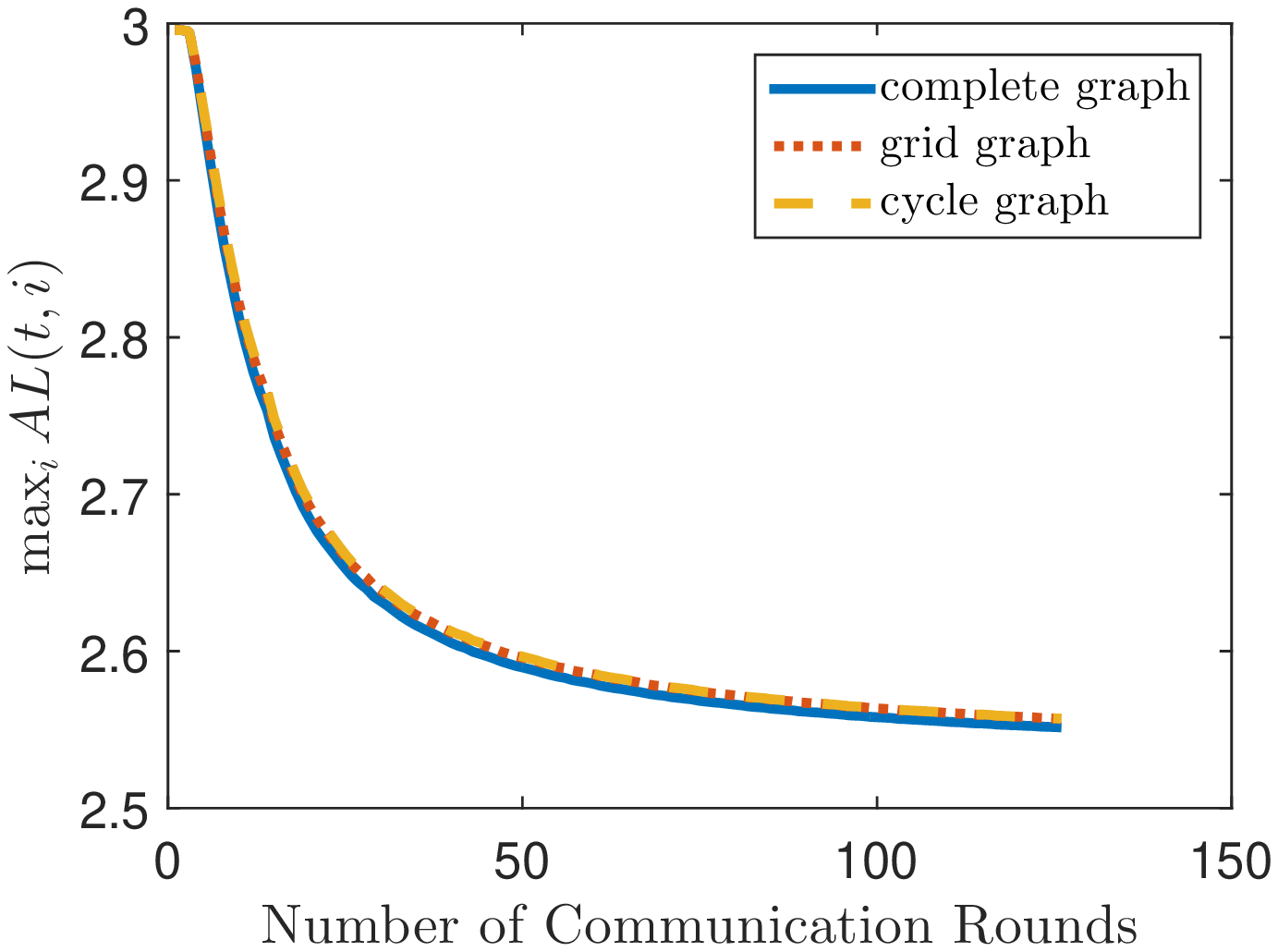}}
\caption{Comparisons of D-BOCG on distributed online multiclass classification over different graphs.}
\label{OMC_fig2}
\end{figure}

Fig.~\ref{OMC_fig1} shows the comparisons of our D-BOCG and D-OCG on distributed online multiclass classification over the complete graph. We find that the average loss of the worst local node in D-BOCG decreases faster than that of the worst local node in D-OCG with the increasing of communication rounds, which verifies our theoretical results about the regret bound and communication complexity of D-BOCG. Furthermore, Fig.~\ref{OMC_fig2} shows comparisons of D-BOCG on distributed online multiclass classification over different graphs. We find that with the improvement of the graph connectivity, the convergence of our D-BOCG is slightly improved, which is also consistent with our theoretical results about the regret bound of D-BOCG.

\subsection{Binary Classification}
We also consider the problem of binary classification in the distributed online learning setting. In the $t$-th round, each local learner $i$ receives a single example $\mathbf{e}_i(t)\in\R^d$ and chooses a decision $\x_i(t)\in\R^{d}$ from the convex set \[\K=\{\x\in\R^{d}|\|\x\|_{1}\leq\tau\}\] where $\tau$ is set to be $10$. Then, the true class label $y_{i}(t)\in\{-1,1\}$ is revealed, and it suffers the regularized hinge loss
\[f_{t,i}(\x_i(t))=\max\left\{1-y_{i}(t)\mathbf{e}_i(t)^\top \x_i(t),0\right\}+\lambda\|\x_i(t)\|_2^2\]
where $\lambda$ is set to be $0.1$. Similar to (\ref{eq_al1}), the average loss of node $i$ at the $t$-th round is defined as
\[
AL(t,i)=\frac{1}{tn}\sum_{q=1}^{t}\sum_{j=1}^nf_{q,j}(\x_i(q)).
\]
Note that the regularized hinge loss is $2\lambda$-strongly convex. To utilize the strong convexity, we can set parameters of D-BOCG according to Corollary \ref{cor-sc}. Moreover, to show the advantage of utilizing the strong convexity, we also run D-BOCG with parameters in Corollary \ref{cor2}, which only utilize the convexity condition. To distinguish between these two different instances of D-BOCG, we denote D-BOCG with parameters in Corollary \ref{cor-sc} as D-BOCG$_{\rm sc}$, and D-BOCG with parameters in Corollary \ref{cor2} as D-BOCG$_{\rm c}$.

For D-OCG, D-BOCG$_{\rm c}$, and D-BOCG$_{\rm sc}$, we simply initialize $\x_i(1)=\tau\mathbf{1}/d,\forall i\in[n]$, where $\mathbf{1}$ denotes the vector with each entry equal 1. The parameters of D-OCG are set in the same way as D-OCG in previous experiments, and the parameters of D-BOCG$_{\rm c}$ are set in the same way as D-BOCG in previous experiments. For D-BOCG$_{\rm sc}$, according to Corollary \ref{cor-sc}, we set $\alpha=2\lambda$ and $K=L=\lfloor T^{2/3}(\ln T)^{-2/3}\rfloor$. Although we use $h=\alpha K$ in Corollary \ref{cor-sc}, in the experiments, we set $h=c^\prime\alpha K$ by tuning the constant $c^\prime$ from $\{1,2,3,4,5\}$. It is easy to verify that the modified $h$ only affect the constant factor of the original regret bound in Corollary \ref{cor-sc}.

Fig.~\ref{fig_b1} shows comparisons of D-OCG, D-BOCG$_{\rm c}$, and D-BOCG$_{\rm sc}$ on distributed online binary classification over the complete graph. First, the average loss of the worst local node in D-BOCG$_{\rm c}$ and D-BOCG$_{\rm sc}$ decreases faster than that of the worst local node in D-OCG with the increasing of communication rounds, which validates our advantage in the communication complexity again. Moreover, our D-BOCG$_{\rm sc}$ outperforms D-BOCG$_{\rm c}$, which further validates the advantage of utilizing the strong convexity. Fig.~\ref{fig_b2-1} and \ref{fig_b2-2} show comparisons of D-BOCG$_{\rm c}$ and D-BOCG$_{\rm sc}$ on distributed online binary classification over different graphs. We find that the effect of the graph connectivity is similar to that presented in Fig.~\ref{OMC_fig2}, though the number of nodes increases from $9$ to $100$.
\begin{figure}[t]
\centering
\subfigure[a9a]{\includegraphics[width=0.46\textwidth]{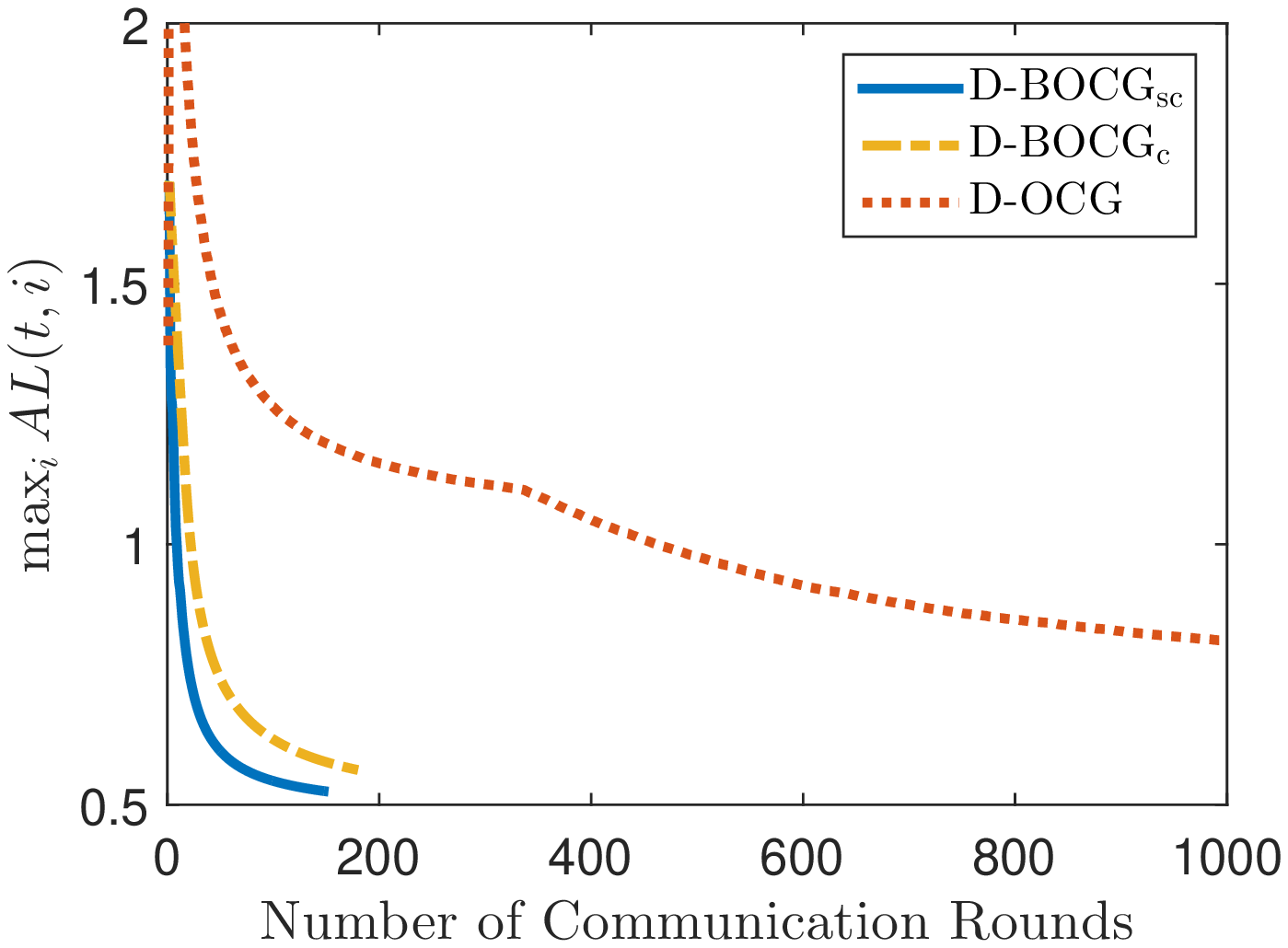}}
\centering
\subfigure[ijcnn1]{\includegraphics[width=0.46\textwidth]{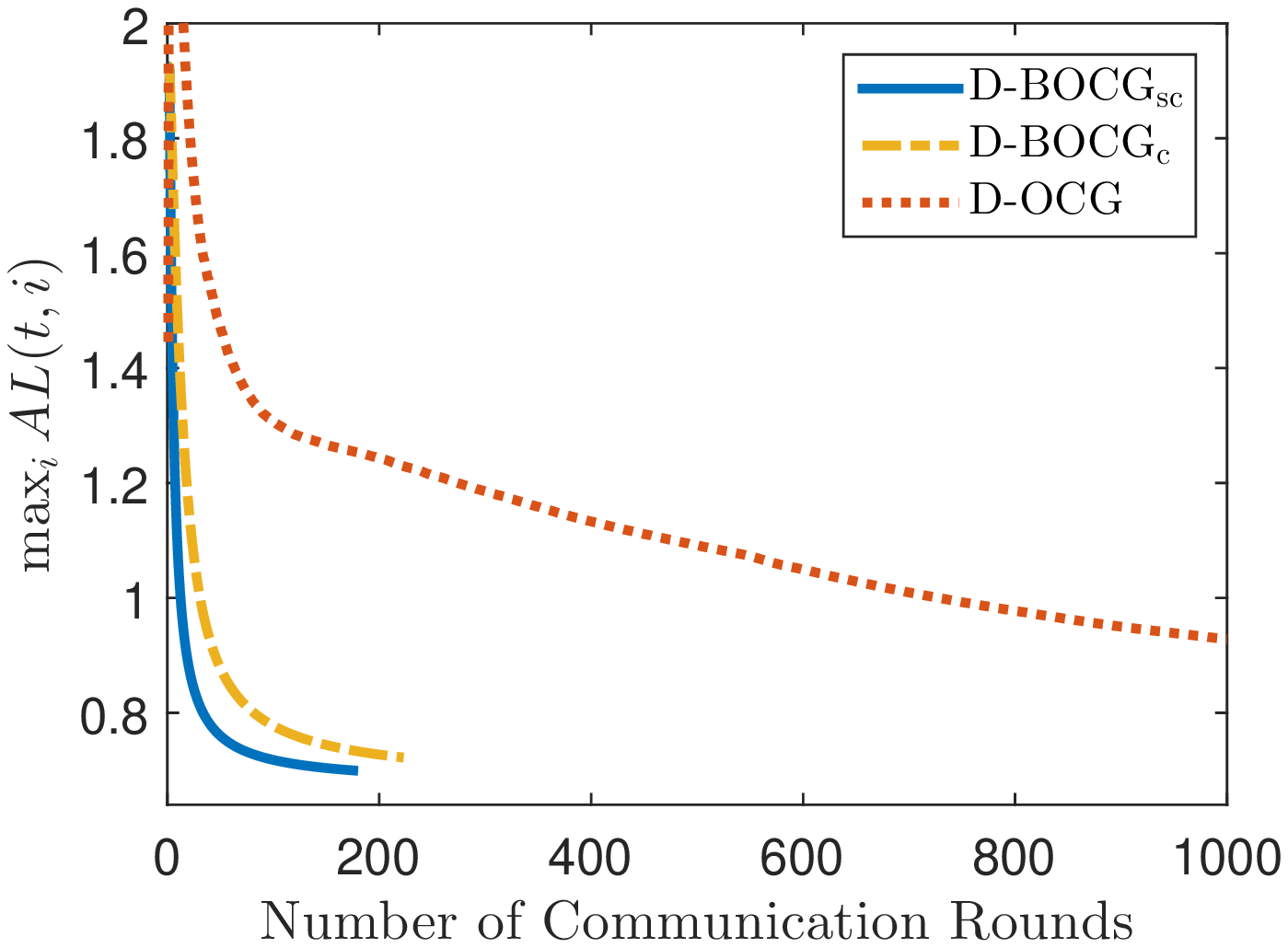}}
\caption{Comparisons of D-OCG, D-BOCG$_{\rm c}$, and D-BOCG$_{\rm sc}$ on distributed online binary classification over the complete graph.}
\label{fig_b1}
\end{figure}
\begin{figure}[t]
\centering
\subfigure[a9a]{\includegraphics[width=0.46\textwidth]{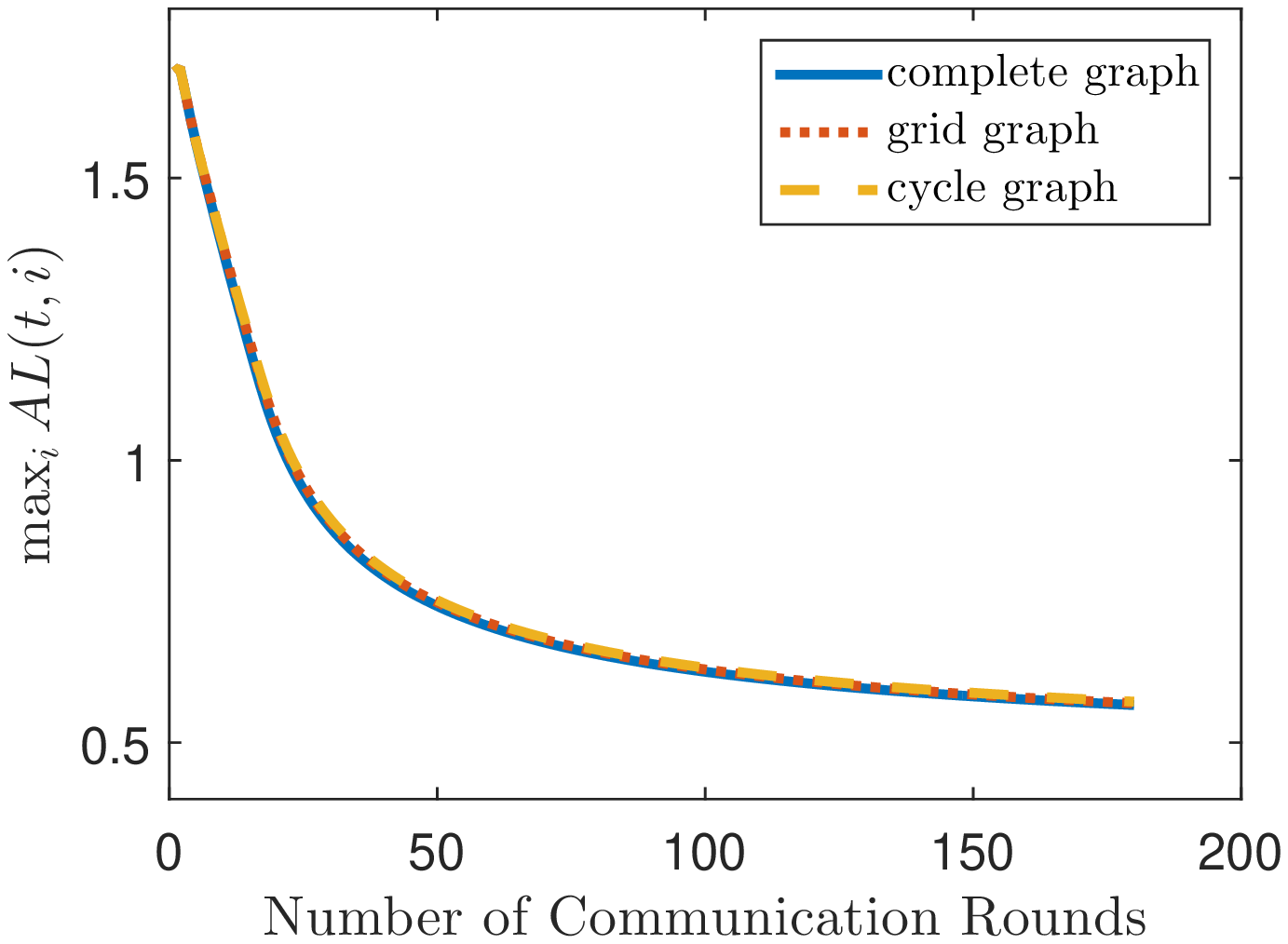}}
\centering
\subfigure[ijcnn1]{\includegraphics[width=0.46\textwidth]{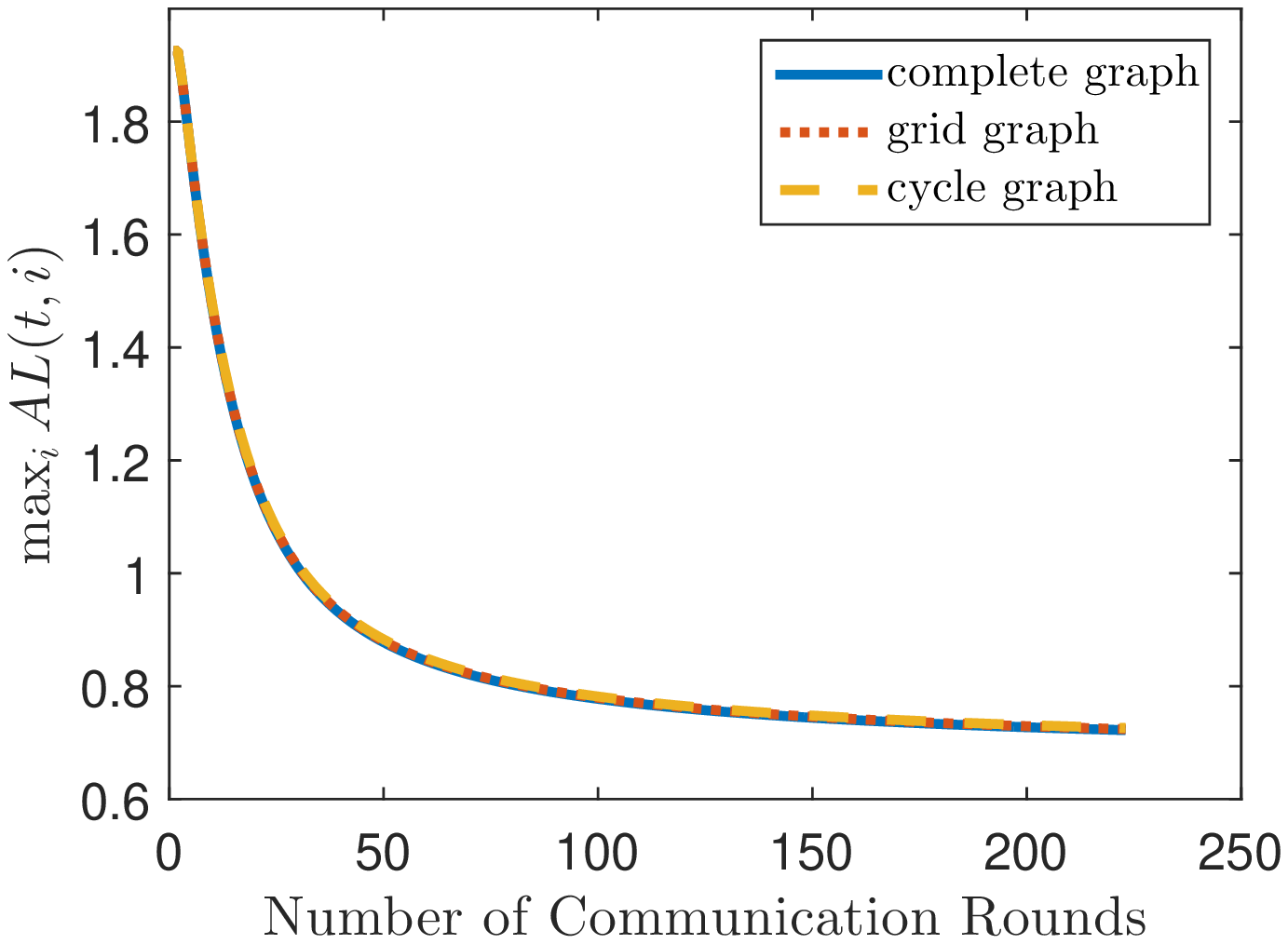}}
\caption{Comparisons of D-BOCG$_{\rm c}$ on distributed online binary classification over different graphs.}
\label{fig_b2-1}
\centering
\subfigure[a9a]{\includegraphics[width=0.46\textwidth]{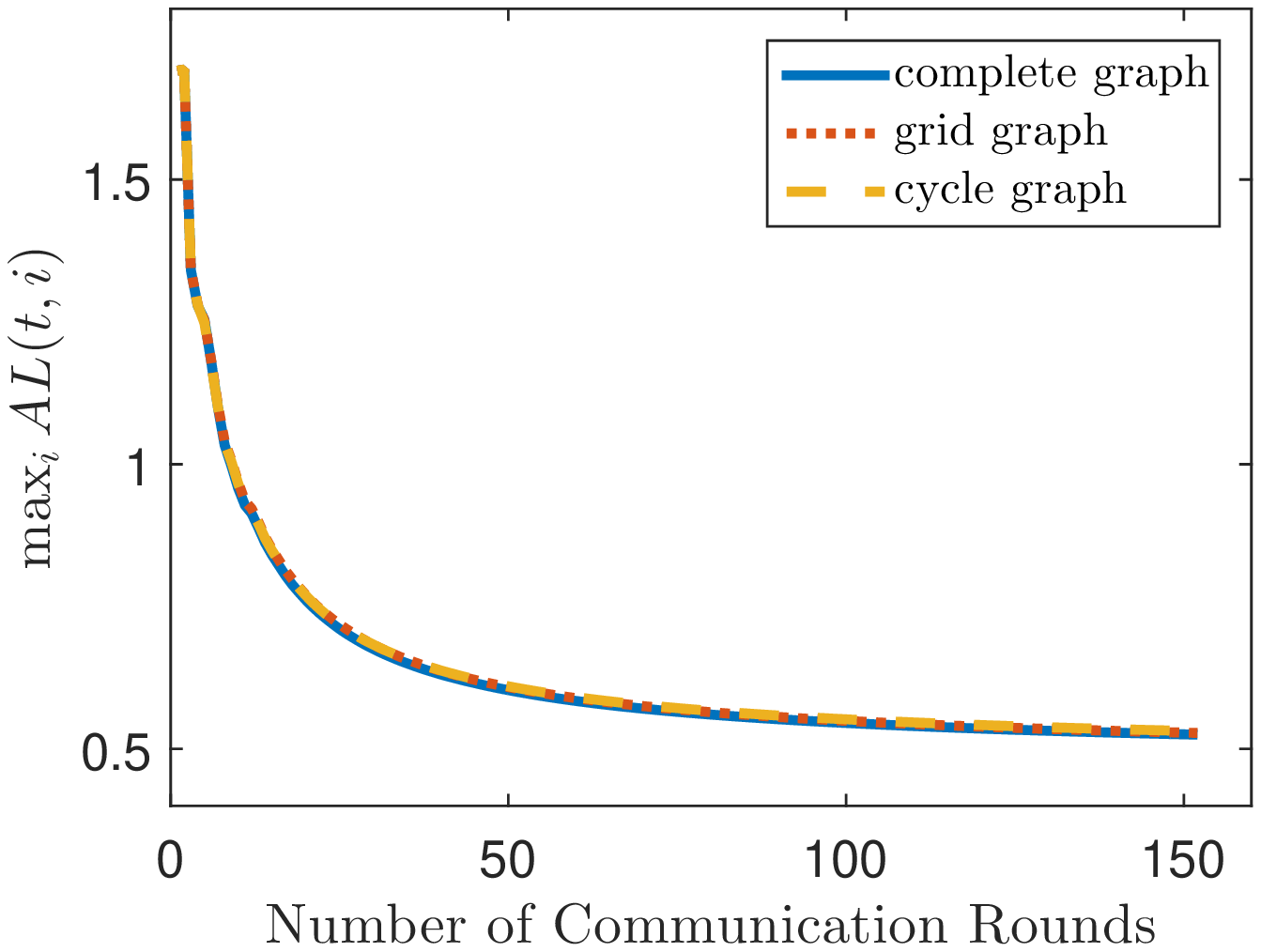}}
\centering
\subfigure[ijcnn1]{\includegraphics[width=0.46\textwidth]{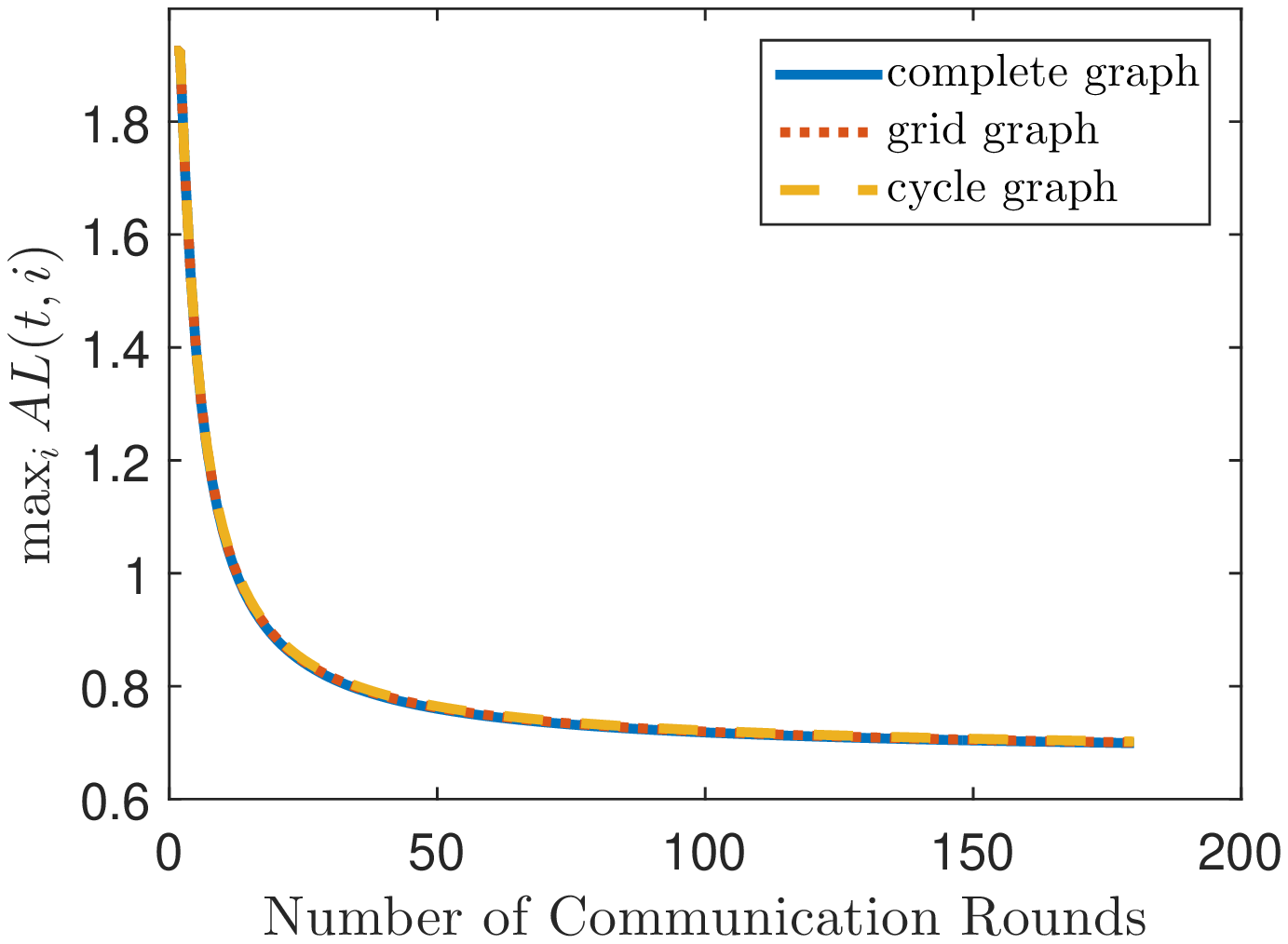}}
\caption{Comparisons of D-BOCG$_{\rm sc}$ on distributed online binary classification over different graphs.}
\label{fig_b2-2}
\end{figure}

Then, to verify the performance of our D-BBCG, we compare it with our D-BOCG. Note that D-BBCG only uses approximate gradients generated by the one-point gradient estimator, the performance of which is highly affected by the dimensionality. Therefore, to make a fair comparison, we only use ijcnn1, the dimensionality of which is relatively small. Specifically, we denote D-BBCG with parameters in Theorem \ref{thm2-sc} as D-BBCG$_{\rm c}$, and D-BBCG with parameters in Theorem \ref{thm2-sc-c} as D-BBCG$_{\rm sc}$. According to Theorems \ref{thm2-sc} and \ref{thm2-sc-c}, we set $\alpha=0$, $K=L=\lfloor\sqrt{T}\rfloor$, $\delta=10T^{-1/4}$, and $h=T^{3/4}/c$ for D-BBCG$_{\rm c}$ where the constant $c$ is tuned from $\{0.01,\dots,1e5\}$, and set $\alpha=2\lambda$, $K=L=\lfloor T^{2/3}(\ln T)^{-2/3}\rfloor$, $\delta=10T^{-1/3}(\ln T)^{1/3}$, and $h=c^\prime \alpha K$ where the constant $c^\prime$ is tuned from $\{1,2,3,4,5\}$. Moreover, we initialize $\x_i(1)=(1-\delta\sqrt{d}/\tau)\mathbf{1}/d,\forall i\in[n]$ for both D-BBCG$_{\rm c}$ and D-BBCG$_{\rm sc}$. Since D-BBCG$_{\rm c}$ and D-BBCG$_{\rm sc}$ are  randomized algorithms, we repeat them 10 times and report the average results.
\begin{figure}[t]
\centering
\subfigure[complete graph]{\includegraphics[width=0.32\textwidth]{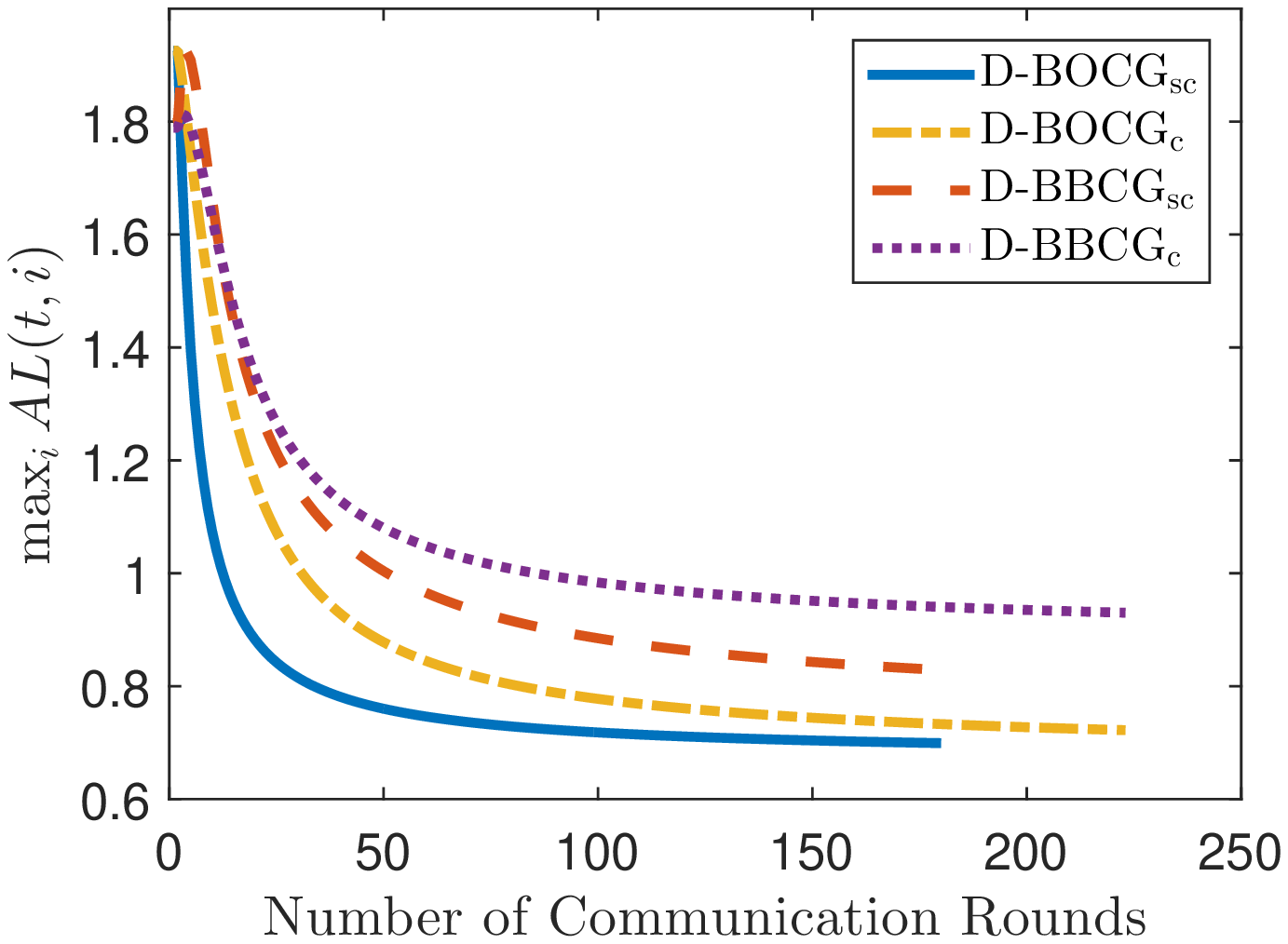}}
\centering
\subfigure[grid graph]{\includegraphics[width=0.32\textwidth]{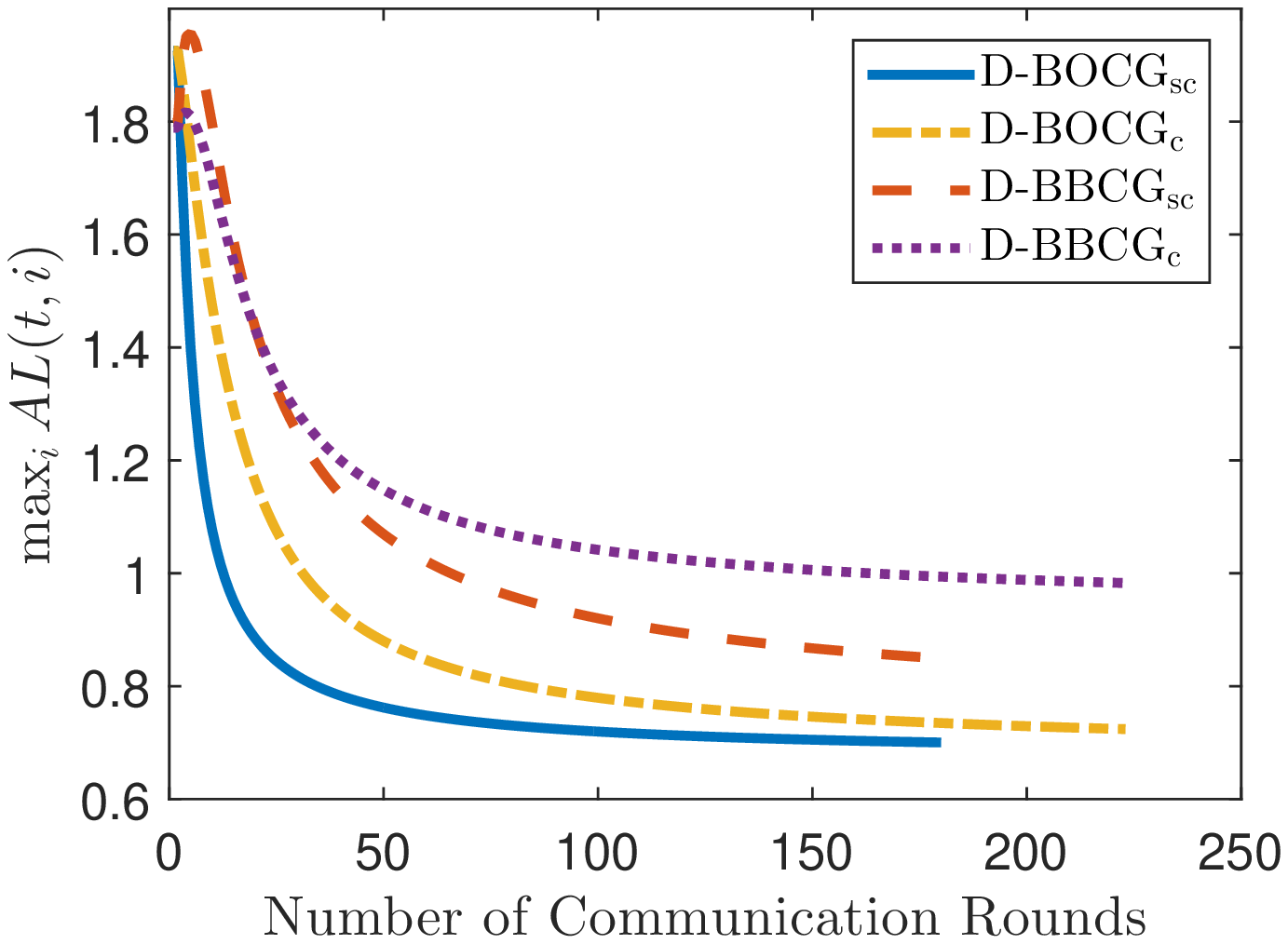}}
\centering
\subfigure[cycle graph]{\includegraphics[width=0.32\textwidth]{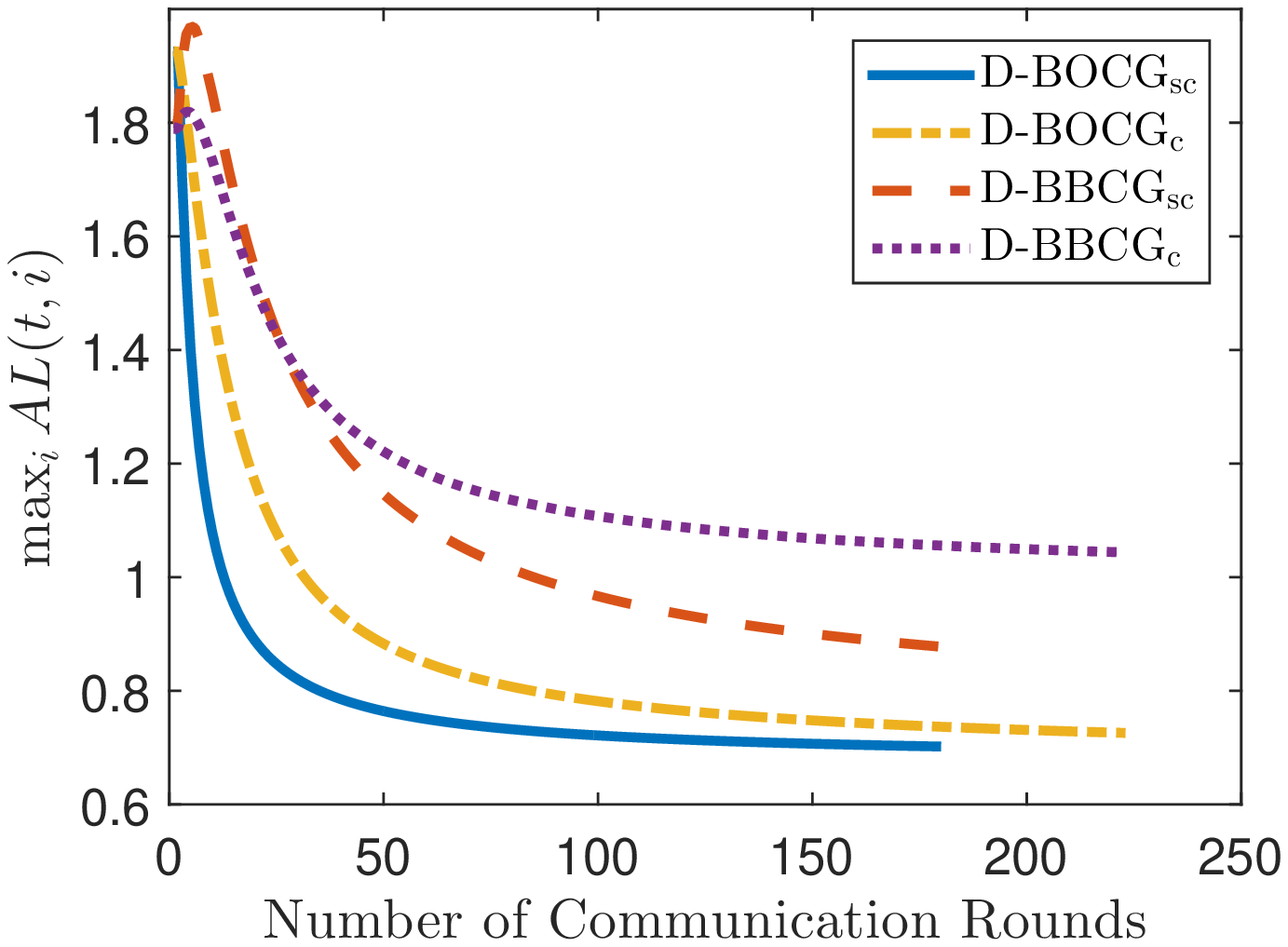}}
\caption{Comparisons of D-BOCG$_{\rm c}$, D-BOCG$_{\rm sc}$, D-BBCG$_{\rm c}$, and D-BBCG$_{\rm sc}$ on distributed online binary classification for ijcnn1.}
\label{fig_b3}
\end{figure}

Fig.~\ref{fig_b3} shows comparisons of D-BOCG$_{\rm c}$, D-BOCG$_{\rm sc}$, D-BBCG$_{\rm c}$, and D-BBCG$_{\rm sc}$ on distributed online binary classification for ijcnn1. For all three types of graphs, we find that D-BBCG$_{\rm c}$ is worse than D-BOCG$_{\rm c}$ and D-BBCG$_{\rm sc}$ is worse than D-BOCG$_{\rm cs}$, which is reasonable because D-BBCG$_{\rm c}$ and D-BBCG$_{\rm sc}$ are working with the more challenging bandit setting. Moreover, D-BBCG$_{\rm sc}$ is better than D-BBCG$_{\rm c}$, which validates the advantage of utilizing the strong convexity in the bandit setting.

\section{Conclusion and Future Work}
In this paper, we first propose a projection-free algorithm called D-BOCG for distributed online convex optimization. Our analysis shows that D-BOCG enjoys an $O(T^{3/4})$ regret bound with $O(\sqrt{T})$ communication rounds for convex losses, and a better regret bound of $O(T^{2/3}(\log T)^{1/3})$ with fewer $O(T^{1/3}(\log T)^{2/3})$ communication rounds for strongly convex losses.
In the case with convex losses, the $O(T^{3/4})$ regret bound of D-BOCG matches the best result established by the existing projection-free algorithm with $T$ communication rounds, and the $O(\sqrt{T})$ communication rounds required by D-BOCG match (in terms of $T$) the lower bound for any distributed online algorithm attaining the $O(T^{3/4})$ regret. In the case with strongly convex losses, we also provide a lower bound to show that the $O(T^{1/3}(\log T)^{2/3})$ communication rounds required by D-BOCG are nearly optimal (in terms of $T$) for obtaining the $O(T^{2/3}(\log T)^{1/3})$ regret bound up to polylogarithmic factors. Furthermore, to handle the bandit setting, we propose a bandit variant of D-BOCG, namely D-BBCG, and obtain similar theoretical guarantees.

Besides the future work discussed in Section \ref{sec5.3} and Remark \ref{rem8}, there are still several open problems to be investigated. First, in the standard OCO, \citet{Hazan20} have proposed a projection-free algorithm that obtains an expected regret bound of $O(T^{2/3})$ for convex and smooth losses. It is interesting to extend their algorithm to the distributed setting studied in this paper. However, their algorithm is not based on conditional gradient, which makes the extension non-trivial. 
Second, in this paper, the weight matrix $P$ is assumed to be symmetric and doubly stochastic. It is appealing to consider a more practical scenario, in which $P$ could be asymmetric or only column (or row) stochastic \citep{Yang19survey,Yi2020}. Finally, we will investigate whether the regret bound for the full information setting can be improved if a few projections are allowed. We note that $O(\log T)$ projections are sufficient to achieve the optimal convergence rate for stochastic optimization of smooth and strongly convex functions \citep{zhangICML13}.




\appendix

\section{Proof of Corollaries \ref{cor2} and \ref{cor-sc}}
\label{A1.3}
Corollary \ref{cor2} can be proved by substituting $\alpha=0$, $K=L=\sqrt{T}$, and $h=\frac{n^{1/4}T^{3/4}G}{\sqrt{1-\sigma_2(P)}R}$ into Theorem \ref{thm1-sc-2}, as follows
\begin{equation*}
\begin{split}
R_{T,i}
\leq&\frac{12nGRT}{\sqrt{\sqrt{T}+2}}+\sum_{m=2}^B\frac{3n^{5/4}T^{1/4}GR}{2\sqrt{1-\sigma_2(P)}}+\sum_{m=1}^B2\sqrt{1-\sigma_2(P)}n^{3/4}T^{1/4}GR+\frac{4n^{5/4}T^{3/4}GR}{\sqrt{1-\sigma_2(P)}}\\
\leq&12nGRT^{3/4}+\frac{11n^{5/4}T^{3/4}GR}{2\sqrt{1-\sigma_2(P)}}+2\sqrt{1-\sigma_2(P)}n^{3/4}T^{3/4}GR
\end{split}
\end{equation*}
where the last inequality is due to $B-1<B=T/K=\sqrt{T}$.

Corollary \ref{cor-sc} can be proved by substituting
$\alpha>0$, $K=L=T^{2/3}(\ln T)^{-2/3}$, and $h=\alpha K$ into Theorem \ref{thm1-sc-2}, as follows
\begin{equation*}
\begin{split}
R_{T,i}
\leq&\frac{12nGRT}{\sqrt{L}}+\sum_{m=2}^B\frac{3nGK(G+\alpha R)\sqrt{n}}{m\alpha (1-\sigma_2(P))}+\sum_{m=1}^B\frac{4nK(G+2\alpha R)^2}{(m+2)\alpha}+4n\alpha KR^2\\
\leq&\frac{12nGRT}{\sqrt{L}}+\left(\frac{3nG(G+\alpha R)\sqrt{n}}{\alpha (1-\sigma_2(P))}+\frac{4n(G+2\alpha R)^2}{\alpha}\right)\sum_{m=1}^{B}\frac{K}{m}+4n\alpha KR^2\\
\leq&12nGRT^{2/3}(\ln T)^{1/3}+\left(\frac{3nG(G+\alpha R)\sqrt{n}}{\alpha (1-\sigma_2(P))}+\frac{4n(G+2\alpha R)^2}{\alpha}\right)K(1+\ln B)+4n\alpha KR^2\\
\leq&\left(\frac{3n^{3/2}G(G+\alpha R)}{\alpha (1-\sigma_2(P))}+\frac{4n(G+2\alpha R)^2}{\alpha}\right)T^{2/3}((\ln T)^{-2/3}+(\ln T)^{1/3})\\
&+12nGRT^{2/3}(\ln T)^{1/3}+4nR^2\alpha T^{2/3}(\ln T)^{-2/3}
\end{split}
\end{equation*}
where the last inequality is due to $K=T^{2/3}(\ln T)^{-2/3}$ and $\ln B\leq \ln T$.

\section{Proof of Theorem \ref{pro_thm2}}
In the beginning, we define several auxiliary variables. Let $\bar{\z}(m)=\frac{1}{n}\sum_{i=1}^n\z_i(m)$ for any $m\in[B+1]$ and $\bar{\g}(m)=\frac{1}{n}\sum_{i=1}^n\widehat{\g}_i(m)$ for any $m\in[B]$. Then, we define $\bar{\x}(1)=\x_{\ii}$ and $\bar{\x}(m+1)=\argmin_{\x\in\K_\delta}
\bar{F}_{m}(\x)$ for any $m\in[B+1]$, where
\[\bar{F}_{m}(\x)=\bar{\z}(m)^{\top}\mathbf{x}+h\|\mathbf{x}-\x_{\ii}\|_2^2.\] Similarly, we define $\widehat{\x}_i(m+1)=\argmin_{\x\in\K_\delta}F_{m,i}(\x)$ for any $m\in[B+1]$, where \[F_{m,i}(\x)=\z_{i}(m)^{\top}\mathbf{x}+h\|\mathbf{x}-\x_{\ii}\|_2^2\] is defined in Algorithm \ref{DBBCG-SC} when $\alpha=0$.

Moreover, let $\x^\ast\in\argmin_{\x\in\K}\sum_{t=1}^Tf_{t}(\x)$ and $\widetilde{\x}^\ast=(1-\delta/r)\x^\ast$. For any $j\in V$ and $t\in[T]$, we denote the $\delta$-smoothed version of $f_{t,j}(\x)$ by $\widehat{f}_{t,j,\delta}(\x)$.
Note that as in (\ref{thm_bsc_new_eq1}), we have proved that Algorithm \ref{DBBCG-SC} ensures
\begin{equation}
\label{thm_bc-eq1}
\begin{split}
R_{T,i}
\leq&\sum_{m=1}^B\sum_{t\in\mathcal{T}_m}\sum_{j=1}^n(\widehat{f}_{t,j,\delta}(\x_{i}(m))-\widehat{f}_{t,j,\delta}(\widetilde{\x}^\ast))+3\delta nGT+\frac{\delta nGRT}{r}.
\end{split}
\end{equation}
To bound the term $\sum_{m=1}^B\sum_{t\in\mathcal{T}_m}\sum_{j=1}^n(\widehat{f}_{t,j,\delta}(\x_{i}(m))-\widehat{f}_{t,j,\delta}(\widetilde{\x}^\ast))$ in (\ref{thm_bc-eq1}), we assume that for all $i\in V$ and $m=1,\dots,B$, Algorithm \ref{DBBCG-SC} ensures that
\begin{equation*}
\|\widehat{\g}_i(m)\|_2\leq \widehat{G}=\xi_T\frac{dM\sqrt{K}}{\delta}+KG.
\end{equation*}
Then, we can derive an upper bound of $\|\x_i(m)-\bar{\x}(m)\|_2$. For any $B\geq m\geq2$, we note that $F_{m-1,i}(\x)$ is $2h$-smooth, and Algorithm \ref{DBBCG-SC} ensures
 \[\mathbf{x}_{i}(m)=\text{CG}(\mathcal{K}_\delta, L, F_{m-1,i}(\x), \x_i(m-1)).\]
According to Lemma \ref{lem_ILO}, Assumption \ref{assum1}, and $\mathcal{K}_\delta\subseteq\mathcal{K}$, for $B\geq m\geq2$, it is easy to verify that
\[F_{m-1,i}(\x_i(m))-F_{m-1,i}(\widehat{\x}_i(m))\leq\frac{16hR^2}{L+2}.\]
Then, for any $B\geq m\geq2$, we have
\begin{equation}
\label{thm_bc-eq2}
\begin{split}
\|\x_i(m)-\bar{\x}(m)\|_2\leq&\|\x_i(m)-\widehat{\x}_i(m)\|_2+\|\widehat{\x}_i(m)-\bar{\x}(m)\|_2\\
\leq&\sqrt{\frac{F_{m-1,i}(\x_i(m))-F_{m-1,i}(\widehat{\x}_i(m))}{h}}+\|\widehat{\x}_i(m)-\bar{\x}(m)\|_2\\
\leq&\frac{4R}{\sqrt{L+2}}+\|\widehat{\x}_i(m)-\bar{\x}(m)\|_2\\
\leq&\frac{4R}{\sqrt{L+2}}+\frac{1}{2h}\|\z_i(m)-2h\x_{\ii}-\bar{\z}(m)+2h\x_{\ii}\|_2\\
\leq&\frac{4R}{\sqrt{L+2}}+\frac{\widehat{G}\sqrt{n}}{2h(1-\sigma_2(P))}
\end{split}
\end{equation}
where the second inequality is due to the fact that $F_{m-1,i}(\x)$ is $2h$-strongly convex and (\ref{cor_scvx}), the fourth inequality is due to Lemma \ref{dual_lem1}, and the last inequality is due to Lemma \ref{graph_lem_zhang}.

For brevity, let $\epsilon=\frac{4R}{\sqrt{L+2}}+\frac{\widehat{G}\sqrt{n}}{2h(1-\sigma_2(P))}$. By combining (\ref{thm_bc-eq2}) with $\x_i(1)=\bar{\x}(m)=\x_{\ii}$, for any $m\in[B]$, we have
\begin{equation}
\label{thm_bc-eq3}
\begin{split}
\|\x_i(m)-\bar{\x}(m)\|_2\leq\epsilon.
\end{split}
\end{equation}
For any $i,j\in V$, $m\in[B]$, and $t\in\mathcal{T}_m$, according to Lemma \ref{smoothed_lem1} and Assumption \ref{assum4}, $\widehat{f}_{t,j,\delta}(\x)$ is also convex and $G$-Lipschitz. Then, by combining with (\ref{thm_bc-eq3}), we have
\begin{equation}
\label{thm_bc-eq4}
\begin{split}
&\widehat{f}_{t,j,\delta}(\x_{i}(m))-\widehat{f}_{t,j,\delta}(\widetilde{\x}^\ast)\\
\leq&\widehat{f}_{t,j,\delta}(\bar{\x}(m))-\widehat{f}_{t,j,\delta}(\widetilde{\x}^\ast)+G\|\bar{\x}(m)-\x_i(m)\|_2\\
\leq&\widehat{f}_{t,j,\delta}(\x_{j}(m))-\widehat{f}_{t,j,\delta}(\widetilde{\x}^\ast)+G\|\bar{\x}(m)-\x_j(m)\|_2+G\epsilon\\
\leq&\nabla\widehat{f}_{t,j,\delta}(\x_{j}(m))^\top(\x_{j}(m)-\widetilde{\x}^\ast)+2G\epsilon\\
\leq&\nabla\widehat{f}_{t,j,\delta}(\x_{j}(m))^\top(\x_{j}(m)-\bar{\x}(m))+\nabla\widehat{f}_{t,j,\delta}(\x_{j}(m))^\top(\bar{\x}(m)-\widetilde{\x}^\ast)+2G\epsilon\\
\leq&\|\nabla\widehat{f}_{t,j,\delta}(\x_{j}(m))\|_2\|\x_{j}(m)-\bar{\x}(m)\|_2+\nabla\widehat{f}_{t,j,\delta}(\x_{j}(m))^\top(\bar{\x}(m)-\widetilde{\x}^\ast)+2G\epsilon\\
\leq&\nabla\widehat{f}_{t,j,\delta}(\x_{j}(m))^\top(\bar{\x}(m)-\widetilde{\x}^\ast)+3G\epsilon.
\end{split}
\end{equation}
By combining (\ref{thm_bc-eq1}) with (\ref{thm_bc-eq4}), for any $i\in V$, we have
\begin{equation*}
\begin{split}
R_{T,i}\leq&\sum_{m=1}^B\sum_{t\in\mathcal{T}_m}\sum_{j=1}^n\nabla\widehat{f}_{t,j,\delta}(\x_{j}(m))^\top(\bar{\x}(m)-\widetilde{\x}^\ast)+3nGT\epsilon+3\delta nGT+\frac{\delta nGRT}{r}.
\end{split}
\end{equation*}
Then, to bound $\sum_{m=1}^B\sum_{t\in\mathcal{T}_m}\sum_{j=1}^n\nabla\widehat{f}_{t,j,\delta}(\x_{j}(m))^\top(\bar{\x}(m)-\widetilde{\x}^\ast)$, we introduce the following lemma.
\begin{lem}
\label{lem2_azuma}
Let $\bar{\z}(m)=\frac{1}{n}\sum_{i=1}^n\z_i(m)$ for any $m\in[B+1]$ and $\bar{\g}(m)=\frac{1}{n}\sum_{i=1}^n\widehat{\g}_i(m)$ for any $m\in[B]$. Define $\bar{\x}(1)=\x_{\ii}$, where $\x_{\ii}$ is an input of Algorithm \ref{DBBCG-SC}. Moreover, define $\bar{F}_{m}(\x)=\bar{\z}(m)^{\top}\mathbf{x}+h\|\mathbf{x}-\x_{\ii}\|_2^2$ and $\bar{\x}(m+1)=\argmin_{\x\in\K_\delta}
\bar{F}_{m}(\x)$ for any $m\in[B+1]$. Under Assumptions \ref{assum4}, \ref{assum1}, \ref{assum5}, and an additional assumption that $\|\widehat{\g}_i(m)\|_2\leq \widehat{G}$ for any $i\in V$ and $m\in[B]$, with probability at least $1-\gamma$, Algorithm \ref{DBBCG-SC} with $\alpha=0$ has
\begin{align*}
\sum_{m=1}^B\sum_{t\in\mathcal{T}_m}\sum_{j=1}^n\nabla\widehat{f}_{t,j,\delta}(\x_{j}(m))^\top(\bar{\x}(m)-\widetilde{\x}^\ast)\leq2nR(KG+\widehat{G})\sqrt{2B\ln\frac{1}{\gamma}}+4nhR^2+\frac{2nB\widehat{G}^2}{h}
\end{align*}
where $\widetilde{\x}^\ast=(1-\delta/r)\x^\ast$, $\x^\ast\in\argmin_{\x\in\K}\sum_{t=1}^Tf_{t}(\x)$, and $\widehat{f}_{t,j,\delta}(\x)$ denotes the $\delta$-smoothed version of $f_{t,j}(\x)$.
\end{lem}
According to Lemma \ref{lem2_azuma}, by assuming that $\|\widehat{\g}_i(m)\|_2\leq \widehat{G}$ for any $i\in V$ and $m\in[B]$, with probability at least $1-\gamma$, we have
\begin{equation*}
\begin{split}
R_{T,i}\leq&2nR(KG+\widehat{G})\sqrt{2B\ln\frac{1}{\gamma}}+4nhR^2+\frac{2nB\widehat{G}^2}{h}+3nGT\epsilon+3\delta nGT+\frac{\delta nGRT}{r}.
\end{split}
\end{equation*}
By substituting $\epsilon=\frac{4R}{\sqrt{L+2}}+\frac{\widehat{G}\sqrt{n}}{2h(1-\sigma_2(P))}$, $h=\frac{n^{1/4}\xi_TdMT^{3/4}}{\sqrt{1-\sigma_2(P)}R}$, $\delta=cT^{-1/4}$, $K=L=\sqrt{T}$, and $\widehat{G}=\xi_T\frac{dM\sqrt{K}}{\delta}+KG$ into the above inequality, we have
\begin{equation*}
\begin{split}
R_{T,i}\leq &2nR\left(2G+\frac{\xi_TdM}{c}\right)\sqrt{2\ln\frac{1}{\gamma}}T^{3/4}+\frac{4\xi_Tn^{5/4}dMR}{\sqrt{1-\sigma_2(P)}}T^{3/4}\\
&+2n^{3/4}\sqrt{1-\sigma_2(P)}\left(\frac{R}{c}+\frac{RG}{\xi_TdM}\right)\left(\frac{\xi_T dM}{c}+G\right)T^{3/4}\\
&+12n GRT^{3/4}+\frac{3n^{5/4}G}{2\sqrt{1-\sigma_2(P)}}\left(\frac{R}{c}+\frac{RG}{\xi_T dM}\right)T^{3/4}\\
&+3cnGT^{3/4}+\frac{c nGR}{r}T^{3/4}\\
= &O\left(n^{5/4}(1-\sigma_2(P))^{-1/2}T^{3/4}\xi_T\right).
\end{split}
\end{equation*}
Let $\mathcal{A}$ denote the event of $\|\widehat{\g}_i(m)\|_2\leq \widehat{G},\forall i\in V,m\in[B]$. Because we have used the event $\mathcal{A}$ as a fact, the above result should be formulated as
\begin{equation}
\label{eq_last}
\prob\left.\left(R_{T,i}= O\left(n^{5/4}(1-\sigma_2(P))^{-1/2}T^{3/4}\xi_T\right)\right|\mathcal{A}\right)\geq1-\gamma.
\end{equation}
Furthermore, we introduce the following lemma with respect to the probability of the event $\mathcal{A}$.
\begin{lem}
\label{lem_gradient}
Under Assumptions \ref{assum4} and \ref{assum2}, for all $i\in V$ and $m\in[B]$, Algorithm \ref{DBBCG-SC} has
\[
\|\widehat{\g}_i(m)\|_2\leq\left(1+\sqrt{8\ln\frac{nB}{\gamma}}\right)\frac{dM\sqrt{K}}{\delta}+KG
\]
with probability at least $1-\gamma$.
\end{lem}
Then, by applying Lemma \ref{lem_gradient} with $B=T/K=\sqrt{T}$, we have
\begin{equation}
\label{eq_last2}
\prob\left(\mathcal{A}\right)\geq1-\gamma.
\end{equation}
Finally, we complete the proof by combining (\ref{eq_last}) with (\ref{eq_last2}).

\section{Proof of Lemmas \ref{graph_lem_zhang} and \ref{graph_lem2_exp}}
These two lemmas can be derived by following the proof of Lemma 6 in \citet{wenpeng17}. For completeness, we include the detailed proof in this paper.

Let $P^{s}$ denote the $s$-th power of $P$ and $P^{s}_{ij}$ denote the $j$-th entry of the $i$-row in $P^{s}$ for any $s\geq0$. Note that $P^0$ denotes the identity matrix $I_n$.
For $m=1$, it is easy to verify that
\begin{equation}
\label{eq0_zbound}
\|\z_i(m)-\bar{\z}(m)\|_2=0\leq \frac{\sqrt{n}\widehat{G}}{1-\sigma_2(P)}.
\end{equation}
To analyze the case with $B\geq m\geq2$, we introduce two intermediate results from \citet{wenpeng17} and \citet{DADO2011}. First, as shown in the proof of Lemma 6 at \citet{wenpeng17}, for any $B\geq m\geq2$, we have
\begin{equation}
\label{eq1_zbound}
\begin{split}
\|\z_i(m)-\bar{\z}(m)\|_2
\leq\sum_{\tau=1}^{m-1}\sum_{j=1}^n\left|P_{ij}^{m-1-\tau}-\frac{1}{n}\right|\|\dd_j(\tau)\|_2
\end{split}
\end{equation}
under Assumption \ref{assum5}. Second, as shown in Appendix B of \citet{DADO2011}, when $P$ is a doubly stochastic matrix, for any positive integer $s$ and any $\x$ in the $n$-dimensional probability simplex, it holds that
\begin{equation}
\label{eq3_zbound}
\|P^0\x-\mathbf{1}/n\|_1\leq \sigma_2^s(P)\sqrt{n}
\end{equation}
where $\mathbf{1}$ is the all-ones vector in $\mathbb{R}^n$.

Let $\mathbf{e}_i$ denote the $i$-th canonical basis vector in $\mathbb{R}^n$. By substituting $\x=\mathbf{e}_i$ into (\ref{eq3_zbound}), we have
\begin{equation}
\label{eq4_zbound}
\|P^s\mathbf{e}_i-\mathbf{1}/n\|_1\leq \sigma_2^s(P)\sqrt{n}
\end{equation}
for any positive integer $s$. If $s=0$, we also have
\begin{equation}
\label{eq5_zbound}
\|P^0\mathbf{e}_i-\mathbf{1}/n\|_1=\frac{2(n-1)}{n}\leq\sqrt{n}=\sigma_2^0(P)\sqrt{n}
\end{equation}
where the inequality is due to $n\geq1$.

Then, for any $B\geq m\geq2$, by combining (\ref{eq1_zbound}) and $\|\dd_i(m)\|_2\leq \widehat{G}$, we have
\begin{equation}
\label{eq6_zbound}
\begin{split}
\|\z_i(m)-\bar{\z}(m)\|_2
\leq&\widehat{G}\sum_{\tau=1}^{m-1}\sum_{j=1}^n\left|P_{ij}^{m-1-\tau}-\frac{1}{n}\right|=\widehat{G}\sum_{\tau=1}^{m-1}\sum_{j=1}^n\left|P_{ji}^{m-1-\tau}-\frac{1}{n}\right|\\
=&\widehat{G}\sum_{\tau=1}^{m-1}\left\|P^{m-1-\tau}\mathbf{e}_i-\frac{\mathbf{1}}{n}\right\|_1
\end{split}
\end{equation}
where the first equality is due to the symmetry of $P$.

Because of (\ref{eq4_zbound}), (\ref{eq5_zbound}), and $\sigma_2(P)<1$, for any $B\geq m\geq2$, we have
\begin{equation}
\label{eq7_zbound}
\begin{split}
\|\z_i(m)-\bar{\z}(m)\|_2
\leq\widehat{G}\sum_{\tau=1}^{m-1}\sigma_2(P)^{m-1-\tau}\sqrt{n}
=\frac{(1-\sigma_2(P)^{m-1})\widehat{G}\sqrt{n}}{1-\sigma_2(P)}\leq\frac{\sqrt{n}\widehat{G}}{1-\sigma_2(P)}.
\end{split}
\end{equation}
By combining (\ref{eq0_zbound}) and (\ref{eq7_zbound}), we can complete the proof of Lemma \ref{graph_lem_zhang}.

Furthermore, by taking the expectation on the both sides of (\ref{eq1_zbound}) and combining with $\E[\|\dd_i(m)\|_2]\leq \widehat{G}$, we can prove Lemma \ref{graph_lem2_exp} in a similar way.

\section{Proof of Lemma \ref{lem2_azuma}}
We first introduce the classical Azuma's inequality \citep{Azuma67} for martingales in the following lemma.
\begin{lem}
\label{azuma}
Suppose $D_1,\dots,D_s$ is a martingale difference sequence and
\[|D_j|\leq c_j\]
almost surely. Then, we have
\[\prob\left(\sum_{j=1}^sD_j\geq\Delta\right)\leq\exp\left(\frac{-\Delta^2}{2\sum_{j=1}^sc_j^2}\right).\]
\end{lem}
To apply Lemma \ref{azuma}, with $\mathcal{T}_m=\{(m-1)K+1,\dots,mK\}$, we define
\begin{equation}
\label{lem2_azuma_eq1}
\begin{split}
D_m&=\sum_{t\in\mathcal{T}_m}\sum_{j=1}^n\left(\nabla\widehat{f}_{t,j,\delta}(\x_{j}(m))-{\g}_j(t)\right)^\top(\bar{\x}(m)-\widetilde{\x}^\ast)\\
&=\sum_{j=1}^n\left(\sum_{t\in\mathcal{T}_m}\nabla\widehat{f}_{t,j,\delta}(\x_{j}(m))-\widehat{\g}_j(m)\right)^\top(\bar{\x}(m)-\widetilde{\x}^\ast).
\end{split}
\end{equation}
According to Algorithm \ref{DBBCG-SC} and Lemma \ref{smoothed_lem2}, we have
\[\E\left.\left[D_m\right|\x_1(m),\dots,\x_n(m),\bar{\x}(m)\right]=0\]
which further implies that $D_1,\dots,D_B$ is a martingale difference sequence with
\begin{equation}
\label{R1-eq-Dm}
\begin{split}
\left|D_m\right|&=\left|\sum_{j=1}^n\left(\sum_{t\in\mathcal{T}_m}\nabla\widehat{f}_{t,j,\delta}(\x_{j}(m))-\widehat{\g}_j(m)\right)^\top(\bar{\x}(m)-\widetilde{\x}^\ast)\right|\\
&\leq\sum_{j=1}^n\left\|\sum_{t\in\mathcal{T}_m}\nabla\widehat{f}_{t,j,\delta}(\x_{j}(m))-\widehat{\g}_j(m)\right\|_2\left\|(\bar{\x}(m)-\widetilde{\x}^\ast)\right\|_2\\
&\leq2R\sum_{j=1}^n\left(\left\|\sum_{t\in\mathcal{T}_m}\nabla\widehat{f}_{t,j,\delta}(\x_{j}(m))\right\|_2+\left\|\widehat{\g}_j(m)\right\|_2\right)\\
&\leq2R\sum_{j=1}^n\sum_{t\in\mathcal{T}_m}\left\|\nabla\widehat{f}_{t,j,\delta}(\x_{j}(m))\right\|_2+2nR\widehat{G}\\
&\leq2nRKG+2nR\widehat{G}
\end{split}
\end{equation}
where the second inequality is due to Assumption \ref{assum1}, and the last inequality is due to Lemma \ref{smoothed_lem1} and $|\mathcal{T}_m|=K$.

Then, by applying Lemma \ref{azuma} with $\Delta=2nR(KG+\widehat{G})\sqrt{2B\ln\frac{1}{\gamma}}$, with probability at least $1-\gamma$, we have
\begin{equation}
\label{lem2_azuma_eq2}
\sum_{m=1}^BD_m\leq\Delta=2nR(KG+\widehat{G})\sqrt{2B\ln\frac{1}{\gamma}}.
\end{equation}
Additionally, by combining (\ref{lem2_azuma_eq1}) with $\bar{\g}(m)=\frac{1}{n}\sum_{i=1}^n\widehat{\g}_i(m)$, we further have
\begin{equation}
\label{lem2_azuma_eq3}
\sum_{m=1}^B\sum_{t\in\mathcal{T}_m}\sum_{j=1}^n\nabla\widehat{f}_{t,j,\delta}(\x_{j}(m))^\top(\bar{\x}(m)-\widetilde{\x}^\ast)=\sum_{m=1}^BD_m+n\sum_{m=1}^B\bar{\g}(m)^\top(\bar{\x}(m)-\widetilde{\x}^\ast).
\end{equation}
Therefore, we still need to bound $\sum_{m=1}^B\bar{\g}(m)^\top(\bar{\x}(m)-\widetilde{\x}^\ast)$. According to Assumption \ref{assum5}, it is easy to verify that Algorithm \ref{DBBCG-SC} with $\alpha=0$ ensures
\begin{equation*}
\begin{split}
\bar{\z}(m+1)=&\frac{1}{n}\sum_{i=1}^n\z_i(m+1)=\frac{1}{n}\sum_{i=1}^n\left(\sum_{j\in N_i}P_{ij}\z_{j}(m)+\widehat{\g}_{i}(m)\right)\\
=&\frac{1}{n}\sum_{i=1}^n\sum_{j=1}^nP_{ij}\z_{j}(m)+\bar{\g}(m)=\frac{1}{n}\sum_{j=1}^n\sum_{i=1}^nP_{ij}\z_{j}(m)+\bar{\g}(m)\\
=&\bar{\z}(m)+\bar{\g}(m)=\sum_{s=1}^{m}\bar{\g}(s).
\end{split}
\end{equation*}
Moreover, according to the definition, for any $m\in[B+1]$, we have \[\bar{\x}(m+1)=\argmin\limits_{\x\in\K_\delta}
\bar{F}_{m}(\x)=\argmin\limits_{\x\in\K_\delta}\bar{\z}(m)^{\top}\mathbf{x}+h\|\mathbf{x}-\x_{\ii}\|_2^2.\]
By applying Lemma \ref{ftrl1} with the linear loss functions $\{\bar{\g}(m)^\top\x\}_{m=1}^B$, the decision set $\K=\K_\delta$ and the regularizer $\mathcal{R}(\x)=h\|\mathbf{x}-\x_{\ii}\|_2^2$, we have
\begin{equation}
\label{lem2_azuma_eq4}
\begin{split}
\sum_{m=1}^B\bar{\g}(m)^\top(\bar{\x}(m+1)-\widetilde{\x}^\ast)&\leq h\|\widetilde{\x}^\ast-\x_{\ii}\|_2^2+\sum_{m=1}^B\bar{\g}(m)^\top(\bar{\x}(m+1)-\bar{\x}(m+2))\\
&\leq 4hR^2+\sum_{m=1}^B\|\bar{\g}(m)\|_2\|\bar{\x}(m+1)-\bar{\x}(m+2)\|_2
\end{split}
\end{equation}
where the last inequality is due to Assumption \ref{assum1}.

Note that $\bar{F}_{m+1}(\x)$ is $2h$-strongly convex and $\bar{\x}(m+2)=\argmin_{\x\in\K_\delta}\bar{F}_{m+1}(\x)$. For any $m\in[B]$, we have
\begin{equation*}
\begin{split}
&h\|\bar{\x}(m+1)-\bar{\x}(m+2)\|_2^2\\
\leq& \bar{F}_{m+1}(\bar{\x}(m+1))-\bar{F}_{m+1}(\bar{\x}(m+2))\\
=&\bar{F}_{m}(\bar{\x}(m+1))+\bar{\g}(m)^\top\bar{\x}(m+1)-\bar{F}_{m}(\bar{\x}(m+2))-\bar{\g}(m)^\top\bar{\x}(m+2)\\
\leq&\|\bar{\g}(m)\|_2\|\bar{\x}(m+1)-\bar{\x}(m+2)\|_2
\end{split}
\end{equation*}
where the first inequality is due to (\ref{cor_scvx}) and the second inequality is due to $\bar{\x}(m+1)=\argmin_{\x\in\K_\delta}\bar{F}_{m}(\x)$.

The above inequality implies that for any $m\in[B]$, it holds that
\[\|\bar{\x}(m+1)-\bar{\x}(m+2)\|_2^2\leq\frac{\|\bar{\g}(m)\|_2}{h}.\]
By combining with (\ref{lem2_azuma_eq4}), we have
\begin{equation}
\label{lem2_azuma_eq5}
\begin{split}
&\sum_{m=1}^B\bar{\g}(m)^\top(\bar{\x}(m)-\widetilde{\x}^\ast)\\
=&\sum_{m=1}^B\bar{\g}(m)^\top(\bar{\x}(m)-\bar{\x}(m+1))+\sum_{m=1}^B\bar{\g}(m)^\top(\bar{\x}(m+1)-\widetilde{\x}^\ast)\\
\leq& \sum_{m=1}^B\|\bar{\g}(m)\|_2\|\bar{\x}(m)-\bar{\x}(m+1)\|_2+4hR^2+\sum_{m=1}^B\|\bar{\g}(m)\|_2\|\bar{\x}(m+1)-\bar{\x}(m+2)\|_2\\
\leq& 4hR^2+\frac{1}{h}\sum_{m=2}^B\|\bar{\g}(m)\|_2\|\bar{\g}(m-1)\|_2+\|\bar{\g}(1)\|_2\|\bar{\x}(1)-\bar{\x}(2)\|_2+\frac{1}{h}\sum_{m=1}^B\|\bar{\g}(m)\|_2^2\\
\leq& 4hR^2+\frac{1}{h}\sum_{m=2}^B\|\bar{\g}(m)\|_2\|\bar{\g}(m-1)\|_2+\frac{1}{h}\sum_{m=1}^B\|\bar{\g}(m)\|_2^2
\end{split}
\end{equation}
where the last inequality is due to $\bar{\x}(1)=\x_{\ii}$ and $\bar{\x}(2)=\argmin_{\x\in\K_\delta}
\bar{F}_{1}(\x)=\x_{\ii}$.

Since $\|\widehat{\g}_i(m)\|_2\leq \widehat{G}$, for any $m\in[B]$, we also have
\begin{equation}
\label{thm2_eq_azuma_2}\|\bar{\g}(m)\|_2=\left\|\frac{1}{n}\sum_{i=1}^n\widehat{\g}_i(m)\right\|_2\leq\frac{1}{n}\sum_{i=1}^n\|\widehat{\g}_i(m)\|_2\leq \widehat{G}.
\end{equation}
By substituting (\ref{thm2_eq_azuma_2}) into (\ref{lem2_azuma_eq5}), we have
\begin{equation}
\label{lem2_azuma_eq6}
\begin{split}
\sum_{m=1}^B\bar{\g}(m)^\top(\bar{\x}(m)-\widetilde{\x}^\ast)\leq 4hR^2+\frac{(2B-1)\widehat{G}^2}{h}\leq 4hR^2+\frac{2B\widehat{G}^2}{h}.
\end{split}
\end{equation}
Finally, by substituting (\ref{lem2_azuma_eq2}) and (\ref{lem2_azuma_eq6}) into (\ref{lem2_azuma_eq3}), we complete the proof.
\section{Proof of Lemma \ref{lem_gradient}}
This proof is inspired by the proof of Theorem 12 in \citet{Best11}, which gave the classical Bernstein inequality for independent vector-valued random variables. However, the vector-valued random variables in this proof are only conditionally independent, and we do not need to use the Bernstein inequality to incorporate the variance information.

According to Algorithm \ref{DBBCG-SC}, for any $i\in V$ and $m=1,\dots,B$, conditioned on $\x_i(m)$,
\[\g_i((m-1)K+1),\dots,\g_i(mK)\]
are $K$ independent random vectors. For brevity, for $j=1,\dots,K$, let \[X_j=\g_i(t_j)\]
where $t_j=(m-1)K+j$, and let $N=\left\|\sum_{j=1}^KX_j\right\|_2$, $\widehat{S}_j=\sum_{k\neq j}X_k$.

To bound $N$ by using Lemma \ref{azuma}, we define $\mathbf{X}_0=\{\x_i(m)\}$, $\mathbf{X}_j=\{\x_i(m),X_1,\dots,X_j\}$ for $j\geq1$ and a sequence $D_1,\dots,D_K$ as
\[D_j=\E[N|\mathbf{X}_j]-\E[N|\mathbf{X}_{j-1}].\]
It is not hard to verify that
\[\E[D_{j}|\mathbf{X}_{j-1}]=\E[\E[N|\mathbf{X}_j]-\E[N|\mathbf{X}_{j-1}]|\mathbf{X}_{j-1}]=0\]
which implies that $D_1,\dots,D_K$ is a martingale difference sequence.

Then, using the triangle inequality, we have
\begin{equation}
\label{lem5_eq2}
N\leq\|\widehat{S}_j\|_2+\|X_j\|_2 \text{ and } N\geq\|\widehat{S}_j\|_2-\|X_j\|_2.
\end{equation}
Moreover, according to the Algorithm \ref{DBBCG-SC} and Assumption \ref{assum2}, we have
\begin{equation*}
\begin{split}
\|X_j\|_2=\left\|\frac{d}{\delta}f_{t_j,i}(\y_{i}(t_j))\uu_{i}(t_j)\right\|_2\leq\frac{dM}{\delta}.
\end{split}
\end{equation*}
Therefore, by combining with (\ref{lem5_eq2}), we have
\begin{equation}
\label{eq1-R1}
N\leq\|\widehat{S}_j\|_2+\frac{dM}{\delta} \text{ and } N\geq\|\widehat{S}_j\|_2-\frac{dM}{\delta}.
\end{equation}
Then, we have
\begin{align*}
D_j\leq\E[\|\widehat{S}_j\|_2|\mathbf{X}_j]+\frac{dM}{\delta}-\E[\|\widehat{S}_j\|_2|\mathbf{X}_{j-1}]+\frac{dM}{\delta}=\frac{2dM}{\delta}
\end{align*}
and
\begin{align*}
D_j\geq&\E[\|\widehat{S}_j\|_2|\mathbf{X}_j]-\frac{dM}{\delta}-\E[\|\widehat{S}_j\|_2|\mathbf{X}_{j-1}]-\frac{dM}{\delta}=-\frac{2dM}{\delta}
\end{align*}
where the above two equalities are due to $\E[\|\widehat{S}_j\|_2|\mathbf{X}_j]=\E[\|\widehat{S}_j\|_2|\mathbf{X}_{j-1}]$, because $\widehat{S}_j$ dose not depend on $X_j$ given $\x_i(m)$. Therefore, we have $|D_j|\leq\frac{2dM}{\delta}$.

Let $\Delta=\frac{\sqrt{K}dM}{\delta}\sqrt{8\ln\frac{nB}{\gamma}}$. Then, by applying Lemma \ref{azuma}, with probability at least $1-\frac{\gamma}{nB}$, we have
\begin{align*}
N-\E[N|\x_i(m)]=\E[N|\mathbf{X}_K]-\E[N|\mathbf{X}_0]=\sum_{j=1}^KD_j&\leq\frac{\sqrt{K}dM}{\delta}\sqrt{8\ln\frac{nB}{\gamma}}
\end{align*}
which implies that
\[
\|\widehat{\g}_i(m)\|_2=N\leq\frac{\sqrt{K}dM}{\delta}\sqrt{8\ln\frac{nB}{\gamma}}+\E[N|\x_i(m)]\leq\frac{\sqrt{K}dM}{\delta}\sqrt{8\ln\frac{nB}{\gamma}}+\sqrt{\E[N^2|\x_i(m)]}.
\]
where the last inequality is due to Jensen's inequality.

By combining the above inequality with $N^2=\|\widehat{\g}_i(m)\|_2^2$ and (\ref{eq_EN}), with probability at least $1-\frac{\gamma}{nB}$, we have
\begin{align*}
\|\widehat{\g}_i(m)\|_2\leq\left(1+\sqrt{8\ln\frac{nB}{\gamma}}\right)\frac{dM\sqrt{K}}{\delta}+KG.
\end{align*}
Finally, by using the union bound, we complete the proof for all $i\in V$ and $m=1,\dots,B$.

\vskip 0.2in
\bibliography{ref}

\end{document}